\documentclass[12pt]{article}
\makeatletter
\def\maxwidth{ %
  \ifdim\Gin@nat@width>\linewidth
    \linewidth
  \else
    \Gin@nat@width
  \fi
}
\makeatother

\RequirePackage{amsmath,amssymb}
\RequirePackage{amsthm}
\RequirePackage{natbib}
\RequirePackage[colorlinks,allcolors=blue]{hyperref}

\usepackage{fullpage}
\usepackage[nodisplayskipstretch]{setspace}

\usepackage{bbm,graphicx,bm} 
\usepackage{enumitem}
\usepackage{multirow}
\usepackage[all,import]{xy}
\usepackage{append ix}
\usepackage{comment}
\usepackage{color}
\usepackage{soul}
\usepackage{pdflscape}
\usepackage{subcaption}

\usepackage{booktabs}
\usepackage{longtable}
\usepackage{array}
\usepackage[table]{xcolor}
\usepackage{wrapfig}
\usepackage{float}
\usepackage{colortbl}
\usepackage{pdflscape}
\usepackage{tabu}
\usepackage{threeparttable}
\usepackage{threeparttablex}
\usepackage[normalem]{ulem}
\usepackage{makecell}
\usepackage{afterpage}
\usepackage{thmtools}
\usepackage[compact]{titlesec}

\theoremstyle{definition}

\newtheorem{theorem}{Theorem}
\newtheorem{lemma}{Lemma}

\newtheorem{corollary}{Corollary}
\newtheorem{proposition}{Proposition}

\newcommand\indep{\protect\mathpalette{\protect\independenT}{\perp}}
\def\independenT#1#2{\mathrel{\rlap{$#1#2$}\mkern2mu{#1#2}}}
\newcommand{\R}{\ensuremath{\mathbb{R}}}
\newcommand{\B}{\ensuremath{\mathbb{B}}}

\newcommand{\Z}{\ensuremath{\mathbb{Z}}}

\newcommand{\bbone}{\ensuremath{\mathbbm{1}}}
\newcommand{\E}{\ensuremath{\mathbb{E}}}

\newcommand{\logit}{\text{logit}}
\newcommand{\argmax}{\text{argmax}}

\newcommand{\bX}{\bm X}
\newcommand{\bx}{\bm x}

\def\super{\textsuperscript}

\newcommand\mcal[1]{\ensuremath{\mathcal{#1}}}

\def\b1{\boldsymbol{1}}

\def\spacingset#1{\renewcommand{\baselinestretch}%
{#1}\small\normalsize} \spacingset{1}

\definecolor{RED}{RGB}{255,0,0}

\usepackage{soul}
\usepackage[utf8]{inputenc}

\begin{document}

\newcommand{\blind}{0}

\newcommand{\tit}{\LARGE \bf Safe Policy Learning through Extrapolation: Application to
Pre-trial Risk Assessment}

\if0\blind

\title{\tit\thanks{We acknowledge the partial support from Cisco
    Systems, Inc. (CG\# 2370386), National Science Foundation
    (SES--2051196), Sloan Foundation (Economics Program; 2020--13946),
    National Natural Science Foundation of China (Grant No. 12371285, 12292984),
    and Arnold Ventures.  We thank Benedikt Koch and anonymous
    reviewers of the IQSS's Alexander and Diviya Magaro Peer
    Pre-Review Program for useful feedback.}}

\date{\today}

\author{Eli Ben-Michael\thanks{Assistant Professor, Department of
    Statistics \& Data Science and Heinz College of Information
    Systems \& Public Policy, Carnegie Mellon University. 4800 Forbes
    Avenue, Hamburg Hall, Pittsburgh PA 15213.  Email:
    \href{mailto:ebenmichael@cmu.edu}{ebenmichael@cmu.edu} URL:
    \href{https://ebenmichael.github.io}{ebenmichael.github.io}} \and
  \and D. James Greiner\thanks{Honorable S.  William Green Professor
    of Public Law, Harvard Law School, 1525 Massachusetts Avenue,
    Griswold 504, Cambridge, MA 02138.}  \and Kosuke
  Imai\thanks{Professor, Department of Government and Department of
    Statistics, Harvard University.  1737 Cambridge Street, Institute
    for Quantitative Social Science, Cambridge MA 02138.  Email:
    \href{mailto:imai@harvard.edu}{imai@harvard.edu} URL:
    \href{https://imai.fas.harvard.edu}{https://imai.fas.harvard.edu}}
  \and Zhichao Jiang\thanks{Professor, School of Mathematics, Sun
    Yat-sen University, Guangzhou Guangdong 510275, China. Email:
    \href{mailto:jiangzhch7@mail.sysu.edu.cn}{jiangzhch7@mail.sysu.edu.cn}
  }}
\maketitle
  
\fi

\if1\blind
\title{\bf \tit}
\maketitle
\fi

\renewcommand\thmcontinues[1]{Continued}

\thispagestyle{empty}
\pagenumbering{gobble}

\begin{abstract}

  Algorithmic recommendations and decisions have become ubiquitous in
  today's society.  Many of these data-driven policies, especially in
  the realm of public policy, are based on known, deterministic rules
  to ensure their transparency and interpretability.  We examine a
  particular case of algorithmic pre-trial risk assessments in the US
  criminal justice system, which provide 
  deterministic classification scores and recommendations to help
  judges make release decisions.  Our goal is to analyze data from a
  unique field experiment on an algorithmic pre-trial risk assessment
  to investigate whether the scores and recommendations can be improved.
  Unfortunately, prior methods for policy
  learning are not applicable because they require existing policies
  to be stochastic.  We develop a maximin robust
  optimization approach that partially identifies the expected utility
  of a policy, and then finds a policy that maximizes the worst-case
  expected utility.  The resulting policy has a
  statistical safety property, limiting the probability of producing
  a worse policy than the existing one, under structural assumptions
  about the outcomes.
  Our analysis of data from the field experiment shows that we can
  safely improve certain components of the risk
  assessment instrument by classifying arrestees as lower risk under a
  wide range of utility specifications, though the analysis is not 
  informative about several components of the instrument.

\end{abstract}

\pagenumbering{arabic}
\spacingset{1.75} 
\section{Introduction}

Algorithmic recommendations and decisions are ubiquitous in our daily
lives.  Many algorithmic policies are used for consequential
decisions in high stakes settings such as criminal justice, social
policy, and medical care.  One common feature of such policies is that
they are based on known, deterministic rules.  This is often because
transparency and interpretability are required to ensure
accountability especially when used for public policy-making.

In this paper, we focus on a particular case: pre-trial risk
assessment instruments (PRAI) in the American criminal justice system.
The goal of a PRAI is to aid judges in deciding which arrestees should
be released pending the disposition of any criminal charges.  We
consider a particular PRAI used in Dane County, Wisconsin, which includes the state capital, Madison
(Section~\ref{sec:PSA}).  This PRAI assigns scores to arrestees
according to the risk that they are predicted to engage in undesirable
behavior.  It then aggregates these scores using a deterministic
function and provides an overall release recommendation to the judge.

We analyze data from a unique field experiment on the PRAI
\citep{greiner2020_rct,Imai2020}.  Our goal is to learn new
algorithmic scoring and recommendation rules that can lead to better
overall outcomes while retaining the transparency of the existing
instrument.  Importantly, we focus on changing the algorithmic
policies, which we can intervene on, rather than judge's decisions,
which we cannot.

The large amounts of data collected after implementing {\it
  deterministic} policies such as PRAIs provide an opportunity to
learn new policies that improve on the status quo.  Unfortunately,
prior approaches to policy learning are not applicable because they
require existing policies to be {\it stochastic}, typically relying on
inverse probability weighting
(Section~\ref{sec:methods}).

To address this challenge (Section~\ref{sec:safe_policy}), we
partially identify the expected utility of a policy by calculating all
potential values consistent with the observed data.  This makes
choosing an ``optimal'' policy ambiguous: a policy can perform well
under some outcome models that are consistent with the data and poorly
in others.  We use the maximin criterion that finds a policy that
maximizes the worst-case performance relative to the status quo.  The
resulting policy has a statistical {\it safety} property that limits
the probability of yielding a worse outcome than the status quo
policy, under the structural assumptions made about the outcomes.
However, this safety property comes at the cost of
potentially choosing a sub-optimal policy, though it is no worse than
the status quo.  We formally characterize the gap between this safe
policy and the infeasible oracle policy.

We use this approach to explore whether the data from our field
experiment support alterations to the existing PRAI
(Section~\ref{sec:application}).  We explore the three risk
measures based on the predicted likelihood that an arrestee, upon
release, will (i) fail to appear in court (FTA), (ii) engage in new
criminal activity (NCA), or (iii) engage in new violent criminal
activity (NVCA).  We also inspect the algorithm that recommends to the
judge the level of cash bail and pre-trial supervision and monitoring
conditions to impose.

We find that under several specifications of the utility function, it
can be possible to improve safely upon the existing NVCA scoring rule by
classifying arrestees as lower risk; if the policy maker is
primarily focused on avoiding NVCAs, the resulting safe policy falls
back on the existing scoring rule. 
However, our approach has limitations. Conducting our analysis
requires several non-trivial choices that may be challenging in
practice.  In addition, our analysis does not provide meaningful
insights about components of the instrument other than the NVCA
scoring rule. This arises from identifiability issues caused by
the structure of the underlying rules, as well as a high degree of
statistical uncertainty due to small sample sizes for rare
combinations of risk factors.  We discuss these and other
limitations in Section~\ref{sec:discussion}.

\section{Pre-trial Risk Assessment}
\label{sec:PSA}

We now briefly describe the particular PRAI, called the Public Safety
Assessment (PSA), used in Dane County, Wisconsin.  The PSA is an
algorithmic recommendation designed to help judges make their
pre-trial release decisions.  We will also describe an original
randomized experiment we conducted to evaluate the impact of the PSA
on judges' decisions.  In Section~\ref{sec:application}, we analyze
this experimental dataset and consider how to improve outcomes by
modifying certain aspects of the PSA system.  Interested readers
should consult \citet{greiner2020_rct} and \citet{Imai2020} for
further details of the PSA and experiment; the study dataset has been
made publicly available.

\subsection{The PSA-DMF system}
\label{sec:psa_explain}

The goal of the PSA is to help judges decide, at first appearance
hearings, whether to allow an arrestee's release without bail or
release them only if the arrestee posts bail/bond (or meets other
conditions).  Because arrestees are presumed to be innocent, judges
must avoid unnecessary incarceration.  The PSA has several outputs.
First, it returns three classification scores based on the predicted
risk that each arrestee will engage in an FTA, NCA, or NVCA.  Law
requires judges to balance between these risks and the cost of
incarceration.  These three PSA scores are then combined via the
so-called ``Decision Making Framework'' (DMF) into two overall
recommendations: (i) whether to require a signature bond
(i.e., release on their own recognizance) or some level of cash bail
for release, and (ii) what, if any, monitoring conditions to place on
release.
Given the complexity of the system, 
  our empirical analysis will focus on the question of how to improve each component
separately (see Section~\ref{sec:application}).

\paragraph{FTA, NCA, and NVCA risk scores.}
These scores are deterministic functions of eight risk factors.  The
only demographic factor is the arrestee's age, and neither gender nor
race is used.  The other risk factors include the current offense and
pending charges as well as measures of criminal history based on prior
convictions and prior FTAs.  These scores are constructed by assigning
an integer-valued weight to each present risk factor, adding them
together, and thresholding this value into a number of bins.  For the
sake of transparency, the foundation that funded the PSA’s creation
made these weights and thresholds publicly (see
\href{https://advancingpretrial.org/psa/factors/}{https://advancingpretrial.org/psa/factors};
Appendix Table~\ref{tab:nvca_weights}
 summarizes the weights.

The FTA score has six levels and is based on four
risk factors.  The values range from 0 to 7, and the final score is
thresholded into values between 1 (lowest risk) and 6 (highest risk)
by assigning
$\{0 \to 1, 1 \to 2, 2 \to 3, (3,4) \to 4, (5,6) \to 5, 7 \to
6\}$. The NCA score also has six levels, but is based on six risk
factors and has a maximum value of 13 before being collapsed into six
levels by assigning
$\{0 \to 1, (1,2) \to 2, (3,4) \to 3, (5,6) \to 4, (7,8) \to 5),
(9,10,11,12,13) \to 6\}$.  Finally, the NVCA score is a binary flag
based on five risk factors: if the sum of the weights is greater than
or equal to 4, the PSA flags the arrestee as being at elevated risk of
an NVCA.  Otherwise, the NVCA score is 0, and the arrestee is not
flagged as being at elevated risk.

\begin{figure}[t!]
  \vspace{-.25in}
  \centering \includegraphics[width=0.7\maxwidth]{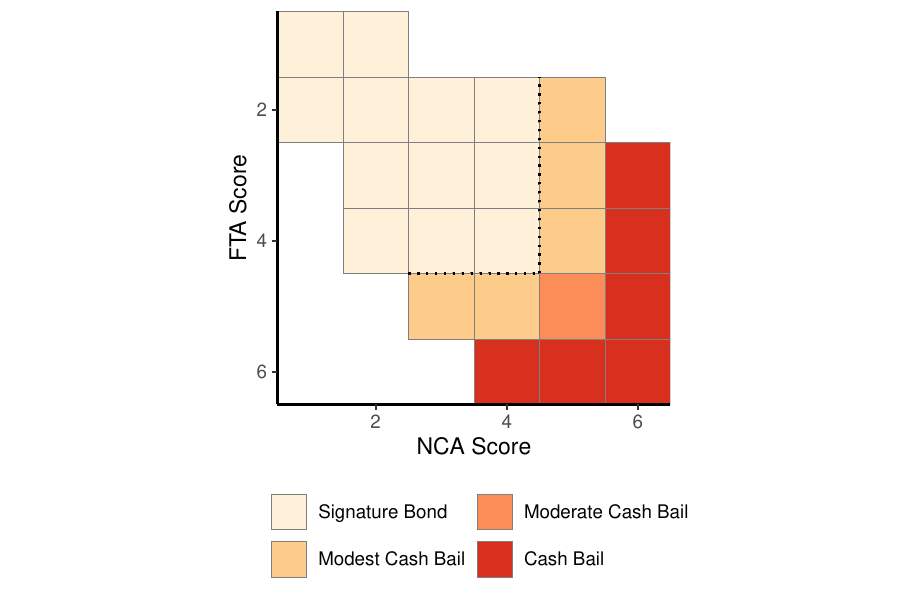}
\caption{Decision Making Framework (DMF) matrix for cases where the current charge is not a serious violent offense, the NVCA flag is not triggered, and the defendant was not extradited. If the FTA score and the NCA score are both less than 5, then the recommendation is to only require a signature bond. Otherwise the recommendation is to require some amount of cash bail. The dashed line indicates this boundary.
Unshaded areas indicate impossible combinations of FTA and NCA scores.}
\label{fig:dmf_matrix}
\end{figure}

\paragraph{Recommendations via the DMF.}
Next, the DMF transforms these three PSA risk scores into a recommendation
regarding cash bail and one regarding additional monitoring conditions.
For cases where the current charge is one of several serious violent
offenses, the defendant was extradited, or the NVCA score is 1, the
DMF automatically recommends cash bail with maximum supervision and
monitoring conditions.  For the remaining cases, the FTA and NCA risk
scores are combined into one of 7 overall risk levels.  If the FTA and
NCA scores are both less than 5, and so the risk level is 3 or lower,
then the recommendation is to only require a signature bond.
Otherwise the recommendation is to require cash bail (limited to
``modest'' at levels 4--5 and ``moderate'' at level 6).
Figure~\ref{fig:dmf_matrix} visualizes the cash bail portion of the
DMF.  The risk levels similarly encode a recommendation for an
increasing amount of pre-trial supervision and monitoring conditions,
ranging from none (level 1) to maximum supervision with biweekly phone
and face-to-face contacts (level 7).  Appendix
Figure~\ref{fig:dmf_matrix_full} shows these conditions along with
the cash bail recommendations.

\subsection{The experimental data}

We analyze the data from a randomized controlled trial conducted in
Dane County, Wisconsin.  In this experiment, the PSA was computed for
each first appearance hearing that a single judge oversaw during the
study period.  Across cases, we randomized whether the PSA was made
available in its entirety to the judge.  If a case is assigned to the
treatment group, the judge received the three PSA scores, the DMF
recommendations, and all of the risk factors that were used to
construct them on a single sheet of paper.  For the control
group, the judge did not receive the PSA scores and DMF
recommendations. Since the risk factors that go into the PSA were
  made available in other case files, the judge could, in principle,
  reconstruct the PSA output with enough time.

For each case, we observe the three scores (FTA, NCA, and NVCA) and
the DMF recommendation, the underlying risk factors used to construct
the scores, the binary decision by the judge (signature bond or cash
bail), and three binary outcomes (FTA, NCA, and NVCA).  We focus on
first arrest cases in order to avoid spillover effects between
cases. All told, there are 1,891 cases, 948 of which the judge was
given access to the PSA.

Our goal is to improve the PSA recommendation system while
taking into account the judicial decisions that partly result from
the algorithmic recommendations; see
Appendix~\ref{sec:human_decisions} for further discussion on
incorporating judicial decisions into the analysis.  Crucially, each
component of the PSA is \emph{deterministic} and no aspect of it was
randomized as part of the study.  Therefore, there is a lack of
\emph{overlap}: the probability that any case would have had a
different algorithmic recommendation than it actually received is
exactly zero.  This makes existing approaches to policy learning
inapplicable because they rely on the inverse of this probability.
Instead, learning a new recommendation policy in the absence of
overlap requires \emph{extrapolation}.  Below, we will develop a
methodological framework that provides a statistical property that
the new, learned rules perform at least as well as the original
recommendation.

\section{Policy Learning with Observational Data}
\label{sec:methods}

\subsection{Notation and setup}

Suppose that we have a representative sample of $n$ units
independently drawn from a population $\mathcal{P}$.  For each unit
$i=1,\ldots,n$, we observe a set of covariates
$\bX_i \in \mcal{X} \subseteq \R^p$ (e.g., the risk factors from
Appendix Table~\ref{tab:nvca_weights}) and a binary outcome $Y_i \in \{0,1\}$.
In our analysis presented in Section~\ref{sec:application}, we
alternately consider the outcome $Y_i = 1$ as the \emph{absence} of an
FTA, NCA, or NVCA.
We consider a set of $K$ possible actions, denoted by
$\mathcal{A} = \{0,1,2,\ldots,K-1\}$ that can be taken for each unit.

The actions correspond to the PSA recommendation: there are
$K=6$ possible actions when we consider the FTA and NCA risk scores,
$K = 2$ for the NVCA flag, $K = 7$ for the overall DMF bail and
monitoring recommendation, and $K=2$ for the signature bond versus
cash bail recommendation.  In our experimental evaluation, we have
access to the algorithm that generated the observed actions.
Formally, we encode this as a known baseline deterministic policy
$\tilde{\pi} : \mcal{X} \to \mathcal{A}$ that generates the observed
actions $A_i = \tilde{\pi}(\bX_i)$.  Throughout this paper, we will
also refer to the baseline policy as
$\tilde{\pi}(\bx, a) \equiv \bbone\{\tilde{\pi}(\bx) = a\}$, the
indicator of whether the baseline policy yields action $a$ given the
covariates $\bx$.

We consider the effects of the algorithmic recommendation on the
outcome, and assume that the algorithmic action $A_i$ may affect its
own unit's outcome $Y_i$ but has no impact on the outcomes of other units
(no interference between units; \cite{rubin1980}).  Then, we can
write the potential outcome under each action $A_i = a$ as $Y_i(a)$
where $a \in \mathcal{A}$ and the observed outcome as
$Y_i = Y_i(A_i) = Y_i(\tilde{\pi}(\bX_i))$ \citep{neyman1923}.  This
setup focuses on the impacts of the algorithmic recommendation whose
provision was randomized in our experimental evaluation.  We
marginalize over the potential human judicial decisions that may be
influenced by the algorithmic recommendation (see
Appendix~\ref{sec:human_decisions} for further formalization).
Finally, our setting implies that
$(\{Y_i(a)\}_{a \in \mathcal{A}}, \bX_i)$ are independent and
identically distributed, so we sometimes drop the $i$ subscript.

\subsection{Optimal policy learning}
\label{sec:opt_pol_learn}

Our primary goal is to find a new deterministic policy
$\pi:\mcal{X} \to \mathcal{A}$, that has a high expected utility.  We
will again use the notation $\pi(\bx, a) \equiv \bbone\{\pi(\bx) = a\}$
for the policy being equal to action $a$ given the covariates $\bx$.
Let $u(y, a)$ denote the utility for outcome $y$ under action $a$.
Because the outcomes are binary, we can write this utility function
as:\footnote{While we focus here on binary potential outcomes, this
  form of the utility function shows that we can extend our results to
  the case with continuous outcomes with utility functions that are
  linear in the (possibly transformed) outcomes. }
\begin{equation*}
  Y(a)u(1,a) + \{1 - Y(a)\}u(0,a) \ = \ \{u(1,a) - u(0,a)\}Y(a) + u(0,a).
\end{equation*}

The two key components of this utility function are (i) the utility
change between the two outcomes for action $a$,
$u(a) \equiv u(1,a) - u(0,a)$, which we assume is non-negative without
loss of generality, and (ii) the utility for an outcome of zero with
an action $a$, $c(a) \equiv u(0,a)$. We will refer to the latter term
as the ``cost'' because it denotes the utility under action $a$ when
the outcome event does not happen; $c(a)=0$ corresponds to the action
having no cost.  We define the utility using both the outcome
$y$ and the action $a$ to capture the fact that some actions are
costly. For example, in Section~\ref{sec:application}, we will place a
cost on triggering the NVCA flag, recommending cash bail, or assigning
a high NCA, FTA, or overall risk score.
We note, however, that our approach is agnostic to the particular choice
of the utility function.

While this utility only takes into account the policy action and the
outcome, policy makers may also be concerned about the costs of
subsequent human decisions that are possibly affected by algorithmic
recommendations or actions. In Appendix~\ref{sec:human_decisions}, we
show how to incorporate such factors into the utility function.

The value of policy $\pi$ is the expected utility under policy $\pi$
across the population,
\begin{equation}
  \label{eq:policy_model}
  V(\pi, m^\ast) \ = \ \E\left[\sum_{a \in \mathcal{A}}\pi(\bX, a)
      \{u(a) Y(a) + c(a)\}\right] \ = \ \E\left[\sum_{a \in
      \mathcal{A}} \pi(\bX, a) \{u(a) m^\ast(a, \bX) + c(a)\}\right], 
\end{equation}
where we have used the law of iterated expectations, with the
  first expection over $\bX$ and $Y(a)$, and the second expectation
  over $\bX$, to show the dependence on the conditional expected
potential outcome function
$m^\ast(a,\bx) \equiv \E[Y(a)\mid \bX = \bx] $. We explicitly denote
the value under different potential models for our development below;
in cases where it is not ambiguous, we omit the $m^\ast$ argument to
indicate the value under the true conditional expected potential
outcome function.

Ideally, we would like to find a policy $\pi$ that has the highest
value within a policy class $\Pi$. We can write a population optimal
policy as one that maximizes the value, i.e.,
$\pi^\ast \in \argmax_{\pi \in \Pi} V(\pi)$.
The policy class $\Pi$ is an important object both in the theoretical
analysis and in applications.  In Section~\ref{sec:application}, we
discuss the substantive choice of policy class when applied to a PRAI.

To find an optimal policy, we need to point-identify the value
$V(\pi, m^\ast)$ for all candidate policies $\pi \in \Pi$.  Existing
methods rely on an overlap assumption for identification. In our
context, this would require that each case has a non-zero
probability of being assigned algorithmic action $A=a$, i.e. that
$P(A = a \mid \bX) > 0$ for all $a \in \mathcal{A}$. If the baseline
policy were stochastic, satisfying the overlap assumption, we could
directly use inverse probability weighting, model-based weighting,
or a doubly robust approach to learn an optimal policy from data
\citep[e.g.][]{Qian2011,Zhao2012,Kitagawa2018,Dudik2011, Athey2021}.
In our application and many
other settings, however, the baseline policy $\tilde{\pi}$ is a
deterministic function of covariates, implying a lack of overlap.
Thus, we cannot point-identify the value $V(\pi, m^\ast)$ for all
policies $\pi \in \Pi$ and hence cannot using existing approaches.
In Appendix~\ref{sec:impute}, we provide further discussion about
this identification issue.  

\section{Safe Policy Learning through Extrapolation}
\label{sec:safe_policy}

To deal with the lack of overlap brought on by the deterministic
policy, we propose to first partially identify the conditional
expectation, and then use robust optimization to find the best policy
under the worst-case model.  We will develop our optimal safe policy
approach in two parts.  First, we show how to construct a safe policy
if we had access to an infinite number of samples, i.e., in the
population.  We then discuss how to construct policies empirically
from data, and establish finite-sample statistical properties of the policies.
Finally, we show how to incorporate the experimental
control units to weaken the assumptions of our general approach and
discuss the practical implementation of the procedure for our
analysis.

\subsection{Partially identifying the value of a policy}
\label{sec:partial}

To understand how the lack of overlap affects our ability to find a
new policy, we will separate the value of a policy into identifiable
and unidentifiable components.  We will then consider scenarios where
it is possible to at least \emph{partially identify} the latter term.
To do so, we can write the value $V(\pi,m^\ast)$ in terms of the observed outcome
$Y$ when our policy $\pi$ agrees with the baseline policy
$\tilde{\pi}$, and the unidentifiable full model $m^\ast(a,\bx)$ when $\pi$
disagrees with $\tilde{\pi}$:
\begin{equation}
  \label{eq:value_model}
    V(\pi, m^\ast) \ = \ \E\left[\sum_{a \in \mathcal{A}} \pi(\bX, a) \left\{
        u(a) \left[\tilde{\pi}(\bX, a) Y + \left\{1 - \tilde{\pi}(\bX,a)\right\}
         m^\ast(a, \bX)\right]+c(a)\right\} \right].
\end{equation}
  
Without further assumptions, we cannot point-identify the value of the
conditional expectation when $a$ is different from the baseline policy
and so we cannot identify $V(\pi,m^\ast)$ for an arbitrary policy $\pi$. 
If we place restrictions on $m^\ast(a,\bx)$, however, we can partially
identify a range of potential values for a given policy $\pi$
\citep{Manski2005_partial}.  Specifically, we encode the conditional
expectation as a function $m:\mathcal{A} \times \mcal{X} \to [0,1]$,
and restrict it to be in a particular model class $\mcal{F}$.  We then
combine this with the fact that we have identified some function
values, i.e., the conditional expectation of the observed outcome
under the baseline policy
$\tilde{m}(\bx) \equiv m^\ast(\tilde{\pi}(\bx),\bx) = \E[Y \mid \bX = \bx] $, to
form a restricted model class:
\begin{equation}
  \label{eq:restricted_class}
  \mcal{M} \ = \ \{f \in \mcal{F} \; \mid f(a, \bx) = \tilde{m}(\bx) \; \forall \bx \in \mathcal{X}, a = \tilde{\pi}(\bx)\}.
\end{equation}

This restricted model class combines the structural information
from the underlying class $\mcal{F}$ (i.e., $f \in \mathcal{F}$)
with the observable implications from the data (i.e.,
$f(\tilde{\pi}(\bx), \bx) = \tilde{m}(\bx)$).
With this setup, a policy $\pi$ can be associated with a range of
possible values $\{V(\pi, m) \mid m \in \mathcal{M}\}$, one for each
observationally indistinguishable model.  We discuss particular
choices of the model class $\mcal{F}$ in our study (see
Section~\ref{sec:application}), deferring computations to construct
the associated restricted model class $\mcal{M}$ to
Appendix~\ref{sec:representative_cases}.

\subsection{Criteria for decision-making under ambiguity}
\label{sec:opt_pop}

The lack of identifiability leads to an ambiguity in choosing
an ``optimal'' policy: a policy could have a high value under one
model and a low value under another, and no amount of data can help to
adjudicate between the two scenarios.  However, the value of the
baseline policy $\tilde{\pi}$ is point-identified using the observed
policy values and outcomes:
\begin{equation*}
  V(\tilde{\pi}) \ = \ \E\left[\sum_{a \in \mathcal{A}} \tilde{\pi}(\bX, a)\{u(a) Y + c(a)\}\right].
\end{equation*}
The baseline policy $\tilde{\pi}$ is also already
implemented, so a natural requirement of a new policy is that it
performs \emph{at least as well} as the baseline.

To construct such a policy, we take a maximin approach by finding a
policy that maximizes the improvement over the baseline in the worst
case:
\begin{equation}
  \label{eq:maximin}
   \pi^{\inf} \in \underset{\pi \in \Pi}{\argmax}\min_{m
      \in \mcal{M}}  \left\{  V(\pi, m) - V(\tilde{\pi})\right\}.
\end{equation}
Because the value of the baseline is point-identified, this is
equivalent to finding a policy that maximizes the worst-case value
across the set of potential models $\mcal{M}$, i.e.,
$\underset{\pi \in \Pi}{\argmax} \underset{m \in \mcal{M}}{\min}
V(\pi, m)$.

Such maximin criteria have been widely used for policy learning in
various contexts with partial identification
\citep[e.g.,][]{Kallus2021,Pu2021}.  Other applications include
decision problems with ambiguity more broadly, such as robust
statistical learning and robust optimization
\citep[e.g.,][]{Duchi2021_dro,Bertsimas2011}.  In
addition, \citet{Gilboa1989} show that the maximin expected utility
criterion is equivalent to having a preference relation among policies
that satisfies a notion of \emph{uncertainty aversion} (in addition to
other more standard properties).

A benefit of choosing the maximin criterion is that so long as the
policy class $\Pi$ includes the baseline policy $\tilde{\pi}$, and the
underlying model lies in the restricted model class $\mcal{M}$, the
maximin optimal policy $\pi^{\inf}$ will be at least as good as the
baseline.  We formalize this as the following proposition.
\begin{proposition}[Population safety]
  \label{prop:pop_safety} \singlespacing Let $\pi^{\inf}$ be a
  solution to Eqn~\eqref{eq:maximin}. If $m^\ast \in \mcal{M}$,
  and $\tilde{\pi} \in \Pi$, then
  $V(\tilde{\pi}, m^\ast)\leq V(\pi^{\inf}, m^\ast) $.
\end{proposition}

We call this a ``safety'' property because the baseline policy acts
as a fallback option. If deviating from the baseline policy can lead
to a worse expected utility, a maximin policy will stick to the
baseline. In this way, the new policy will change the baseline only
when there is sufficient evidence for improvement.  We stress
that this safety property only holds if the structural assumptions
about the true model $m^\ast$ are correct, i.e., $m^\ast \in \mathcal{M}$.
Furthermore, this notion of safety is from the point of view of the
policy maker that sets the utility function: it says nothing
about the expected utility for other stakeholders with
different utility functions.

Furthermore, this safety property comes at a cost: maximin policies can be
conservative and sub-optimal relative to the (infeasible) oracle
policy, $\pi^\ast \in \argmax_{\pi \in \Pi} \; V(\pi, m^\ast)$
\cite[e.g.,][]{Manski2005_partial,Cui2021_partial}.
Because the maximin criterion limits the downside risks of deviating
from the baseline policy, it can miss situations where doing so could
lead to large utility gains.  We bound this sub-optimality at the
population level in Appendix Theorem~\ref{thm:pop_optimal_gap} and for
policies learned empirically from finite samples in
Theorem~\ref{thm:emp_optimal_gap}.
An alternative criterion that addresses
this is the \emph{minimax regret} criterion that measures the maximum
value difference between the (infeasible) oracle and the chosen policy
\citep[e.g.,][]{Manski2007,Stoye2012,Song2014_point_decision}.
In addition, maximin policies can be sensitive to the existence of
edge cases. Searching for the worst case across \emph{all} possible
models ignores the fact that we may find some models unlikely, even if
they are possible.  A Bayesian criterion that explicitly places prior
over models and computes the posterior expected utility given the
observed data would counteract this \citep{jia2023bayesian}.

\subsection{The empirical safe policy}
\label{sec:opt_data}

Next, we show
how to learn a policy from the observed data
$\{\bX_i, \tilde{\pi}(\bX_i), Y_i(\tilde{\pi}(\bX_i))\}_{i=1}^n$.  We begin
with a sample analog to the value function in
Eqn~\eqref{eq:value_model}:
\begin{equation}
  \label{eq:value_model_empirical}
  \hat{V}(\pi, m) \ = \ \frac{1}{n}\sum_{i=1}^n\sum_{a \in \mathcal{A}} \pi(\bX_i, a) \left\{
    u(a) \left[\tilde{\pi}(\bX_i, a) Y_i + \left\{1 - \tilde{\pi}(X_i, a)\right\}m(a, \bX_i) \right]+c(a)\right\}.
\end{equation}
With this, we could find the worst-case sample value across all models
in the restricted model class $\mcal{M}$ from
Eqn~\eqref{eq:restricted_class}.  Unfortunately, since we do not have
the \emph{true} conditional expectation
$\tilde{m}(\bx) = \E[Y(\tilde{\pi}(\bx))]$, we cannot compute the true
restricted model class.  One potential approach is to obtain an
estimator of the conditional expectation function,
$\hat{\tilde{m}}(\bx)$, and use the estimate in place of the true
values. However, this fails to take into account the estimation
uncertainty, and could lead to a policy that improperly deviates from
the baseline due to noise, especially when the convergence rate of the
estimated model $\hat{\tilde{m}}(\bx)$ is slow.

Instead, we construct a \emph{larger}, empirical model class
$\widehat{\mcal{M}}_n(\alpha)$, based on the observed data, that
contains the true restricted model class with a probability at least
$1 - \alpha$, i.e.,
$P\left(\mcal{M} \subseteq \widehat{\mcal{M}}_n(\alpha)\right) \geq 1
- \alpha.$
Then, we construct our empirical policies by first finding the
worst-case in-sample value improvement, then maximizing this objective
across policies $\pi$:
\begin{equation}
  \label{eq:maximin_emp}
  \hat{\pi}\in \underset{\pi \in \Pi}{\argmax} \min_{m \in \widehat{\mcal{M}}_n(\alpha)} \left\{\hat{V}(\pi, m) - \hat{V}(\tilde{\pi})\right\}.
\end{equation}
We refer to $\hat{\pi}$ as the empirical safe policy, as it is the empirical
analog to the $\pi^{\inf}$.
Note that since the empirical restricted
model class is larger than the true restricted model class, a policy
derived from it is more likely to fall back to the status quo rule.

To construct the empirical model class $\widehat{\mcal{M}}_n(\alpha)$,
we use a uniform $1-\alpha$ confidence band for the conditional
expectation function $\tilde{m}(\bx)$, with lower and upper bounds
$\widehat{C}_\alpha(\bx) = [\widehat{C}_{\alpha \ell} (\bx),
\widehat{C}_{\alpha u}(\bx)]$ such that
$  
P\left(\tilde{m}(\bx) \in \widehat{C}_\alpha(\bx) \;\; \forall \; \bx\right) \geq 1 - \alpha.
$
With such a confidence band, we construct the empirical restricted model class as
\[
  \widehat{\mcal{M}}_n(\alpha) \ = \ \{f \in \mcal{F} \; \mid f(\tilde{\pi}(\bx), \bx)  \in \widehat{C}_\alpha(\bx) \; \forall \bx \in \mathcal{X}\}.
\]
Throughout, we construct our confidence bands so that the 0\% confidence band corresponds to the point estimate: $\widehat{C}_{\alpha \ell}(\bx) = \widehat{C}_{\alpha u}(\bx) = \hat{\tilde{m}}(\bx)$, and therefore setting $\alpha = 1$ creates the restricted model class directly from the point estimates as described above.
In our analysis in Section~\ref{sec:application}, the covariates are
all discrete.  Thus, we first construct a point-wise confidence
interval for each unique data point, and then create a uniform
confidence band by using a Bonferonni correction for the number of
unique data points.
We discuss how to construct the empirical model class and solve this
optimization problem in Section~\ref{sec:application}.

\subsection{Finite sample statistical properties}
\label{sec:stats_theory}
Compared to the population maximin problem, the empirical problem has
an additional layer of uncertainty due to sampling error that arises in finite samples.
First, we establish a statistical safety property: if the
structural assumptions about the true model $m^\ast$ are correct,
the learned policy will perform at least as well as the baseline policy
with probability approximately $1 - \alpha$.
We then characterize how conservative the solution is via the {\it
  optimality gap}, $V(\pi^\ast) - V(\hat{\pi})$: the policy value
difference between the infeasible oracle that knows the true model,
and our data-driven maximin policy that uses the worst-case model.

The results below use the \emph{population Rademacher complexity} of a
function class $\mathcal{G}$:
$
\mcal{R}_n(\mathcal{G})  \equiv \E\left[\sup_{g \in \mathcal{G}}\left|\frac{1}{n}\sum_{i=1}^n \varepsilon_i g(\bX_i)\right|\right]
$
where $\varepsilon_i$ is an i.i.d. Rademacher random variable, i.e.,
$\Pr(\varepsilon_i=1)=\Pr(\varepsilon_i=-1)=1/2$, and the expectation
is taken over both $\varepsilon_i$ and $\bX_i$ \citep[ \S
4]{wainwright_2019}.  We consider the maximum Rademacher complexity
across the sub-policy classes for actions $a \in \mathcal{A}$:
$\Pi_a \equiv \{\pi(\cdot, a ) \mid \pi \in \Pi\}$.
This measures the ability of the policy class to overfit.
Using this measure, we establish a statistical safety property.
\begin{theorem}[Statistical safety]
  \label{thm:emp_safety} \singlespacing
  If the baseline policy
  $\tilde{\pi} \in \Pi$ and the true conditional expectation
  $m^\ast(a,\bx) \in \mcal{M}$, for any $0 < \delta \leq e^{-1}$, the value of
  $\hat{\pi}$ relative to the baseline $\tilde{\pi}$ is,
  \[
    V(\tilde{\pi})\ - V(\hat{\pi})\leq   \ 6 C (K-1) \left[\max_{a} \mathcal{R}_n(\Pi_a) + 2\sqrt{\frac{1}{n}\log \frac{K-1}{\delta}} \right],
  \]
  with probability at least $1 - \alpha  - \delta$,
  where $C = \max_{y \in \{0,1\}, a \in \{0, 1\}} |u(y, a)|$.
\end{theorem}

Like Proposition~\ref{prop:pop_safety}, Theorem~\ref{thm:emp_safety} is only meaningful
if the assumptions about the true model $m^\ast$ are correct.
If they are, Theorem~\ref{thm:emp_safety} shows that the empirical safe policy
will not have a lower policy value than the baseline, up to standard
empirical process terms:
the Rademacher complexity of the policy class $\Pi$, and an error term
due to sampling variability that decreases at a rate of $n^{-1/2}$.
The complexity of the policy class $\Pi_a$ controls the chance that the learned
policy is worse than the baseline due to overfitting.

For many standard
policy classes, we expect the Rademacher complexity to decrease to
zero as the sample size increases, with the complexity determining the
rate of convergence.  For simple policy classes, the bound will
quickly go towards zero for any level $\alpha$; complex policy classes
will require larger samples to ensure that the safety property is
meaningful, regardless of the level $\alpha$.
By using the larger model class
$\widehat{\mcal{M}}_n(\alpha)$, the estimation error for the
conditional expectation $\hat{\tilde{m}}(\bx) - \tilde{m}(\bx)$ does not
directly enter into the bound.\footnote{In Appendix~\ref{sec:alpha_1}, we extend these results to consider the case where $1 - \alpha = 0$ and we use point estimates rather than confidence bounds. We show that the bounds have additional terms  due to estimation error of the model.} However, if we cannot estimate
$\tilde{m}(\bx)$ well, the empirical restricted model class
$\widehat{\mcal{M}}_n(\alpha)$ will be large, and so the empirical
safe policy may collapse to the baseline policy.

To quantify the optimality gap, we denote
$\widehat{\mcal{W}}_{\widehat{\mcal{M}}_n(\alpha)}(\pi^\ast(1-\tilde{\pi}))$
as the width of the empirical
model class $\widehat{\mcal{M}}_n(\alpha)$ in the direction that
$\pi^\ast$ and $\tilde{\pi}$ disagree, where
\[
  \widehat{\mcal{W}}_\mcal{F}(g) = \sup_{f \in \mcal{F}} \frac{1}{n}\sum_{i=1}^n \sum_{a \in \mathcal{A}} f(a, \bX_i) g(a, \bX_i) - \inf_{f \in \mcal{F}} \frac{1}{n}\sum_{i=1}^n\sum_{a \in \mathcal{A}} f(a, \bX_i) g(a, \bX_i)
\]
is the usual notion of the width of a set, for the set defined by all
possible values of a function $f \in \mathcal{F}$ at the
data points $\bX_1,\dots,\bX_n$ for actions $a \in \mathcal{A}$, in the direction
defined by the vector of all values of another function $g(a, \bX_i)$.

\begin{theorem}[Optimality gap]
  \label{thm:emp_optimal_gap} \singlespacing
  Let $u(a) = u > 0$ for all actions.
  If the true conditional expectation $m^\ast \in \mcal{M}$, then for any $0 < \delta \leq e^{-1}$ the optimality gap is
  \[
    \begin{aligned}
      V(\pi^\ast) - V(\hat{\pi}) & 
      & \ \leq \ 2C \widehat{\mcal{W}}_{\widehat{\mcal{M}}_n(\alpha)}\left(\pi^\ast (1 - \tilde{\pi})\right)  +6 C (K-1) \left[\max_{a} \mathcal{R}_n(\Pi_a) + 2\sqrt{\frac{1}{n}\log \frac{K-1}{\delta}} \right],
    \end{aligned}
    \]
    with probability at least $1 - \alpha  - \delta$, where $C = \max_{y \in \{0,1\}, a \in \{0, 1\}} |u(y, a)|$.
\end{theorem}
\noindent To simplify the statement, we have assumed that the utility gain
across different actions is constant and, without loss of generality,
positive.

The bound on the empirical
optimality gap contains the width term
$\widehat{\mcal{W}}_{\widehat{\mcal{M}}_n(\alpha)}\left(\pi^\ast (1 - \tilde{\pi})\right)$,
in addition to the standard empirical process terms found in Theorem~\ref{thm:emp_safety}.
If the baseline policy is the oracle policy, then
this width is zero, the bounds in Theorems~\ref{thm:emp_safety} and
\ref{thm:emp_optimal_gap} coincide,
and the regret of $\hat{\pi}$ relative to the oracle $\pi^\ast$ will converge
to zero so long as the complexity of the policy class goes to zero.
Otherwise, the width term does not necessarily converge to zero: if
the baseline and oracle policies disagree for many cases, the empirical safe
policy could perform substantially worse than the oracle.

This leads to a tradeoff between statistical safety
(Theorem~\ref{thm:emp_safety}) and optimality
(Theorem~\ref{thm:emp_optimal_gap}).  Increasing the confidence
level will yield a greater probability that the learned
policy is safe relative to the baseline, but it will also widen the
potential optimality gap when the baseline and oracle policies
disagree.  This is similar to the tradeoff between a low type I
error rate ($\alpha$ low) and high power
($\widehat{\mcal{W}}_{\widehat{\mcal{M}}_n(\alpha)}\left(\pi^\ast (1
  - \tilde{\pi})\right)$ low) in hypothesis testing.  The trade-off
extends to the choice of model class as well: statistical safety
requires that the model class contains the true conditional
expectation function, i.e., $m^\ast \in \mathcal{M}$.  This is
palatable if we choose a complex model class, but complex model
classes may lead to a greater amount of uncertainty due to severe
lack of identification and/or greater estimation error.

This tradeoff does not exist if the baseline policy is stochastic and
there is overlap between actions. In this case, the conditional
expectation function is non-parametrically identifiable.  While we can
still account for statistical uncertainty by constructing the
empirical model class $\widehat{\mathcal{M}}_n(\alpha)$, we stress
that our approach is not appropriate when the baseline policy is
stochastic. It only uses a model for the outcomes and so will rely on
stronger assumptions on the outcome model and be inefficient relative
to a doubly robust approach that incorporates the action probabilities
as proposed by \citet{Athey2021}.

In practice, we do not know the oracle policy. To operationalize
the bound in Theorem~\ref{thm:emp_optimal_gap}, we can further upper
bound the optimality gap by finding the policy that leads to the
\emph{worst-case} width, were it the oracle policy:
$\widehat{\mathcal{S}}(\mathcal{F}, \Pi; \tilde{\pi}) \equiv
\sup_{\pi \in \Pi}
\widehat{\mcal{W}}_{\widehat{\mcal{M}}_n(\alpha)}\left(\pi (1 -
  \tilde{\pi})\right)$.  We refer to this quantity as the ``size''
of the empirical restricted model class because it measures the
degree of uncertainty about the true model $m^\ast$ in regions of
the covariate space where a policy $\pi \in \Pi$ could deviate from
the baseline.

We use this as a diagnostic measure in
Section~\ref{sec:application}.  Note that the policy class $\Pi$
affects the size.  Restricting to policies that can only disagree
with the baseline in only a few cases will lead to a small
size. Conversely, if we attempt to optimize over an expansive policy
class, the size diagnostic can be large.  However, the size term is
a loose upper bound: even if the size is large, the optimality gap
may still be small if it happens that the oracle policy $\pi^\ast$
is similar to the baseline $\tilde{\pi}$. Therefore, a large size
term is a warning that there may be insufficient information to
learn an improved policy, but it does not rule it out entirely.

\subsection{Learning from experiments evaluating a deterministic policy}
\label{sec:ab_test}

In our empirical study, the existing PSA-DMF system was compared
to not providing algorithmic recommendations.  While a primary goal
of this randomized controlled trial was to evaluate whether one
should adopt the algorithmic policy, we can leverage the control
group data to weaken the restrictions of the underlying model class
$\mathcal{M}$ by placing assumptions on \emph{treatment effects}
rather than the expected potential outcomes.

We consider an expanded set of actions that includes all actions
in $\mathcal{A}$ and a ``null'' action (i.e., do not provide an
algorithmic recommendation).  We denote the null action as $a = -1$,
with potential outcome $Y(-1)$.
Let $Z_i \in \{0,1\}$ be a treatment assignment indicator where
$Z_i = 0$ if no policy is enacted (i.e., the null policy), and
$Z_i = 1$ if the policy follows the baseline policy $\tilde{\pi}$.
Let $e(x) = P(Z = 1 \mid \bX = \bx)$ be the probability of assigning
the treatment condition for an individual with covariates $\bx$.  This
is the propensity score \emph{for the treatment assignment} and since
this is an experiment, it is known and strictly between 0 and 1.
While we consider general propensity scores when describing the
method, in our experiment $e(x) = 0.5$ for all cases.  This allows
us to identify the conditional expectation function,
$m^\ast(-1, \bx) = \E[Y \mid \bX = \bx, Z = 0]$.  Defining the true
conditional average treatment effect (CATE) of the action $a$ relative
to the null action $-1$ as
$\tau^\ast(a, \bx) \equiv m^\ast(a,\bx) - m^\ast(-1,\bx)$, we can also
identify the CATE under the baseline policy $\tilde{\pi}(\bx)$,
$\tilde{\tau}(\bx) = \tau^\ast(\tilde{\pi}(\bx),\bx)$.

We now write the value function in terms of the (partially-identified)
CATE and the (point-identified) expected outcome under the null
action. With a constant utility gain $u(a) = u$, we can write it
as:\footnote{Proposition~\ref{prop:ab_test_utility} in the Appendix
  shows this result for the general utility case.}
\[
  V(\pi) = \E\left[\sum_{a \in
    \mathcal{A}} \pi(\bX, a) \{u \cdot \tau^\ast(a, \bX) + c(a)\}\right] + u \cdot \E[m^\ast(-1,\bX)].
\]
Because the baseline term $\E[m^\ast(-1,\bX)]$ does not depend on $\pi$
and is point-identified, we can re-parameterize the model class to
impose restrictions on the \emph{treatment effects}
$\mcal{T} = \{f(a,\bx) \equiv m^\ast(-1, \bx) + h(a, \bx) \mid h \in
\mathcal{F}, h(\tilde{\pi}(\bx), \bx) = \tau(\tilde{\pi}(\bx), \bx)\}$.
The CATE function is sometimes assumed to be simpler (e.g. smoother, sparser,
fewer interaction terms) than the conditional expected potential outcome
\citep[see, e.g.,][who argue that the CATE should be easier to estimate]{Kunzel2019}.
Therefore, we may consider a smaller model
class for the treatment effects than for the baseline outcomes,
leading to a smaller optimality gap in
Theorem~\ref{thm:emp_optimal_gap}.  We can also construct the
empirical analog by creating a larger empirical model class
$\widehat{\mcal{T}}_n(\alpha)$ as in Section~\ref{sec:opt_data}.

Finally, following \citet{Kitagawa2018}, to account for potential
unequal assignment into treatment, we can solve the population and
empirical robust optimization problems using the inverse probability
weighted outcome
$\Gamma(Z, \bX, Y) \equiv Y\{Z(1-2e(\bX)) + e(\bX)\}/\{e(\bX)(1 -
e(\bX))\}$, which equals the conditional expected potential outcome in
expectation, i.e.,
$\E[\Gamma(Z, \bX, Y) \mid Z = z, \bX = \bx] = z \cdot
m^\ast(\tilde{\pi}(\bx), \bx) + (1-z) \cdot m^\ast(-1, \bx)$.

\section{Empirical Analysis of the Pre-trial Risk Assessment}
\label{sec:application}

\subsection{Implementation details}
\label{sec:implement}

For our empirical analysis, we will represent the empirical restricted model
classes as the set of functions that are upper and lower bounded
point-wise by two bounding functions,
 $ \widehat{\mcal{T}}_n(\alpha)  \ = \  \{f: \mathcal{A} \times \mcal{\bX} \to \R \; \mid \; \widehat{B}_{\alpha\ell}(a,\bx) \leq f(a, \bx) \leq \widehat{B}_{\alpha u}(a, \bx)\},$
where the upper and lower bounds are chosen to satisfy
$P\left(\mcal{T} \subseteq \widehat{\mcal{T}}_n(\alpha)\right) \geq 1
- \alpha$.  In Appendix~\ref{sec:representative_cases}, we show how to
compute these bounds using simultaneous confidence intervals when the
underlying model class is the set of Lipschitz functions or linear models.

The point-wise bound allows us to solve for the worst-case empirical
value $\hat{V}^{\inf}(\pi)$ by finding the minimal value for each
action-covariate pair \citep[see][]{Pu2021}.  Finding the empirical
safe policy by solving Eqn~\eqref{eq:maximin_emp} is equivalent to
solving an empirical welfare maximization problem using a
quasi-outcome that imputes the counterfactual outcome with the
lower bound when the action disagrees with the baseline policy:
{\small\begin{equation}
  \label{eq:maximin_emp_bnd}
  \begin{aligned}
    \max_{\pi \in \Pi} & \;\; \frac{1}{n}\sum_{i=1}^n
    \sum_{a \in \mathcal{A}} \pi(\bX_i, a)\left\{u(a) \left[
    \tilde{\pi}(\bX_i, a) \{\Gamma(1,\bX_i,Y_i) - \Gamma(0, \bX_i, Y_i)\}+ \{1 - \tilde{\pi}(\bX_i, a)\} \widehat{B}_{\alpha \ell}(a, \bX_i)\right] + c(a)\right\},
  \end{aligned}
\end{equation}}where we have omitted terms that do not depend on $\pi$.  A similar
implementation strategy is applicable to cases where we model
potential outcomes rather than treatment effects.

\subsection{Learning a new NVCA flag threshold}
\label{sec:nvca_threshold}

We begin our analysis by considering a small change to the existing
system: learning a new threshold for the NVCA flag.  Our goal here is
to find the optimal NVCA threshold in the worst case, where our
preferred outcome is no NVCA.

\paragraph{Choosing the policy class.}
We first formalize our choice of policy class.  Let
$x_\text{nvca} \in \{0,\ldots,6\}$ be the total number of NVCA points for an
arrestee, computed using the point system in
Appendix Table~\ref{tab:nvca_weights}.  Recall that the baseline NVCA algorithm
is to trigger the flag if the number of points is greater than or
equal to 4, i.e.,
$\tilde{\pi}(x_\text{nvca}) = \bbone\{x_\text{nvca} \geq 4\}$.  Our
policy learning problem is to choose a policy among the class of
threshold policies,
$ 
  \Pi_\text{thresh} = \left\{ \pi(x) = \bbone\{x_\text{nvca} \geq \eta \} \; \vert \; \eta \in \{0,\ldots,7\}\right\}.
$
We will keep the baseline weighting on arrestee risk factors and
\emph{only} change the threshold $\eta$.
Since this policy class only has eight elements, we can compute the
empirical maximin policy $\hat{\pi}$ by solving
Eqn~\eqref{eq:maximin_emp_bnd} via an exhaustive search.

\paragraph{Choosing the model class.}
We next choose a model class for the CATE on no NVCA occurring,
$\tau^\ast(a, x_\text{nvca})$.  There are many potential ways to
characterize the complexity of functions of one variable such as
$\tau^\ast(a, \cdot)$. Here, we characterize it via a Lipschitz
constraint that
$\left|\tau(a, x_\text{nvca}) - \tau(a, x_\text{nvca}')\right| \leq
\lambda_a |x_\text{nvca} - x_\text{nvca}'|$ for any pair of NVCA
points $x_\text{nvca}, x_\text{nvca}'$.

To construct the empirical restricted model class, we set the level to
$1 - \alpha = 0.8$, allowing some tolerance for statistical
uncertainty and construct a simultaneous 80\% confidence interval for
the CATE via a Bonferroni correction for the 7 unique values (see
Appendix~\ref{sec:lip} for details on computing the bounds).  We also
restrict the treatment effects to be bounded between $-1$ and 1, since
the outcome is binary.\footnote{This is not the tightest possible
  bound, since the restriction is that
  $0 \leq m(-1,x) + \tau(a,x) \leq 1$. To incorporate the uncertainty
  in estimating $m(-1,x)$ in finite samples we could use analogous
  techniques to those in Section~\ref{sec:opt_data}; we leave this to
  future work.}

For this model class, we need to specify the Lipschitz constants for
the CATE when the flag is and is not triggered ($\lambda_1$ for
$\tau(1,x_\text{nvca}))$ and $\lambda_0$ for $\tau(0, x_\text{nvca})$,
respectively).  We adapt a suggestion from \citet{Imbens2019_optrdd}
for model classes with a bounded second derivative to the Lipschitz
case.  We estimate the CATE function by taking the difference in NVCA
rates with and without provision of the PSA at each level of
$x_\text{nvca}$. Then, we measure the largest consecutive difference
between CATE estimates (0.016 and 0.433 for $a = 0, 1$, respectively).
Finally, we set the Lipschitz constants to be a constant multiple $C$
of this difference yielding $\lambda_0 = C \times 0.016$ and
$\lambda_1 = C \times 0.433$.
Setting $C = 1$ gives the smallest Lipschitz constants supported by the data;
increasing $C$ will be more conservative.

\paragraph{Choosing the utility function.}
Recall that in our parameterization we must define the difference in
utilities when there is and is not an NVCA, $u(a) = u(1,a) - u(0,a)$,
for both actions $a \in \{0,1\}$.
This captures the benefits of avoiding an NVCA.
Countering this benefit is the baseline cost of action $a$, $c(a) = u(0,a)$.
The marginal monetary cost of  triggering the NVCA flag is zero
given the initial fixed cost of collecting the data for the PSA.
However, to the extent
that triggering the NVCA flag increases the likelihood of pre-trial
detention, it will lead to an increase in fiscal costs --- e.g.,
housing, security, and transportation --- for the
jurisdiction. Furthermore, there are potential socioeconomic costs to
the defendant and their community that balance against potential benefits from
avoiding more criminal activity.

To represent these costs, we will place zero cost on not triggering
the NVCA flag, $c(0) = 0$, and a cost of 1 on triggering the flag,
$c(1) = -1$. We then assign an equal utility gain from avoiding an
NVCA, $u(1) = u(0) = u$ (equivalently, the cost of an NVCA is $-u$).
This yields a utility function of the form $u(y,a) = u \times y - a$,
where $u$ is the ratio of the cost of an NVCA to the cost of
triggering the flag.  Choosing a particular value of $u$ is outside
the scope of this paper and indeed would be inappropriate for us to
do: the choice depends on societal preferences and should be arrived
at in a collaborative process between policy-makers in the criminal
justice system and the communities impacted by it.  Instead, we
examine how adjusting the ratio $u$ affects the policies we learn.\footnote{
Note that mathematically one could use a negative cost
of triggering the flag, but this would encourage triggering the flag even if
it would not avoid an NVCA.}

\begin{figure}[t!]
  \centering \vspace{-.25in}
   \vspace{-.2in}
  \begin{subfigure}[t]{0.45\textwidth}  
    \caption{} 
    {\centering \includegraphics[width=\textwidth]{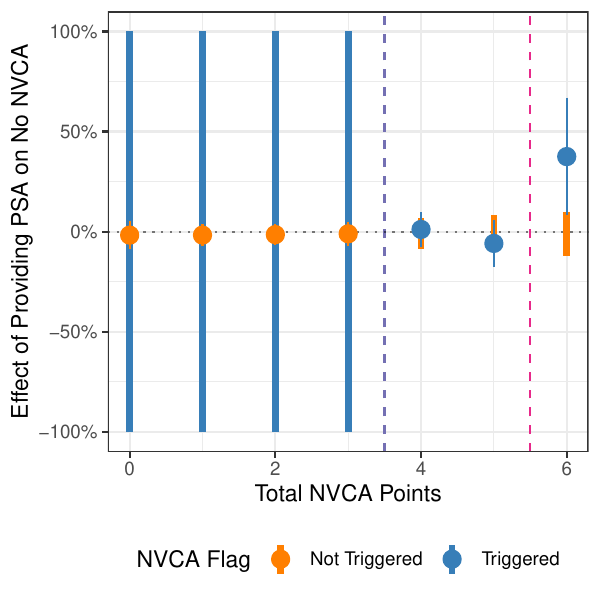}
    }
    \label{fig:nvca_threshold}
     \end{subfigure}
     \begin{subfigure}[t]{0.45\textwidth}  
      \caption{}
      {\centering \includegraphics[width=\textwidth]{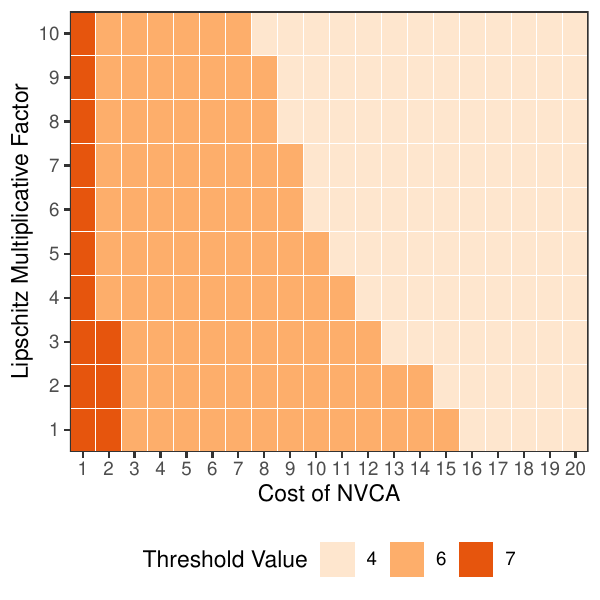} 
      }
      
        \label{fig:threshold_lip}

       \end{subfigure}\quad
   \vspace{-.2in}
   \caption{Learning a new NVCA flag threshold. 
    (a) Empirical restricted model class and maximin
    threshold with a Lipschitz multiplicative factor of $C = 3$.
    The points and thin lines around them are point estimates and a
    simultaneous 80\% confidence interval for the partial CATE
    function $\tau(\tilde{\pi}(x_\text{nvca}), x_\text{nvca})$ when
    the NVCA flag is not triggered ($\tilde{\pi}(x_\text{nvca}) = 0$,
    in orange) and is triggered ($\tilde{\pi}(x_\text{nvca}) = 1$, in
    blue).  The thick solid lines represent the partial identification
    set for the unobservable components of the CATE,
    $\tau(1, x_\text{nvca})$ for $x_\text{nvca} < 4$ and
    $\tau(0, x_\text{nvca})$ for $x_\text{nvca} \geq 4$. The purple
    dashed line represents the baseline policy of triggering the flag
    when $x_\text{nvca} \geq 4$, and the pink dashed line is the
    empirical safe policy that only triggers the flag when
    $x_\text{nvca} \geq 6$. 
    (b) Maximin threshold values solving Eqn~\eqref{eq:maximin_emp_bnd} for the
     NVCA flag threshold rule with a level of $1 - \alpha = 80\%$
     as the cost of an NVCA increases from
     1 to 20 times of the cost of triggering the NVCA flag, and the
     multiplicative factor on the estimated Lipschitz constant varies
     from 1 to 10.  }
\end{figure}

\paragraph{Learning a maximin NVCA threshold.}

Figure~\ref{fig:nvca_threshold} presents the empirical restricted
model class with a particular multiplicative constant of $C=3$ by
showing point estimates and simultaneous 80\% confidence intervals for
the observable component of the CATE function
$\tau^\ast(\tilde{\pi}(x_\text{nvca}), x_\text{nvca})$ and the partial
identification set for the unobservable component. There is
substantially more information when extrapolating the CATE for the
case that the NVCA flag is triggered. This is because the point
estimates do not vary much with the NVCA points, leading to a small
Lipschitz constant.  On the other hand, when extrapolating in the
other direction, there is a large jump in the point estimates between
$x_\text{nvca} = 5$ and $x_\text{nvca} = 6$, leading to a large
Lipschitz constant.  This means that the empirical restricted model
class puts essentially no restrictions on
$\tau^\ast(1, x_\text{nvca})$ for $x_\text{nvca} < 4$.

Estimating this policy requires choosing the Lipschitz
multiplicative factor $C \geq
1$.
We fit the empirical safe policy across a range of values to see if
the results are stable.  Figure~\ref{fig:threshold_lip} shows the
learned thresholds as we vary both the relative cost $u$ of an NVCA
in the utility function and the
multiplicative factor $C$.  When the cost of an NVCA is $u \leq 7$, the data support increasing the threshold to at least 6 even
in the worst case and even with $C = 10$, only triggering the flag for
arrestees with the observed maximum of 6 total NVCA points. The
results for larger costs are more sensitive to the choice of $C$, and
the learned threshold collapses back to the baseline of 4 for
intermediate choices of $C$.

Raising the threshold to $\eta = 6$ is a much more lenient policy than the
status quo,
reducing the number of arrestees flagged as at risk of an
NVCA by 95\%.
We find evidence for such a large change because there is no meaningful effect of
providing the PSA on the absence of an NVCA, except for those arrestees who have
the maximum of 6 points (Figure~\ref{fig:nvca_threshold}).
One possible reason for these small effects is that the judge's behavior is not affected.
This appears to be the case when $x_\text{nvca} \leq 4$: there is little effect
on the judge's bail decision in these cases.
However, for $x_\text{nvca} > 4$, providing the PSA increases cash bail decisions
by over 30 pp (see Appendix~\ref{sec:addl_nvca} for further discussion).
This suggests that PSA provision is leading to additional bail decisions without
a requisite decrease in NVCAs for $x_\text{nvca} = 5$.

Thus, even in the worst case, the threshold could be raised to
$\eta = 6$ without an increase in the NVCA rate that outweighs costs
from triggering the flag.  As we increase the cost of an NVCA,
however, at some point (e.g., $u \geq 13$ for $C = 3$), the cost
becomes large enough, making the empirical safe policy revert to the
status quo with the threshold at $\eta = 4$.

\subsection{Learning new FTA, NCA, and NVCA risk scoring rules}
\label{sec:risk_scores}

We next turn to constructing new, maximin optimal FTA, NCA, and NVCA
risk scores. For each risk score, we focus on the \emph{absence} of
the corresponding negative outcome.

\paragraph{Choosing the policy classes.}
A key consideration is the form of the policy classes used for each
risk score.  One possibility is to allow the policies to be flexible
functions of all the information available in the system.
Although the oracle policy may have a high expected utility in
this case, in finite samples a complex policy can over-fit and
reduce the quality of the safety property in
Theorem~\ref{thm:emp_safety}.  In addition, the oracle policy may be
substantially different from the baseline policy, leading to a large
optimality gap in Theorem~\ref{thm:emp_optimal_gap}.  Lastly, in
real-world applications, policy makers might be reluctant to adapt
the existing system to an entirely new policy. For these reasons,
we use the same set of risk factors and focus on changing the weight
applied to each risk factor (see Appendix Table~\ref{tab:nvca_weights}).

For each risk score, we formally describe the status quo rule as
consisting of a vector of integer-valued weights $\tilde{\boldsymbol\theta}$ on
the risk factors $\bx$, and a mapping from the linear combination of the
risk factors $\tilde{\boldsymbol\theta}^\top \bx$ to the $K$ risk levels via
thresholds.
We consider the policy class that consists of all
possible vectors of integer-valued weights and all possible
thresholds, restricting to inter-valued weights in order to mimic the structure
of the existing risk scores.
For example, recall that the NVCA flag has $K=2$ risk
levels (a flag for elevated NVCA risk), 7 binary risk factors, and the
baseline policy is
$\tilde{\pi}(\bx) = \bbone\left\{\sum_{j=1}^7 \tilde{\theta}_j x_j \geq
  4\right\}$.
We then write the corresponding NVCA flag policy class
as:
\begin{equation}
  \label{eq:integer_weight_policy}
  \Pi_{\text{nvca}} = \left\{\pi(\bx) = \bbone\left\{\sum_{j=1}^7 \theta_j x_j \geq \eta \right\} \; \left \vert  \vphantom{\sum_{j=1}^7 }\right. \; \theta_j \in \Z, \eta \geq 0 \right\}.
\end{equation}

This policy class includes the original NVCA flag rule as a special
case.  We can construct the policy classes for the FTA and NCA rules
similarly by including multiple thresholds (see
Appendix~\ref{sec:addl_fta_nca} for a formal definition).  To simplify
comparisons to the status quo and avoid identifiability issues, we
will primarily constrain the thresholds $\eta$ to be equal to the
status quo values.  This allows us to understand any differences from
the status quo rule by comparing the learned weight vector to the
baseline weight vector $\tilde{\boldsymbol\theta}$.  With this policy class, the
optimization problem is a mixed integer program; we solve this with
the Gurobi solver.

\paragraph{Choosing the model class.}
In contrast to changing only the NVCA threshold above, here the CATE is a
function of multiple binary variables. A natural way to characterize
the complexity of such models is by the number and strength of
interaction terms between the variables.  We focus on the two simplest
models: an additive effect model
$\mcal{T}_{\text{add}} \equiv \left\{\tau(a, \bx) =
  \sum_{j}\tau_{aj}x_j\right\}$ and a second order effect model
$\mcal{T}_{\text{two}} \equiv \left\{\tau(a, \bx) = \sum_{j}\sum_{k <
    j} \tau_{ajk}x_jx_k \right\}$. Because the covariates are
discrete, we can write these using linear models.  We again restrict
the treatment effects to be bounded between $-1$ and 1.  These two
model classes lead to different restrictions and ultimately affect
what policies we learn from the experiment (see Appendix~\ref{sec:glm}
for details). This is partly because even with infinite data the
models may not be identifiable.  But it is also because with finite
data there is a different amount of uncertainty in each model class.

\begin{figure}[t!]
  \centering \vspace{-.25in}
    \includegraphics[width = 0.8\maxwidth]{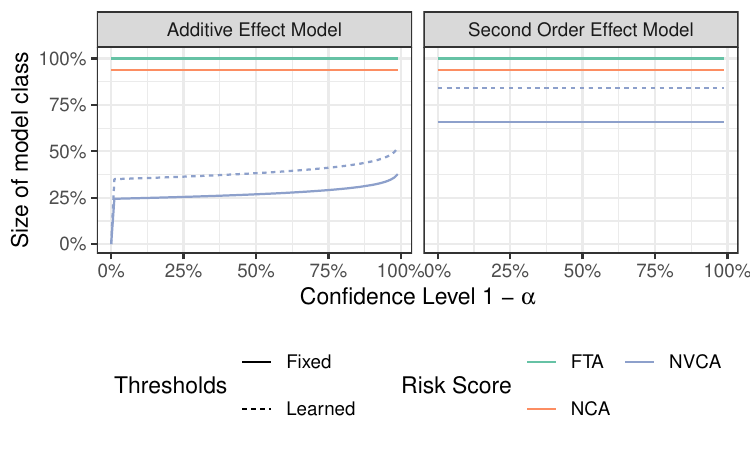}
    \vspace{-.25in}
    \caption{The size (as a percentage of its maximum
      value) of two different model classes with respect to the
      linear threshold policy class versus the confidence
      level $1-\alpha$ for the FTA (green), NCA (orange), and NVCA (purple)
      scoring rules. The dashed purple line shows the size for the NVCA model
      class when the threshold is included as a decision variable and learned
      in addition to the weights.}
  \label{fig:nvca_widths}
\end{figure}

To diagnose the amount of information available in each model class,
we use the size measure $\widehat{\mathcal{S}}(\widehat{\mathcal{T}}_n(\alpha), \Pi; \tilde{\pi})$.
Figure~\ref{fig:nvca_widths} depicts this information by showing how
the size of the model class (vertical axis), changes with the desired
confidence level $1 - \alpha$ (horizontal axis) for each risk score
and model class. We
also show the difference in the size for the NVCA rule when fixing the
threshold to the existing value versus including it as a decision variable.

There is a stark difference in the amount of information between the
different risk scores within the same model class.  Under the additive
model for the NVCA rule, the size is zero when the confidence level is
zero, implying that this model is identifiable. This is due to the
structure of the NVCA flag rule: for any given value of a covariate,
it is possible to observe cases with the flag set to zero or one. When
accounting for the statistical uncertainty, the size increases, but it
is substantially smaller than the size for the FTA and NCA rules, both
of which are near or at the maximum value of 2.  Because these risk
scores have 6 levels, we would need to observe cases with all 6 levels
for any given value of a covariate in order to identify the additive
model.  Overall, there is little information available to support changing
the FTA and NCA scoring rules.

Turning to the second order treatment effect model for the NVCA score,
which makes weaker assumptions, we find that it is likely too large a
class for us to learn a new NVCA rule, with roughly twice the size as
for the additive effect model.  This is because there are several
pairs of variables that always trigger the NVCA flag (e.g., if
both the current offense is violent and the arrestee has 3 or more
prior violent convictions).
Finally, we observe that increasing the
flexibility of the policy class by including the threshold as a
decision variable rather than keeping it fixed increases the size
because it is a function of \emph{both} the model class and the policy
class.

These diagnostics point to focusing on the NVCA score with an additive
effect model. There is likely not enough information to make any changes to
the NCA and FTA scores under either model, and the 2\super{nd} order effect
model for the NVCA flag is not well enough identified. 
However, in Appendix~\ref{sec:addl_nvca}, we find some
evidence for the existence of interactions for the NVCA score via classical
model testing procedures.
Therefore, we caution
over-interpreting our results.  For completeness, we show these results
in Appendix~\ref{sec:addl_results} and indeed find that the optimal
solution for the worst case is to not deviate from the status quo
rules.

\paragraph{Choosing the utility function.}
We use the same utility parameterization as
in Section~\ref{sec:nvca_threshold}.  For this value function,
the marginal decrease in the utility from triggering the flag
is constant regardless of the proportion of arrestees that are
classified as an NVCA risk. However, higher
levels of pre-trial incarceration can have additional negative impacts
on the community above and beyond the cost to the individual. In
Appendix~\ref{sec:addl_nvca}, we include an
additional penalty to triggering the NVCA flag that scales with the
proportion of arrestees classified as being at risk.

\paragraph{Learning a maximin NVCA flag.}

\begin{figure}[t!]
  \centering\vspace{-.5in}
        \begin{subfigure}[t]{0.45\textwidth}  
    \caption{Varying the confidence level}
    {\centering \includegraphics[width=\maxwidth]{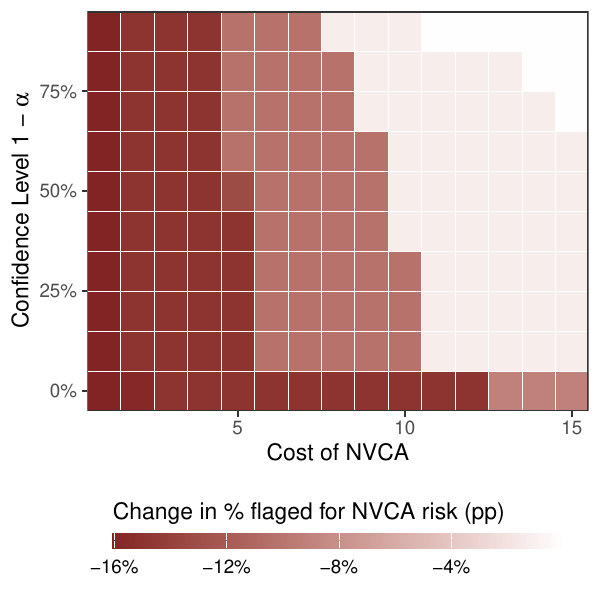} 
    }
    \label{fig:pct_diff_lvl}
  \end{subfigure}
    \begin{subfigure}[t]{0.5\textwidth}  
   \caption{Maximin NVCA flag weights} 
  {\centering \includegraphics[width=\textwidth]{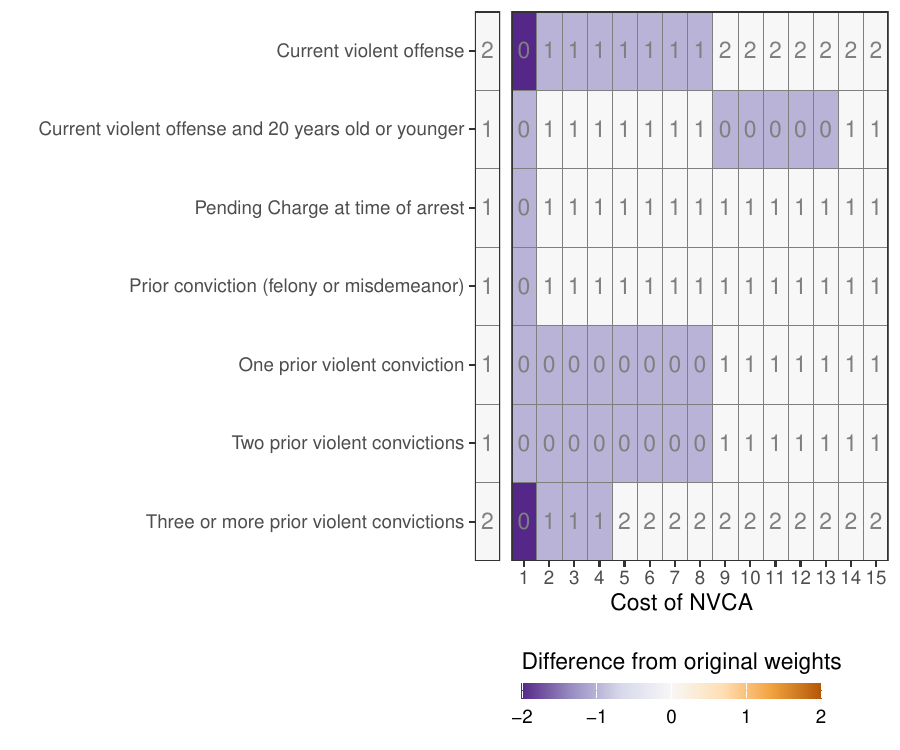} 
  }
    \label{fig:coef_by_cost}
    \end{subfigure}\quad
  \vspace{-.2in}
  \caption{(a) The percentage point difference in the proportion of
    arrestees flagged for NVCA risk between the maximin policy and the
    original NVCA score as the cost of an NVCA increases from 1 to 15
    times of the cost of triggering the NVCA flag and the confidence
    level varies between 0\% and 100\%.  (b) Change in Maximin NVCA
    flag weights $\theta$ (in Eqn~\eqref{eq:integer_weight_policy}) as
    the cost of an NVCA increases from 1 to 15 times the cost of
    triggering the NVCA flag, at a confidence level of
    $1 - \alpha = 80\%$.}
  \label{fig:nvca_results_lvl}
\end{figure}

Figure~\ref{fig:pct_diff_lvl} presents the changes to the original
rule made by the maximin policy that solves the optimization problem
given in Eqn~\eqref{eq:maximin_emp_bnd} under the additive treatment
effect class $\mcal{T}_{\text{add}}$.  The changes are shown in terms
of the proportion of arrestees flagged for an NVCA risk as we vary the
cost of an NVCA $-u$ and the confidence level $1 - \alpha$.  Across
every confidence level, the maximin policy differs less and less from
the original rule as the cost of an NVCA increases, moving from never
triggering the flag at a 1--1 cost to eventually collapsing back to the
status quo if the cost crosses an $\alpha$ dependent threshold.  For a
given cost of an NVCA, policies at lower confidence levels are more
aggressive in deviating from the original rule, prioritizing a
potentially lower regret relative to the (infeasible) optimal policy
at the expense of a higher chance that the new policy is worse than
the original rule.\footnote{Except for when the cost of an NVCA is
  greater than 12 and the confidence level is 0\%, the maximin
  policies do not trigger the flag when the original rule does not.}

Figure~\ref{fig:coef_by_cost} shows the integer weights on the risk
factors for the maximin policy at the $1 - \alpha = 80\%$ level as the
cost of an NVCA increases. In the limiting setting where an NVCA is
given the same cost as triggering the NVCA flag, the maximin policy
never triggers the flag because it cannot be worth the cost.  Once the
cost is at least 14 times the cost of triggering the flag, the learned
policy reduces to the original rule.  In light of the sizes shown in
Figure~\ref{fig:nvca_widths}, this behavior is primarily due to
increased uncertainty in the effect of triggering the NVCA flag.  If
the policy maker were to set the cost of an NVCA above a certain
point, any change in the policy would be too risky to act upon.  For
intermediate values, the learned policy places less weight on the
number of prior violent convictions and whether the current offense is
violent than the original rule.
In Appendix~\ref{sec:addl_nvca}, we consider a more flexible policy
class that includes additional risk factors.

\subsection{Learning a new DMF matrix for bail recommendation}
\label{sec:dmf_matrix}

Finally, we analyze the overall recommendation given by the DMF matrix
(see Figure~\ref{fig:dmf_matrix}). This aggregates the FTA and NCA
scores into an overall recommendation on the level of cash bail and
pre-trial supervision and monitoring conditions.  Below, we focus on
using the absence of an NVCA as the primary outcome.

\paragraph{Choosing the policy class.}
We consider constructing a new DMF matrix based on the FTA and NCA
scores, which we combine into a vector
$(x_{\text{fta}}, x_{\text{nca}}) \in \{1,\ldots,6\}^2$, restricting
our analysis to the 1,544 cases that used the DMF matrix rather than
those for whom cash bail was automatically recommended. We will focus on a
policy class that keeps the structure of the existing rule encoded by
the DMF matrix. An important aspect of the rule is that it is
\emph{monotonic}; the risk level cannot decrease if either the FTA or
NCA score increases.  Formally, we can encode the monotonic policy
class as,
$ \Pi_{\text{mono}} \equiv \left\{\pi((x_{\text{fta}},
  x_{\text{nca}})) \leq \pi((x_{\text{fta}} + 1, x_{\text{nca}})) \;
  \text{ and } \; \pi((x_{\text{fta}}, x_{\text{nca}})) \leq
  \pi((x_{\text{fta}}, x_{\text{nca}} + 1))\right\}$.  Again, this
leads to an integer program, which we solve via the Gurobi solver. We
will consider four variations of the policy: (i) the overall
risk level from 1 to 7; (ii) the quaternary recommendation of a signature bond,
modest cash bail, moderate cash bail, or (full) cash bail; (iii) the ternary
recommendation that combines modest and moderate cash bail; and (iv) the binary
recommendation that
collapses together all cash bail recommendations.

\paragraph{Choosing the model class.}
We again focus on the class of additive treatment effect models
$\tau_{\text{add}}(a,\bx) \ = \ \tau_{\text{fta}}(a, x_{\text{fta}}) +
\tau_{\text{nca}}(a, x_{\text{nca}})$.  We only condition on the FTA
and NCA scores since they are the two components of the DMF decision
matrix.  Because $x_{\text{fta}}$ and $x_{\text{nca}}$ are discrete
with six values, we can further parameterize the additive terms as
six-dimensional vectors. Importantly, this rules out interactions
between the FTA and NCA scores in the effect. In Appendix~\ref{sec:addl_dmf}
we test for the presence of interactions and do not find evidence against
the null of an additive model.\footnote{Note that we
  could also use a Lipschitz restriction as in
  Section~\ref{sec:nvca_threshold}.  This alternative assumption may
  be significantly weaker, though it would require choosing the
  Lipschitz constant.}

\begin{figure}[t!]
  \centering\vspace{-.5in}
        \begin{subfigure}[t]{0.4\textwidth}  
    \caption{Empirical size}
    {\centering \includegraphics[width=\maxwidth]{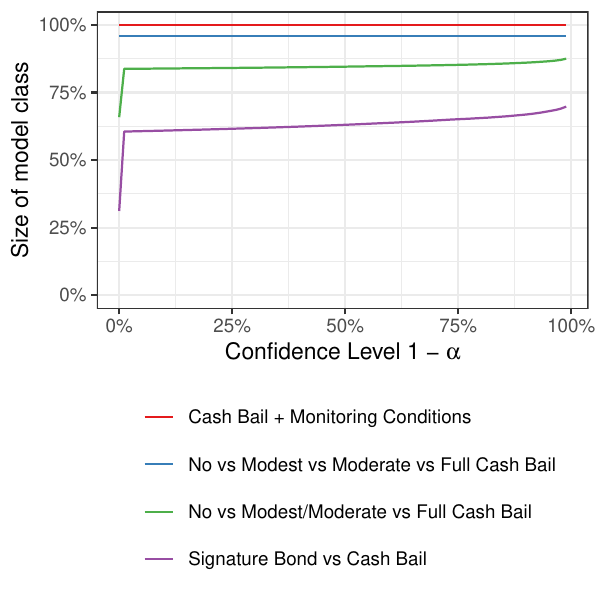} 
    }
    \label{fig:dmf_size}
  \end{subfigure}
  \begin{subfigure}[t]{0.45\textwidth}
    \caption{Maximin DMF Matrices}
    \vspace{-.45in}
  {\centering \includegraphics[width=1.2\textwidth]{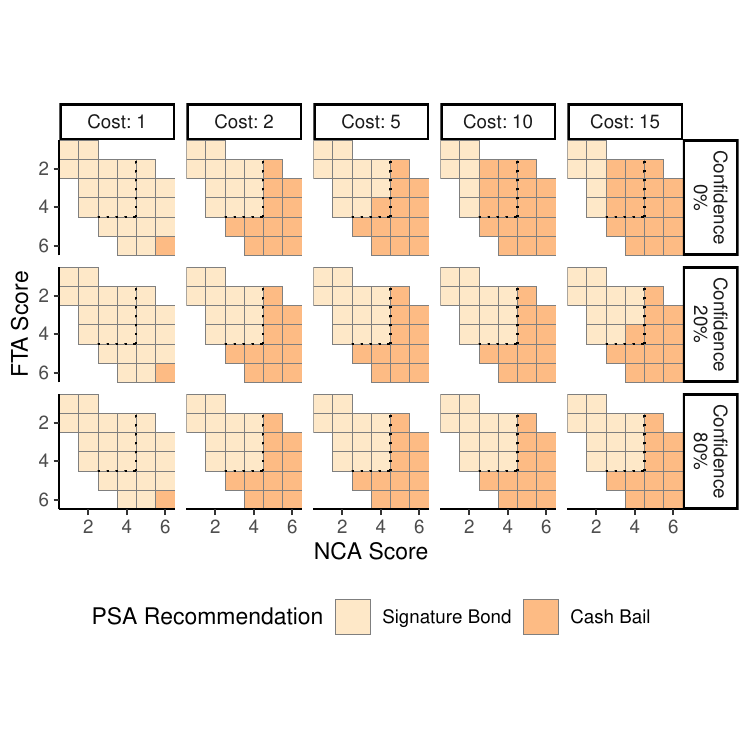} 
  }
    \label{fig:diff_recs_release}
    \end{subfigure}\quad
  \vspace{-.75in}
  \caption{(a) The size (as a percentage of the maximum value) of
    the additive model class with respect to the monotone policy class
    as the confidence level varies for cash bail recommendation policies,
    collapsing together successively more gradations on bail. The coarsest
    policy---Signature Bond vs any Cash Bail---has the most information available.
    (b) Maximin monotone cash bail recommendations under an additive
    model for the treatment effects, as the cost of an NVCA and the
    confidence level vary. The dashed black line indicates the
    original decision boundary between a signature bond (above and to
    the left) and cash bail (below and to the right).
    The original decision boundary is modified only when the
    cost and confidence are low.}
  \label{fig:dmf_results}
\end{figure}

Figure~\ref{fig:dmf_size} presents the size of this
model class relative to the monotone policy class for the three types of
recommendations as we vary the confidence level for the three types of
DMF recommendations. There is no information to
learn reliably a new fine-grained overall risk score or quaternary bail recommendation.
This is due to the structure of the DMF matrix: some risk
levels are only possible for a single NCA score, and others
(such as the moderate cash bail condition) only for a
single combination of FTA and NCA scores.

Therefore, we focus here on the binary cash bail recommendation, where the size 
of the model class is large, but smaller than for the ternary bail recommendation.
This is because we can
never observe a case where the DMF recommends a signature bond with
\emph{either} an FTA score or NCA score above 4, nor can we observe a
case where the DMF recommends cash bail with either an FTA score below
2 or an NCA score below 3.  In the middle is an intermediate area with
FTA scores between 2 and 4 and NCA scores between 3 and 4 where we can
fully identify the effect of assigning cash bail under the additive
model.  For this intermediate area, there is a significant amount of
uncertainty due to small sample sizes: there are only 3 cases where
cash bail is recommended that have an NCA score of 3.  Appendix
Figure~\ref{fig:dmf_bounds} visualizes this uncertainty.

\paragraph{Choosing the utility function.}
We follow
Sections~\ref{sec:nvca_threshold} and \ref{sec:risk_scores}, and
parameterize the utility as a fixed cost of 1 for
recommending cash bail  and varying the cost of an NVCA.

\paragraph{Learning a maximin DMF Matrix.}
Figure~\ref{fig:diff_recs_release} shows the learned policies when
varying the cost of an NVCA and different confidence levels. In the
limiting case where the cost of an NVCA is equal to recommending cash
bail, the safe policy recommends cash bail only for the most
extreme cases.
In the other limiting case, where we set the confidence level to 0 and
rely on the the point estimates directly rather than accounting for
the statistical uncertainty,
increasing the cost of an NVCA
leads to more of the
intermediate area with FTA scores between 2 and 4 and NCA scores
between 3 and 4 being assigned cash bail until the cost is high enough
that the entire identified area is assigned cash bail. However, this
does not hold up to even the slightest degree of statistical uncertainty due
to the uncertainty in the treatment effects. Because the effects of
assigning cash bail are both small and uncertain, the learned policy
reduces to the existing DMF matrix.

\section{Discussion}
\label{sec:discussion}

Data-driven algorithmic policies and recommendations
have become an integral part of our society.  An important challenge
when learning a new policy is to
ensure that it does not perform worse than the existing one.  In
settings like ours where decisions are highly consequential,
policy makers should be able to limit the probability
that a new algorithmic recommendation system achieves a worse outcome than
the existing system.
This is particularly essential when it is impossible to randomize the
algorithm output for ethical and logistical reasons. The lack of
identification necessitates extrapolation, making it
impossible to learn a new policy using standard statistical methods.

We address these challenges by partially identifying
the value of potential policies.  Since this leads to a
decision-making problem
under ambiguity, we use the maximin criterion that selects the best
policy in the worst case. Our methodology has a statistical safety
property: if we make correct structural assumptions about the true model,
the resulting policy will not be worse than the
status quo policy with some probability, up to sampling uncertainty.

Our goal is to understand what changes to the PSA-DMF
recommendation system should be made, if any. 
We do not find strong
support to change the existing FTA and NCA scores, nor the
overall risk score and bail recommendation.  This is due to a
confluence of factors.  Foremost is the conservative nature of
the maximin criterion that yields a strong bias towards the
status quo. However, we emphasize that
failing to find strong evidence to change the status quo
policy does not necessarily imply that the status quo is
desirable.

With the conservative criterion, our analysis is not informative about
the FTA and NCA scores and the overall risk level due to the design of
these algorithms. They have many fine gradations and in some cases
only a single unique combination of inputs can lead to a particular
output. This means that there is little to extrapolate from and
the bounds are uninformative, even
with strong structural assumptions.
In contrast, our analysis is not informative about the binary bail
recommendation due to a combination of identifiability issues and
limited sample sizes.  With an additive model, we can only identify
impacts for cases with intermediate FTA and NCA scores, but the sample
sizes in this intermediate area are too small to make strong
conclusions.

However, the data do support altering the NVCA flag,
even with this conservative criterion,
either by raising the threshold or by putting
less weight on violent convictions and offenses.
Both of these would lead to a more lenient rule that flags
fewer arrestees, and the data support these changes even when the cost of
an NVCA is 8--13 times the cost of triggering the flag.
\citet{Stevenson2022} present survey evidence showing that 50\% of individuals
rate being a victim of an assault as bad as between 5 days and 6 months of
detention; implying a cost ratio for one month of detention between
$\frac{1}{6}$ and 6.
Choosing any ratio within this range would lead to a change to the system,
though since our focus is on triggering the flag, rather than detention,
a larger benchmark may be more appropriate. 

Our analysis serves as an initial proof of concept, probing
various elements of the existing risk assessment system.
As such, there are several limitations and various ways that the analysis has
been simplified. In particular, we place costs on the algorithmic outputs
(the PSA recommendations) rather than on the resulting human decisions (the
judge's bail decision). In Appendix \ref{sec:human_decisions}, we
directly incorporate the costs of the judges'
decisions, but find that this adds too much statistical
uncertainty to improve reliably upon the existing rule.

Another limitation of our analysis in
Section~\ref{sec:risk_scores} is that we separately consider
each outcome and its risk score. Since each risk score can effect each outcome,
 all pairs of risk scores and outcomes could be
considered.  This issue is also present in our analysis of the
DMF matrix (Section~\ref{sec:dmf_matrix}), where we focus
on NVCAs but the bail recommendations can impact all three
outcomes.  A fuller analysis may consider all three risk
scores and the bail recommendation simultaneously for all
three outcomes, using a utility function that incorporates all
of the outcomes and includes measures of costs such
as economic and social outcomes.  However, such an analysis
may not be informative, given the limitations
 in the design
and the data 
discussed above.

An important limitation of our methodology is that
learning algorithmic policies requires making many non-trivial
choices. For example, focusing on the simple setting of
changing the NVCA flag threshold in
Section~\ref{sec:nvca_threshold} requires (a) specifying a
model class; (b) specifying a significance level; and (c)
choosing a utility function, among other things.
The analyses in
Sections~\ref{sec:risk_scores}~and~\ref{sec:dmf_matrix}
include even more involved analytical and implementation
decisions.  Therefore, it is important to examine the
sensitivity of empirical results to these choices.

Choosing the model class can be difficult. With randomized evaluations 
of the status quo policy, simple treatment effect structures may be plausible
because treatment effects are often far less heterogeneous than
baseline outcomes.
We inspected the sensitivity and stability of the maximin policy
to modeling choices and hyper-parameters, such as the choice of
Lipschitz constant.
Formalizing these heuristics is an important direction for future work.

Another key choice is the policy class to optimize over.
We recommend choosing a policy class that can lead to
limited adjustments to the baseline policy rather than wholesale
changes.  While more flexible policy classes could yield better
results, we are unlikely to achieve them, and large changes to
existing systems may not be practically feasible.

Finally, our methodological approach has a wide range of
potential applications.  For transparency and
interpretability, many data-driven algorithms in public policy and
medicine are based on known, deterministic rules rather than
randomized rules.  Examples include the SNAP eligibility rule,
the MELD score for liver transplantation, and other risk assessment instruments
used across public policy contexts
\citep[see][and references therein]{Coston2020}. These instances will all have identifiability issues
due to lack of overlap, and our methodology addresses this challenge
by learning a new, safe policy that improves upon the status quo.

If the algorithm is designed in such a way that there is little to
extrapolate from---as was the case for the FTA and NCA scores---our approach is unlikely to be informative.
Our methodology may be more effective when the baseline policy
includes multiple inputs, each with a large region where multiple actions are
possible.
This can be true when there are group-specific thresholds
for a common risk score or decision variable, for example as with
school enrollment and loan access, and income limits for social
programs \citep{Zhang2022_safe}.
However, different studies may require other 
implementation details.
For instance, our study only includes discrete covariates;
incorporating continuous covariates will require additional
implementation work. In
addition,
analyzing continuous outcomes with non-linear utility functions,
incorporating other criteria such as
fairness measures,
or changing the optimality criterion to minimax regret,
would require additional implementation and analysis.

\appendix

\onehalfspacing

\renewcommand\thefigure{\thesection.\arabic{figure}}    
\setcounter{figure}{0}  
\numberwithin{figure}{section}
\renewcommand\theassumption{\thesection.\arabic{assumption}}    
\setcounter{assumption}{0}  
\renewcommand\theexample{\thesection.\arabic{example}}    
\setcounter{example}{0}
\numberwithin{equation}{section}
\setcounter{theorem}{0}
\numberwithin{theorem}{section}
\setcounter{proposition}{0}
\numberwithin{proposition}{section}
\setcounter{lemma}{0}
\numberwithin{lemma}{section}
\setcounter{corollary}{0}
\numberwithin{corollary}{section}
\setcounter{table}{0}
\numberwithin{table}{section}

\clearpage

\section{Additional theoretical results}

\subsection{Population optimality gap}
\label{sec:pop_opt_gap}
We define the population width of function class $\mathcal{F}$ as
\[
  \mcal{W}_{\mcal{F}}(g) \  = \ \sup_{f \in \mcal{F}} \E\left[\sum_{a \in \mathcal{A}} f(a, X) g(a, X)\right] - \inf_{f \in \mcal{F}} \E\left[\sum_{a \in \mathcal{A}} f(a, X) g(a, X)\right].
\]
Given this definition, the following theorem shows that the population
optimality gap is bounded by the width of the function class.
\begin{theorem}[Population optimality gap]
  \label{thm:pop_optimal_gap} \singlespacing Let $u(a) = u > 0$ for
  all actions $a \in \mcal{A}$, and $\pi^{\inf}$ be a solution to
  Eqn~\eqref{eq:maximin}. If $m^\ast \in \mcal{M}$, the regret of
  $\pi^{\inf}$ relative to the optimal policy
  $\pi^\ast \in \argmax_{\pi \in \Pi} \; V(\pi)$ is
  \[
    \frac{V(\pi^\ast) - V(\pi^{\inf})}{u} \leq \mcal{W}_{\mcal{M}}(\pi^\ast(1-\tilde{\pi})).
  \]
\end{theorem}

\subsection{Extension of Theorems~\ref{thm:emp_safety} and \ref{thm:emp_optimal_gap} to the case where $\alpha = 1$}
\label{sec:alpha_1}

In this section we extend Theorems~\ref{thm:emp_safety} and \ref{thm:emp_optimal_gap} to include results for the case where we set $\alpha = 1$. We do so by providing bounds that hold regardless of the level $\alpha$.
In addition, we also provide a tighter bound on the optimality gap involving the difference between the true optimal policy $\pi^\ast$ and the
baseline policy $\tilde{\pi}$.
For clarity, we present them with the bounds in Theorems~\ref{thm:emp_safety} and \ref{thm:emp_optimal_gap} restated.

For a model class define
the empirical support function as
\[h_\mathcal{F}(z) \equiv \sup_{f \in \mathcal{F}} \frac{1}{n}\sum_{i=1}^n\sum_{a \in \mathcal{A}} z_{ia} f(X_i, a),\]
where $z = (z_{10},\ldots,z_{1K-1}, \ldots, z_{n0}, \ldots, z_{nK-1})$ is a length $n(K-1)$ vector.

\begin{theorem}[Statistical safety (with $\alpha = 1$)]
  \label{thm:emp_safety_full} \singlespacing
  If the baseline policy
  $\tilde{\pi} \in \Pi$ and the true conditional expectation
  $m^\ast(a,x) \in \mcal{M}$, for any $0 < \delta \leq e^{-1}$, the value of
  $\hat{\pi}^\text{inf}$ relative to the baseline $\tilde{\pi}$ is,
  \begin{align*}
    V(\tilde{\pi})\ - V(\hat{\pi}) & \leq   \ 6 C (K-1) \left[\max_{a} \mathcal{R}_n(\Pi_a) + 2\sqrt{\frac{1}{n}\log \frac{K-1}{\delta}} \right]\\
    & \qquad + \sup_{\pi \in \Pi} |h_{\widehat{\mathcal{M}}_n(\alpha)}\left(-\pi (1 - \tilde{\pi}) u(\cdot)\right) - h_\mathcal{\mathcal{M}}\left(-\pi (1 - \tilde{\pi}) u(\cdot)\right)|,
  \end{align*}
  with probability at least $ 1- \delta$,
  where $C = \max_{y \in \{0,1\}, a \in \{0, 1\}} |u(y, a)|$.
\end{theorem}

For the point-wise bounded model class that we consider, the extra term simplifies to be the worst-case difference between the true lower bound and the estimated lower bound.

\begin{corollary}[Statistical safety (with $\alpha = 1$)]
  \label{cor:emp_safety_full_pointwise} \singlespacing
  Under the setting of Theorem~\ref{thm:emp_safety_full},
  let the restricted model class $\mathcal{M}$ and the empirical restricted model class  $\widehat{\mathcal{M}}_n(\alpha)$ consist of point-wise bounded functions,  $\mathcal{M} = \{f: \mathcal{A} \times \mcal{X} \to \R \; \mid \; B_{\ell}(a,x) \leq f(a, x) \leq \B_{u}(a, x)\}$ and  $\widehat{\mathcal{M}}_n(\alpha) = \{f: \mathcal{A} \times \mcal{X} \to \R \; \mid \; \widehat{B}_{\alpha\ell}(a,x) \leq f(a, x) \leq \widehat{B}_{\alpha u}(a, x)\}$.
  Then the value of
  $\hat{\pi}^\text{inf}$ relative to the baseline $\tilde{\pi}$ is,
  \begin{align*}
    V(\tilde{\pi})\ - V(\hat{\pi}) & \leq   \ 6 C (K-1) \left[\max_{a} \mathcal{R}_n(\Pi_a) + 2\sqrt{\frac{1}{n}\log \frac{K-1}{\delta}} \right] + 2C \sup_{a, x}|\widehat{B}_{\alpha \ell}(a,x) - B_\ell(a, x)|,
  \end{align*}
  with probability at least $ 1- \delta$,
  where $C = \max_{y \in \{0,1\}, a \in \{0, 1\}} |u(y, a)|$.
\end{corollary}

\begin{theorem}[Optimality gap (with $\alpha = 1$)]
  \label{thm:emp_optimal_gap_full} \singlespacing
  Let $u(a) = u > 0$ for all actions.
  If the true conditional expectation $m^\ast \in \mcal{M}$, then for any $0 < \delta \leq e^{-1}$ the optimality gap is

  \begin{align*}
    V(\pi^\ast) - V(\hat{\pi}^\text{inf}) 
      & \ \leq \ 2C \widehat{\mcal{W}}_{\widehat{\mcal{M}}_n(\alpha)}\left(\pi^\ast (1 - \tilde{\pi})\right)   +6 C (K-1) \left[\max_{a} \mathcal{R}_n(\Pi_a) + 2\sqrt{\frac{1}{n}\log \frac{K-1}{\delta}} \right]\\
    & \qquad +2C\sup_{\pi \in \Pi} |h_{\widehat{\mathcal{M}}_n(\alpha)}\left(-\pi (1 - \tilde{\pi})\right) - h_\mathcal{\mathcal{M}}\left(-\pi (1 - \tilde{\pi})\right)|,
  \end{align*}
  with probability at least $1 - \delta$,
  where $C = \max_{y \in \{0,1\}, a \in \{0, 1\}} |u(y, a)|$.
\end{theorem}

The statement similarly simplifies under the point-wise bounded setting.
\begin{corollary}[Optimality gap (with $\alpha = 1$)]
  \label{cor:emp_optimal_gap_full_pointwise} \singlespacing
  Under the setting of Theorem~\ref{thm:emp_optimal_gap_full},
  let the restricted model class $\mathcal{M}$ and the empirical restricted model class  $\widehat{\mathcal{M}}_n(\alpha)$ consist of point-wise bounded functions,  $\mathcal{M} = \{f: \mathcal{A} \times \mcal{X} \to \R \; \mid \; B_{\ell}(a,x) \leq f(a, x) \leq \B_{u}(a, x)\}$ and  $\widehat{\mathcal{M}}_n(\alpha) = \{f: \mathcal{A} \times \mcal{X} \to \R \; \mid \; \widehat{B}_{\alpha\ell}(a,x) \leq f(a, x) \leq \widehat{B}_{\alpha u}(a, x)\}$.
  Then for any $0 < \delta \leq e^{-1}$, the optimality gap is
  \begin{align*}
    V(\pi^\ast) - V(\hat{\pi}^\text{inf}) 
      & \ \leq \ 2C \widehat{\mcal{W}}_{\widehat{\mcal{M}}_n(\alpha)}\left(\pi^\ast (1 - \tilde{\pi})\right)   +6 C (K-1) \left[\max_{a} \mathcal{R}_n(\Pi_a) + 2\sqrt{\frac{1}{n}\log \frac{K-1}{\delta}} \right]\\
    & \qquad + 2C\sup_{a, x}|\widehat{B}_{\alpha \ell}(a,x) - B_\ell(a, x)|,
  \end{align*}
  with probability at least $1 - \delta$,
  where $C = \max_{y \in \{0,1\}, a \in \{0, 1\}} |u(y, a)|$.
\end{corollary}

\subsection{Learning from experiments evaluating a deterministic
  policy: a generic form of value function}

Below we state a generic form of the value function with access to
experimental data as in Section~\ref{sec:ab_test}. The first line
shows how to write the value of $\pi$ in terms of the true CATE
$\tau^\ast$ and the conditional expected outcome under the null action
$m^\ast(-1, x)$.  The second line further shows how to identify this
expression with observable data via inverse probability weighting.
\begin{proposition}
  \label{prop:ab_test_utility}

  If $Z \perp \!\! \perp Y(a) \mid X$ and $0 < e(x) < 1$, then the expected utility can be written as
  \begin{align*}
    V(\pi) &= \E\left[\sum_{a \in
    \mathcal{A}} \pi(X, a) \{u(a) \tau^\ast(a, X) + c(a) + u(a) m^\ast(-1, X)\}\right]\\
    & = \E\left[\sum_{a \in \mathcal{A}} \pi(X, a) 
      u(a) \left[\tilde{\pi}(X, a) \left(\Gamma(1,X,Y) - \Gamma(0, X, Y)\right) +  c(a) +  u(a) \Gamma(0, X, Y)\right]\right]\\
      & \qquad +  \E\left[\sum_{a \in \mathcal{A}} \pi(X, a) u(a)\left\{1 - \tilde{\pi}(X,a)\right\}
       \tau^\ast(a, X)\right] ,
  \end{align*}
  where $\Gamma(Z, X, Y) \equiv Y\{Z(1-2e(X)) + e(X)\}/\{e(X)(1 - e(X))\}$ is 
  the inverse probability weighted outcome.
\end{proposition}
Note that when the utility gain is constant, ($u(a) = u$ for all $a \in \mathcal{A}$),
we have that
\[
  \E\left[\sum_{a \in \mathcal{A}} \pi(X, a) u(a) m(-1, X)\right] = u\E[m^\ast(-1, X)],
\]
and does not depend on the policy, because $\sum_{a \in \mathcal{A}} \pi(X, a) = 1$.

\section{Proofs}

\begin{proof}[Proof of Proposition \ref{prop:pop_safety}]
  \[
  V(\tilde{\pi}) = V^{\inf}(\tilde{\pi}) \leq V^{\inf}(\pi^{\inf}) \leq V(\pi^{\inf}).
\]
\end{proof}

\begin{proof}[Proof of Theorem \ref{thm:pop_optimal_gap}]
  Since $V^{\inf}(\pi) \leq V(\pi)$ for all policies $\pi$, the regret is bounded by
  \[
    \begin{aligned}
      V(\pi^\ast) - V(\pi^{\inf}) & \leq V(\pi^\ast) - V^{\inf}(\pi^{\inf})\\
      & = V(\pi^\ast) - V^{\inf}(\pi^\ast) + V^{\inf}(\pi^\ast) - V^{\inf}(\pi^{\inf}).
    \end{aligned}
  \]
  Now since $\pi^{\inf}$ is a maximizer of $V^{\inf}(\pi)$, $V^{\inf}(\pi^\ast) - V^{\inf}(\pi^{\inf}) \leq 0$. Now note that for any $\pi$, 
  \begin{align*}
    V(\pi) & = \E\left[\sum_{a \in \mathcal{A}}\pi(X, a) \tilde{\pi}(X, A) (u(a) m^\ast(a, X) + c(a))\right] \\
    &\hspace{0.2in}+ \E\left[\sum_{a \in \mathcal{A}}\pi(X, a) (1 - \tilde{\pi}(X, A)) (u(a) m^\ast(a, X) + c(a))\right] \\
    & = V(\tilde{\pi}) + \E\left[\sum_{a \in \mathcal{A}}\pi(X, a) (1 - \tilde{\pi}(X, A)) (u(a) m^\ast(a, X) + c(a))\right].
  \end{align*}

  This yields that
  \begin{align*}
    V(\pi^\ast) - V(\pi^{\inf}) & \leq \sum_{a \in \mcal{A}} u(a) \E\left[\pi^\ast(X, a)\{1 - \tilde{\pi}(X, a)\} m^\ast(a,X)\right]\\
    & \qquad \qquad - \inf_{f \in \mcal{M}} \sum_{a \in \mcal{A}} u(a) \E\left[\pi^\ast(X, a)\{1 - \tilde{\pi}(X, a)\} f(a,X)\right]\\
    & \leq \sup_{f \in \mcal{M}} \left\{\sum_{a \in \mcal{A}} u(a) \E\left[\pi^\ast(X, a)\{1 - \tilde{\pi}(X, a)\} f(a,X)\right]\right\}\\
    & \qquad  - \inf_{f \in \mcal{M}} \left\{\sum_{a \in \mcal{A}} u(a) \E\left[\pi^\ast(X, a)\{1 - \tilde{\pi}(X, a)\} f(a,X)\right]\right\}\\
    & = |u| \mcal{W}_{\mcal{M}}\left(\pi^\ast (1 - \tilde{\pi})\right).
  \end{align*}
\end{proof}

\begin{lemma}
  \label{lem:rademacher_complexity}

  Define $\hat{V}(\pi) \equiv \hat{V}(\pi, m^\ast)$
  Then for any $0 < \delta < e^{-1}$,
  \[
    \sup_{\pi \in \Pi} |\hat{V}(\pi) - V(\pi)| \leq 3 C (K-1) \max_{a} \mathcal{R}_n(\Pi_a) + 5C(K-1) \sqrt{\frac{1}{n}\log \frac{K-1}{\delta}}
  \]
  with probability at least $1 -\delta$, where $C = \max_{y \in \{0,1\}, a \in \{0, 1\}} |u(y, a)|$.
\end{lemma}

\begin{proof}[Proof of Lemma \ref{lem:rademacher_complexity}]
  First, since $\sum_{a \in \mathcal{A}} \pi(x, a) = 1$ for all $x$, we can write the empirical value as
  \[
    \begin{aligned}
      \hat{V}(\pi) & = \frac{1}{n}\sum_{i=1}^n\sum_{a =1}^{K-1} \pi(X_i, a)\left\{u(a)\left[\left(\tilde{\pi}(X_i, a) - \tilde{\pi}(X_i, 0) \right)Y_i \right. \right.\\
      & \qquad  + \left. \left.\{1 - \tilde{\pi}(X_i, a)\}m^\ast(a,X_i) - (1 - \tilde{\pi}(X_i, 0) m^\ast(0, X_i))\right] + c(a) - c(0)\right\}\\
      & \qquad + u(0) \left[\tilde{\pi}(X_i, 0) Y_i + (1 - \tilde{\pi}(X_i, 0))m^\ast(0, X_i)\right] + c(0). 
    \end{aligned}
  \]
  Now, define the function class with functions $f_a(x,y)$ as
  \begin{align*}
    \mathcal{F}_a & = \left\{
      \pi(X_i, a)\left\{u(a)\left[\left(\tilde{\pi}(X_i, a) - \tilde{\pi}(X_i, 0) \right)Y_i +
      \{1 - \tilde{\pi}(X_i, a)\}m^\ast(a,X_i) - (1 - \tilde{\pi}(X_i, 0) m^\ast(0, X_i))\right]\right. \right.\\
      & \qquad \left. \left. + c(a) - c(0)\right\}
      \mid \pi(X_i, a) \in \Pi_a
    \right\},
  \end{align*}
  where $\Pi_a = \{\bbone\{\pi(\cdot) = a\} \mid \pi \in \Pi\}$ is the set of 
  all potential policy assignments to action $a$. Now notice that
  \[
  \begin{aligned}
    \sup_{\pi \in \Pi} |\hat{V}(\pi) - V(\pi)| & \leq \E_{X, Y, \varepsilon}\left[\left|\frac{1}{n}\sum_{i=1}^n (u(0) \left[\tilde{\pi}(X_i, 0) Y_i + (1 - \tilde{\pi}(X_i, 0))m^\ast(0, X_i)\right] + c(0)) \varepsilon_i\right|\right]\\
    & \quad  + \sum_{a = 1}^{K-1} \sup_{f_a \in \mcal{F}_a} \left|\frac{1}{n}\sum_{i=1}^n f_a(X_i, Y_i) - \E\left[f_a(X, Y)\right]\right|.
  \end{aligned}
  \]

  The class $\mcal{F}_a$ is uniformly bounded by twice the maximum absolute utility 
  $C = \max_{y \in \{0,1\}, a \in \{0, 1\}} |u(y, a)|$,
   so by Theorem 4.5 in \citet{wainwright_2019}
  \[
    \sup_{f_a \in \mcal{F}_a} \left|\frac{1}{n}\sum_{i=1}^n f_a(X_i, Y_i) - \E\left[f_a(X, Y)\right]\right| \leq 2 \mcal{R}_n(\mcal{F}_a) + t,
  \]
  with probability at least $1 - \exp\left(-\frac{nt^2}{8C^2}\right)$. Now because
  \[
    u(0) \left[\tilde{\pi}(X_i, 0) Y_i + (1 - \tilde{\pi}(X_i, 0))m^\ast(0, X_i)\right] \leq u(0),
  \]
  and by independence of the data points and $\varepsilon$, we can get the bound
  \[
  \begin{aligned}
    & \E_{X, Y, \varepsilon}\left[\left|\frac{1}{n}\sum_{i=1}^n (u(0) \left[\tilde{\pi}(X_i, 0) Y_i + (1 - \tilde{\pi}(X_i, 0))m^\ast(0, X_i)\right] + c(0)) \varepsilon_i\right|\right]\\
    \leq \quad & \frac{1}{n}\left(\E_\varepsilon\left[\left(\sum_{i=1}^n (u(0) \left[\tilde{\pi}(X_i, 0) Y_i + (1 - \tilde{\pi}(X_i, 0))m^\ast(0, X_i)\right] + c(0))\varepsilon_i\right)^2\right]\right)^{\frac{1}{2}}\\
    & = \frac{1}{n}\left(\E\left[\sum_{i=1}^n (u(0) \left[\tilde{\pi}(X_i, 0) Y_i + (1 - \tilde{\pi}(X_i, 0))m^\ast(0, X_i)\right] + c(0))^2\varepsilon_i^2\right]\right)^{\frac{1}{2}}\\
    & \leq \frac{1}{n}\left(\sum_{i=1}^n (u(0) + c(0))^2\right)^{\frac{1}{2}}\\
    & = \frac{|u(0) + c(0)|}{\sqrt{n}} \leq \frac{2C}{\sqrt{n}}.
  \end{aligned}
  \]

  Furthermore, since $\mathcal{F}_a$ consists of compositions of functions $g \in \Pi_a$ with linear functions with a bounded slope,
  \[
    \left|u(a)\left[\left(\tilde{\pi}(X_i, a) - \tilde{\pi}(X_i, 0) \right)Y_i + \{1 - \tilde{\pi}(X_i, a)\}m^\ast(a,X_i) - (1 - \tilde{\pi}(X_i, 0) m^\ast(0, X_i))\right] + c(a) - c(0) \right|\leq 3 C,
  \]
  we can use the Talagrand contraction principle \citep[Theorem 4.12][]{ledoux_probability_1991} to bound the Rademacher complexity for $\mcal{F}_a$ by
  \[
    \begin{aligned}
      \mcal{R}_n(\mcal{F}_a)
      & \leq 3C \mcal{R}_n(\Pi_a).  
    \end{aligned}
  \]
  Doing this for each $a = 1,\ldots,K-1$ and using the union bound gives that
  \[
    \sup_{\pi \in \Pi} |\hat{V}(\pi) - V(\pi)| \leq \frac{2C(K-1)}{\sqrt{n}} +3 C (K-1) \max_{a \in \mathcal{A}} \mathcal{R}_n(\Pi_a) + t
  \]
  with probability at least $1 - (K-1)\exp\left(-\frac{nt^2}{8C^2}\right)$.
  Choosing $ t = C\sqrt{\frac{8}{n}\log \frac{K-1}{\delta}}$ and noting that
  $2(K-1) + \sqrt{8 \log \frac{K-1}{\delta}} \leq (2(K-1) + \sqrt{8}) \sqrt{\log \frac{K-1}{\delta}} \leq (2 + \sqrt{8}) (K-1) \sqrt{\log \frac{K-1}{\delta}} \leq 5 (K-1) \sqrt{\log \frac{K-1}{\delta}}$ gives the result.

\end{proof}

\begin{lemma}
  \label{lem:diff_support}

  For the empirical restricted model class $\widehat{\mathcal{M}}_n(\alpha)$ and for a policy $\pi \in \Pi$

  \[
    \hat{V}^{\inf}(\hat{\pi}) - \hat{V}(\hat{\pi})  \leq \sup_{\pi \in \Pi} |h_{\widehat{\mathcal{M}}_n(\alpha)}\left(-\pi (1 - \tilde{\pi}) u(\cdot)\right) - h_\mathcal{\mathcal{M}}\left(-\pi (1 - \tilde{\pi}) u(\cdot)\right)|,
  \]
  where the function $(\pi(1 - \tilde{\pi}) u(\cdot))(x, a) = \pi(x,a) * (1 - \tilde{\pi}(x, a)) u(a)$. Furthermore, if $\mathcal{M} \subseteq \widehat{\mathcal{M}}_n(\alpha)$, then 
  \[
    \hat{V}^{\inf}(\hat{\pi}) - \hat{V}(\hat{\pi})  \leq 0.
  \]

\end{lemma}

\begin{proof}[Proof of Lemma~\ref{lem:diff_support}]
  For a model class $\mathcal{F}$, define $\tilde{V}(\pi, \mathcal{F}) = \inf_{f \in \mathcal{F}} \hat{V}(\pi, f)$, and define $\hat{V}^{\inf}(\pi) = \min_{m \in \widehat{\mathcal{M}}_n(\alpha)} \hat{V}(\pi, m)$, so that $\hat{V}^{\inf}(\pi) = \tilde{V}(\pi, \widehat{M}_n(\alpha))$.
  This implies that
  \begin{align*}
    \hat{V}^{\inf}(\hat{\pi}) - \hat{V}(\hat{\pi}) & \leq \hat{V}^{\inf}(\hat{\pi}) - \tilde{V}(\hat{\pi}, \mathcal{M})\\
    & = \tilde{V}(\hat{\pi}, \widehat{\mathcal{M}}_n(\alpha)) - \tilde{V}(\hat{\pi}, \mathcal{M}).
  \end{align*}

  Now note that if $\mathcal{M} \subseteq \widehat{\mathcal{M}}_n(\alpha)$, then $\tilde{V}(\hat{\pi}, \widehat{\mathcal{M}}_n(\alpha)) - \tilde{V}(\hat{\pi}, \mathcal{M}) \leq 0$. Otherwise  we can write this difference as
  \begin{align*}
    \tilde{V}(\hat{\pi}, \widehat{\mathcal{M}}_n(\alpha)) - \tilde{V}(\hat{\pi}, \mathcal{M}) & = \inf_{m \in \widehat{\mathcal{M}}_n(\alpha)} \left\{ \frac{1}{n}\sum_{i=1}^n \sum_{a \in \mathcal{A}} \pi(X_i, a)(1 - \tilde{\pi}(X_i, a)) u(a) m(a, X_i) \right\}\\
    & \qquad - \inf_{m \in \mathcal{M}} \left\{ \frac{1}{n}\sum_{i=1}^n \sum_{a \in \mathcal{A}} \pi(X_i, a)(1 - \tilde{\pi}(X_i, a)) u(a) m(a, X_i) \right\}\\
    & = - \sup_{m \in \widehat{\mathcal{M}}_n(\alpha)} \left\{ - \frac{1}{n}\sum_{i=1}^n \sum_{a \in \mathcal{A}} \pi(X_i, a)(1 - \tilde{\pi}(X_i, a)) u(a) m(a, X_i) \right\}\\
    & \qquad + \sup_{m \in \mathcal{M}} \left\{- \frac{1}{n}\sum_{i=1}^n \sum_{a \in \mathcal{A}} \pi(X_i, a)(1 - \tilde{\pi}(X_i, a)) u(a) m(a, X_i) \right\}\\
    & = - h_{\widehat{\mathcal{M}}_n(\alpha)}(-\pi (1 - \tilde{\pi}) u(\cdot)) +  h_\mathcal{M}(-\pi (1 - \tilde{\pi}) u(\cdot))
  \end{align*}

  Taking the supremeum over all possible policies $\hat{\pi} \in \Pi$ gives the result.
\end{proof}

\begin{proof}[Proof of Theorems~\ref{thm:emp_safety} and \ref{thm:emp_safety_full}]
  The difference in values between $\tilde{\pi}$ and $\hat{\pi}$ is
  \[
    \begin{aligned}
      V(\tilde{\pi}) - V(\hat{\pi})
      & = V(\tilde{\pi}) - \hat{V}(\tilde{\pi}) + \hat{V}(\tilde{\pi}) - \hat{V}(\hat{\pi}) + \hat{V}(\hat{\pi}) - V(\hat\pi)\\
      & \leq 2 \sup_{\pi \in \Pi} | \hat{V}(\pi) - V(\pi) | + \hat{V}(\tilde{\pi}) - \hat{V}(\hat{\pi})
    \end{aligned}
  \]

  We have bounded the first term in Lemma~\ref{lem:rademacher_complexity}. To bound the second term, notice that
  \begin{align*}
    \hat{V}(\tilde{\pi}) - \hat{V}(\hat{\pi})  & =
    \hat{V}^{\inf}(\tilde{\pi}) - \hat{V}(\hat{\pi})\\
    & = \underbrace{\hat{V}^{\inf}(\tilde{\pi}) - \hat{V}^{\inf}(\hat{\pi})}_{\leq 0} + \hat{V}^{\inf}(\hat{\pi}) - \hat{V}(\hat{\pi})\\
    & \leq \hat{V}^{\inf}(\hat{\pi}) - \hat{V}(\hat{\pi}),
  \end{align*}
  where we have used that $\hat{\pi}$ maximizes $\hat{V}^{\inf}(\pi)$.

  In Lemma~\ref{lem:diff_support} we have bounded this difference.
    Combining the two bounds we have that
  \begin{align*}
    V(\tilde{\pi}) - V(\hat{\pi})\leq 6 C (K-1) \max_{a} \mathcal{R}_n(\Pi_a) + 10C(K-1) \sqrt{\frac{1}{n}\log \frac{K-1}{\delta}}\\
    + \sup_{\pi \in \Pi} |h_{\widehat{\mathcal{M}}_n(\alpha)}\left(-\pi (1 - \tilde{\pi}) u(\cdot)\right) - h_\mathcal{\mathcal{M}}\left(-\pi (1 - \tilde{\pi}) u(\cdot)\right)|,
  \end{align*}
  with probability at least $1 - \delta$. And if $\mathcal{M} \subseteq  \widehat{\mcal{M}}_n(\alpha)$, then we have the further bound
  \[
    V(\tilde{\pi}) - V(\hat{\pi})\leq 6 C (K-1) \max_{a} \mathcal{R}_n(\Pi_a) + 10C(K-1) \sqrt{\frac{1}{n}\log \frac{K-1}{\delta}},
  \]
  with probability at least $1 - \delta$. Noting that $P(\mathcal{M} \subseteq  \widehat{\mcal{M}}_n(\alpha)) \geq 1 - \alpha$, and taking the union bound gives that this second bound holds with probability at least $1 - \delta - \alpha$.

\end{proof}

\begin{proof}[Proof of Theorems \ref{thm:emp_optimal_gap} and \ref{thm:emp_optimal_gap_full}]
  The regret of $\hat{\pi}$ relative to $\pi^\ast$ is
  \[
    \begin{aligned}
      V(\pi^\ast) - V(\hat{\pi})
      & = V(\pi^\ast) - \hat{V}(\pi^\ast) + \hat{V}(\pi^\ast) - \hat{V}(\hat{\pi}) + \hat{V}(\hat{\pi}) - V(\hat{\pi})\\
      & \leq \sup_{\pi \in \Pi} 2|\hat{V}(\pi) - V(\pi)| + \hat{V}(\pi^\ast) - \hat{V}(\hat{\pi}).
    \end{aligned}
  \]
  We have bounded the first term in Lemma \ref{lem:rademacher_complexity}, and we now turn to the second term.
  \[
    \begin{aligned}
      \hat{V}(\pi^\ast) - \hat{V}(\hat{\pi}) & = \hat{V}(\pi^\ast) - \hat{V}^{\inf} (\hat{\pi}) + \hat{V}^{\inf} (\hat{\pi}) - \hat{V} (\hat{\pi})\\
      & = \hat{V}(\pi^\ast) - \hat{V}^{\inf} (\pi^\ast)  + \underbrace{\hat{V}^{\inf} (\pi^\ast)  - \hat{V}^{\inf} (\hat{\pi})}_{\leq 0} + \hat{V}^{\inf} (\hat{\pi}) - \hat{V} (\hat{\pi})\\
      & \leq \hat{V}(\pi^\ast) - \hat{V}^{\inf} (\pi^\ast) + \hat{V}^{\inf} (\hat{\pi}) - \hat{V} (\hat{\pi}),
    \end{aligned}
  \]
  where we have again used that $\hat{\pi}$ maximizes $\hat{V}^{\inf}(\pi)$. We have bounded the second term in Lemma~\ref{lem:diff_support}, now we turn to the first term:
\[
  \begin{aligned}
    \hat{V}(\pi^\ast) - \hat{V}^{\inf}(\pi^\ast)& \leq \frac{1}{n}\sum_{i=1}^n\sum_{a \in \mathcal{A}}u(a)\pi^\ast(X_i, a) \{1 - \tilde{\pi}(X_i, a)\}m^\ast(a, X_i) \\
    & \qquad -  \inf_{f \in \widehat{\mcal{M}}_n(\alpha)} \frac{1}{n}\sum_{i=1}^n\sum_{a \in \mcal{A}} u(a)\pi^\ast(X_i, a))\{1 - \tilde{\pi}(X_i, a)\} f(a, X_i)\\
    & \leq \sup_{f \in \widehat{\mcal{M}}_n(\alpha)} \frac{1}{n}\sum_{i=1}^n\sum_{a \in \mcal{A}} u(a) \pi^\ast(X_i, a))\{1 - \tilde{\pi}(X_i, a)\} f(a, X_i)\\
    & \qquad -  \inf_{f \in \widehat{\mcal{M}}_n(\alpha)} \frac{1}{n}\sum_{i=1}^n\sum_{a \in \mcal{A}} u(a) \pi^\ast(X_i, a))\{1 - \tilde{\pi}(X_i, a)\} f(a, X_i)\\
    & = |u|\widehat{\mcal{W}}_{\widehat{\mcal{M}}_n(\alpha)}\left(\pi^\ast (1 - \tilde{\pi})\right).
  \end{aligned}
\]
Combined with Lemmas \ref{lem:rademacher_complexity} and \ref{lem:diff_support} and the union bound this gives that
\begin{align*}
  V(\pi^\ast) - V(\hat{\pi}) & =   |u|\widehat{\mcal{W}}_{\widehat{\mcal{M}}_n(\alpha)}\left(\pi^\ast (1 - \tilde{\pi})\right)+ 6 C (K-1) \max_{a} \mathcal{R}_n(\Pi_a) + 10C(K-1) \sqrt{\frac{1}{n}\log \frac{K-1}{\delta}}\\
  & \qquad + |u|\sup_{\pi \in \Pi} |h_{\widehat{\mathcal{M}}_n(\alpha)}\left(-\pi (1 - \tilde{\pi})\right) - h_\mathcal{\mathcal{M}}\left(-\pi (1 - \tilde{\pi})\right)|,
\end{align*}
with probability at least $1 - \delta$ and
\[
  V(\pi^\ast) - V(\hat{\pi})  =   |u| \widehat{\mcal{W}}_{\widehat{\mcal{M}}_n(\alpha)}\left(\pi^\ast (1 - \tilde{\pi})\right)+ 6 C (K-1) \max_{a} \mathcal{R}_n(\Pi_a) + 10C(K-1) \sqrt{\frac{1}{n}\log \frac{K-1}{\delta}},
\]
with probability at least $1 - \delta - \alpha$. Noting that $|u| \leq 2C$ gives the result.

\end{proof}

\begin{proof}[Proof of Corollary \ref{cor:emp_safety_full_pointwise} and \ref{cor:emp_optimal_gap_full_pointwise}]
  These Corollaries follow from noting that
  \begin{align*}
    & h_{\widehat{\mathcal{M}}_n(\alpha)}\left(-\pi (1 - \tilde{\pi}) u(\cdot)\right) - h_\mathcal{\mathcal{M}}\left(-\pi (1 - \tilde{\pi}) u(\cdot)\right)\\
     =& \frac{1}{n}\sum_{i=1}^n \pi(X_i, a)(1 - \tilde{\pi}(X_i, a)) u(a) \widehat{B}_{\alpha \ell}(X_i, a) - \frac{1}{n}\sum_{i=1}^n \pi(X_i, a)(1 - \tilde{\pi}(X_i, a)) u(a) B_{ \ell}(X_i, a)\\
     =&  \frac{1}{n}\sum_{i=1}^n \pi(X_i, a)(1 - \tilde{\pi}(X_i, a)) u(a)( \widehat{B}_{\alpha \ell}(X_i, a) - B_{ \ell}(X_i, a))\\
      \leq& C \sup_{x, a}|\widehat{B}_{\alpha \ell}(x, a) - B_{ \ell}(x, a)|.
  \end{align*}
\end{proof}

\begin{proof}[Proof of Proposition~\ref{prop:ab_test_utility}]
  For the first equality, note that $m^\ast(a,X) = \tau^\ast(a,X) + m^\ast(-1, X)$. So the first equality follows by plugging in this equality to Equation~\eqref{eq:policy_model}. Next, note that we can decompose this expression as in Equation~\eqref{eq:value_model} to get that
  \[
    V(\pi, m^\ast) \ = \ \E\left[\sum_{a \in \mathcal{A}} \pi(X, a) \left\{
      u(a) \left[\tilde{\pi}(X, a) \tilde{\tau}(x) + \left\{1 - \tilde{\pi}(X,a)\right\}
       \tau^\ast(a, X)\right]+c(a) + u(a) m^\ast(-1, X)\right\} \right].
  \]
  Now using that $\E[\Gamma(Z, X, Y) \mid Z = z, X = x] = z \cdot m^\ast(\tilde{\pi}(x), x)
  + (1-z) \cdot m^\ast(-1, x)$, and noting that $\tilde{\tau}(x) = \E[\Gamma(1, X, Y)  - \Gamma(1, X, Y) \mid X = x]$,
  gives the second expression.
\end{proof}

\section{Computation for restricted model classes}
\label{sec:representative_cases}

In this section, we show, in detail, how to compute the population and
empirical model classes in a variety of cases: no restrictions,
Lipschitz functions, linear models, and additive models.

First, for point-wise bounded model classes, we can compute the size
term in Theorem~\ref{thm:emp_optimal_gap} by looking for the policy
$\pi \in \Pi$ that disagrees with the baseline policy $\tilde{\pi}$
when the upper and lower bounds are farthest apart:
\begin{equation}
  \label{eq:bounded_sizes}
  \widehat{\mathcal{S}}(\widehat{\mcal{T}}_n(\alpha), \Pi; \tilde{\pi}) =
  \sup_{\pi \in \Pi } \frac{1}{n}\sum_{i=1}^n \sum_{a \in \mathcal{A}} \pi(X_i, a) (1 - \tilde{\pi}(X_i, a)) \left(\widehat{B}_{\alpha u}(a, X_i) - \widehat{B}_{\alpha \ell}(a, X_i)\right).
\end{equation}

\subsection{No restrictions}
\label{sec:no_restriction}

Suppose that the conditional expectation has no restrictions, other
than that it must lie between $L$ and $U$, i.e.,
$\mcal{F} = \{f \mid L \leq f(a,x) \leq U \;\; \forall a\in\mcal{A},
\; x \in \mcal{X}\}$. The restricted model class
$\mcal{M} = \{f \in \mcal{F} \mid f(a,x) = \tilde{m}(x) \;\; \text{for
} a \text{ with } \tilde{\pi}(x) = a\}$ provides no additional
information when the policy $\pi$ disagrees with the baseline policy
$\tilde{\pi}$. The upper and lower bounds are given by
$B_u(a,x) = \tilde{\pi}(x, a)\tilde{m}(x) + \{1 - \tilde{\pi}(
x,a)\} U $ and $B_\ell(a,x) = \tilde{\pi}(x, a)\tilde{m}(x) + \{1 - \tilde{\pi}(
x,a)\} L$, respectively.

To construct the empirical model class $\widehat{\mcal{M}}_n(\alpha)$,
we begin with a simultaneous $1-\alpha$ confidence interval for the
conditional expectation function $\tilde{m}(x)$, with lower and upper
bounds
$\widehat{C}_\alpha(x) = [\widehat{C}_{\alpha \ell} (x),
\widehat{C}_{\alpha u}(x)]$ such that
\begin{equation}
  \label{eq:simul_band}
  P\left(\tilde{m}(x) \in \widehat{C}_\alpha(x) \;\; \forall \; x \in \mathcal{X} \right) \geq 1 - \alpha.
\end{equation}
See \citet{Srinivas2010,Chowdhury2017,Fiedler2021} for examples on
constructing such simultaneous bounds via kernel methods in
statistical control settings. With this confidence band, we can use
the upper and lower bounds of the confidence band in place of the true
conditional expectation $\tilde{m}(x)$, i.e.
$\widehat{B}_{\alpha u}(a,x) = \tilde{\pi}(X, a)\widehat{C}_{\alpha
  u}(x) + \{1 - \tilde{\pi}(X, a)\} U $ and
$\widehat{B}_\ell(a,x) = \tilde{\pi}(X, a)\widehat{C}_{\alpha
  \ell}(x) \{1 - \tilde{\pi}(x, a)\} L$.

\subsection{Lipschitz functions}
\label{sec:lip}

We next consider the case where the covariate space $\mcal{X}$ has a
norm $\|\cdot\|$, and that $m(a, \cdot)$ is a $\lambda_a$-Lipschitz
function,
\begin{equation*}
  \mcal{F} = \{f:\mcal{A} \times \mcal{X} \to \R \;\; \mid \;\; |f(a, x) - f(a, x')| \leq \lambda_a \|x - x'\|\}.
\end{equation*}
Taking the greatest lower bound and least upper bound implied by this
model class leads to lower and upper bounds,
$B_\ell(a,x) = \sup_{x' \in \tilde{\mcal{X}}_a} \left\{\tilde{m}(x') -
  \lambda_a \|x - x'\|\right\}$, and
$B_u(a,x) = \inf_{x' \in \tilde{\mcal{X}}_a} \left\{\tilde{m}(x') +
  \lambda_a \|x - x'\|\right\}$, where
$\tilde{\mcal{X}}_a = \{x \in \mcal{X} \mid \tilde{\pi}(x) = a\}$ is
the set of covariates with the baseline policy giving action $a$.  The
further we extrapolate from the area where the baseline action
$\tilde{\pi}(x) = a$, the larger the value of $\|x - x'\|$ will be and
so there will be more ignorance about the values of the function.

The size of $\mcal{M}$ will depend on the expected distance to the
boundary between baseline actions and the value of the Lipschitz
constant.  If most individuals are close to the boundary, or the
Lipschitz constant is small, $\mcal{M}$ will be small and the safe
policy will be close to optimal. Conversely, a large number of
individuals far away from the boundary or a large Lipschitz constant
will increase the potential for suboptimality.

To construct the empirical version, we again use a simultaneous
confidence band $\widehat{C}_{\alpha}(x)$ satisfying
Equation~\eqref{eq:simul_band}.  Then, the lower and upper bounds use
the lower and upper confidence limits in place of the function values,
$\widehat{B}_{\alpha \ell}(a,X) = \sup_{x' \in
  \tilde{\mcal{X}}_a}\left\{ \widehat{C}_{\alpha \ell}(x') - \lambda_a
  \|X - x'\| \right\}$ and
$\widehat{B}_{\alpha u}(a,X) = \inf_{x' \in \tilde{\mcal{X}}_a}\left\{
  \widehat{C}_{\alpha u}(x') - \lambda_a \|X - x'\|\right\}$.  In our
analysis of the NVCA flag threshold in
Section~\ref{sec:nvca_threshold}, the covariate space $\mathcal{X}$ is
discrete, so we construct a simultaneous confidence interval via the
a Bonferroni correction on the 7 unique values.

Note that it is also possible to construct bounds using a finite set of
evaluation points. For example, if $\check{\mathcal{X}_a}$ is a finite set
of points such that the baseline policy satisfies $\tilde{\pi}(x) = a$, an
alternative procedure to construct a lower bound is to take the greatest lower
bound over the finite set $\check{\mathcal{X}_a}$, i.e.
\[
\check{B}_\ell(a,x) = \max_{x' \in \check{\mathcal{X}_a}} \tilde{m}(x') - \lambda_a \|x - x'\|.
\]
Because the finite set $\check{\mathcal{X}_a} \subseteq \tilde{\mathcal{X}}_a$,
the greatest lower bound over $\check{\mathcal{X}_a}$  will be less than
or equal to the greatest lowest bound over the entire set $\tilde{\mathcal{X}}_a$,
i.e. $\check{B}_\ell(a,x) \geq B_{\ell}(a,x)$. With this finite set, we can
create the empirical version using a simultaneous confidence band
$\check{C}_\alpha(x) $ over only $\check{\mathcal{X}_a}$ that satisfies 
\[
  P\left(\tilde{m}(x) \in \check{C}_\alpha(x) \;\; \forall \; x \in \check{X}_a \right) \geq 1 - \alpha.
\]
Such a bound can be constructed with a simple Bonferroni correction, or via
a more tailored approach.
Then the empirical lower bound would be $\check{B}_{\alpha \ell}(a,X) = \max_{x' \in
\check{\mcal{X}}_a}\left\{ \check{C}_{\alpha \ell}(x') - \lambda_a
\|X - x'\| \right\}$. Unlike in the population case, the empirical lower bound
using the finite set, $\check{B}_{\alpha \ell}(a,x)$, may be greater than
the empirical lower bound using the simultaneous confidence band $\widehat{B}_{\alpha \ell}(a,x)$ if the simultaneous confidence band over the entire set
$\mathcal{X}$ is wider than that over the smaller, finite set $\check{X}_a$.

\subsection{Linear models}
\label{sec:glm}

We next consider, as a model class, a linear model in a
set of basis functions $\phi:\mcal{A} \times \mcal{X} \to \R^d$,
$\mcal{F}= \{f(a,x) = h^{-1}(b^\top \phi(a, x))\}$, where we still
enforce the upper and lower bounds of $U$ and $L$.  The restricted
model class is the set of coefficients $b$ that satisfy
$\tilde{m}(x) = b^\top \phi(a, x)$ for all $x$ and $a$ such that
$\tilde{\pi}(x) =a$.
With discrete covariates, this is a linear system of equations. Slightly
abusing notation, define $\phi(A, X) \in \R^{p \times d K}$ as the matrix of
values $\phi(a, x)$ for the $p$ unique combinations observable in the data,
and $\tilde{m} \in \R^p$ as the corresponding values of $\tilde{m}(x)$.
If the model class is not point identified (e.g. if $p < dK$), then there will be infinitely
many solutions to the equation $\tilde{m} = \phi(A, X) b$.
To characterize these, define
$\beta^\ast$ as the \emph{minimum norm solution}:
\begin{align*}
  \min_{b \in \R^d} \; & \; \|b\|_2\\
  \text{subject to } \; & \tilde{m} = \phi(A, X) b.
\end{align*}

There will also be an unidentified component arising from the
\emph{null space} of the system of linear equations,
$\mcal{N} = \{b \in \R^d \mid \phi(A, X) b = 0\}$.  Let
$D \in \R^{d \times d^\perp}$ be an orthonormal basis for this null
space.  Then, any solution to the linear equations
$\tilde{m} = \phi(A, X) b$ can be written as the minimum norm solution
$\beta^\ast$ plus a vector in the null space, which we can write as
$D b_\mathcal{N}$, where $b_\mathcal{N}$ are free parameters.
Therefore, we can re-write the restricted model in terms of these free
parameters:
\begin{equation*}
  \mcal{M} \ = \ \{f(a, x) = (\beta^\ast + D
  b_{\mcal{N}})^\top \phi(a, x) \;\; \mid \;\; b_{\mcal{N}} \in \R^{d^\perp}\}.
\end{equation*}

Finding the worst-case value will
involve a non-linear optimization over $b_\mcal{N}$.
Rather than taking such an approach, we will
consider a larger class
$\overline{\mcal{M}} \equiv \{f \; \mid \; B_\ell(a, x) \leq f(a, x)
\leq B_u(a,x)\}$ that contains the restricted model class
$\mcal{M}$.
To construct it, we choose the upper and lower bounds
\begin{align*}
  B_\ell(a,x) & = {\beta^\ast}^\top \phi(a,x) \bbone\{D^\top\phi(a,x) = 0\} + \bbone\{D^\top\phi(a,x) \neq 0\} L,\\
  B_u(a,x) & = {\beta^\ast}^\top \phi(a,x) \bbone\{D^\top\phi(a,x) = 0\} + \bbone\{D^\top\phi(a,x) \neq 0\} U.
\end{align*}
For a given action $a$ and covariate vector $x$, we first check whether
$\phi(x,a)$ is in the null space $\mathcal{N}$ by checking whether $D^\top \phi(a, x) = 0$. If it is not in the null space (i.e. $D^\top \phi(a, x) = 0$),
then the lower and upper bounds are equal, $B_\ell(a,x) = B_u(a,x) = h^{-1}({\beta^\ast}^\top \phi(a,x))$ because for any choice of free parameter $b_\mathcal{N}$,
$b_\mathcal{N}^\top D^\top \phi(a,x) = 0$.
In contrast, if $\phi(a,x)$ is in the null space (i.e. $D^\top \phi(a, x) \neq 0$), then the free parameter is unrestrained and $(\beta^\ast + D
b_{\mcal{N}})^\top \phi(a, x)$ can take on any value between $L$ and $U$.

To construct the empirical model class we again begin with a
simultaneous confidence band, this time for the minimum norm
prediction,
$\beta^\ast \cdot \phi(a,x) \in [\widehat{C}_{\alpha \ell}(a, x),
\widehat{C}_{\alpha u}(a, x)]$
where we apply a Bonferroni correction for the $p$ unique observed values
\begin{equation*}
  \beta^\ast \cdot \phi(a,x) \in \hat{\beta}^\ast \cdot \phi(a,x)  \pm \hat{\sigma} t_{n-p-1, 1 - \frac{\alpha}{2p}}\sqrt{\phi(a, x)^\top(\Phi^\top\Phi)^\dagger\phi(a, x)},
\end{equation*}
where $\hat{\beta}^\ast$ is the least squares estimate of the minimum
norm solution, $\hat{\sigma}^2$ is the estimate of the variance from
the MSE,
$\Phi = [\phi(\tilde{\pi}(x_i), x_i)]_{i=1}^n \in \R^{n \times d}$ is
the design matrix,$t_{n-p-1, 1 - \frac{\alpha}{2p}}$ is the
is the $1 - \alpha / p$ quantile of an $t$ distribution $n-p-1$
degrees of freedom, and $A^\dagger$ denotes the pseudo-inverse of a
matrix $A$. This gives lower and upper bounds,
\begin{equation*}
  \begin{aligned}
    \widehat{B}_{\alpha\ell}(a,x) & =
    \max\{L, \widehat{C}_{\alpha \ell}(a,x)\}\bbone\{D^\top\phi(a,x) = 0\} + L \bbone\{D^\top\phi(a,x) \neq 0\},\\
    \widehat{B}_{\alpha u}(a,x) & = \min\{U, \widehat{C}_{\alpha u}(a,
    x)\} \bbone\{D^\top\phi(a,x) = 0\} + U \bbone\{D^\top\phi(a,x) \neq
    0\},
  \end{aligned}
\end{equation*}
where we enforce the constraint that the predictions must be between $L$ and $U$ post-hoc.

\subsection{Additive models}
\label{ex:add}

If the model class for action $a$ consists of additive models, we have
\begin{equation*}
  \mcal{F} = \left\{f(a, x) = \sum_{j=1}^d f_j(a,x_j) + \sum_{j < k}
    f_{jk}(a, (x_j, x_k)) + \ldots \;\; \Bigl | \;\; f_{j}(a,\cdot), f_{jk}(a, \cdot), \ldots, \lambda_a-\text{Lipschitz}\right\},
\end{equation*}
where the component functions
$f_{j}(a,\cdot), f_{jk}(a, \cdot), \ldots$ can be subject to
additional restrictions so that the decomposition is unique.  This
additive decomposition formulation amounts to an assumption that no
interactions exist above a certain order.

By using the same additive decomposition for $\tilde{m}(x)$ into
$\tilde{m}(x) = \sum_j \tilde{m}_j(X_j) + \sum_{j < k}
\tilde{m}_{jk}(X_j, X_k) + \ldots$, we can follow the same bounding
approach as in Appendix~\ref{sec:lip} for each of the component
functions.  For example, for the additive term for covariate $j$,
$m_j(a,x_j)$, the Lipschitz property implies that,
\begin{equation*}
  \tilde{m}_j(x'_j) - \lambda_a |x_j - x'_j|\leq m_j(a,x_j) \leq \tilde{m}(x'_j) +
  \lambda_a|x_j - x'_j| \;\;\; \forall  \; x'  \in \tilde{\mcal{X}}_a.
\end{equation*}
Taking the greatest lower bound and least upper bound for each
component function, the overall lower and upper bounds are,
\begin{equation}
  \label{eq:add}
  \begin{aligned}
    B_\ell(a,X) & \ = \ \sum_j \sup_{x' \in \tilde{\mcal{X}}_a}\left\{ m_j(x_j') - \lambda_a |X_j - x_j'| \right\}+ \sum_{j < k} \sup_{x' \in \tilde{\mcal{X}}_a}\left\{ m_{jk}(x_j',x_k') - \lambda_a \|X_{(j,k)} - x_{(j,k)}'\| \right\} + \cdots\\
    B_u(a,X) & \ = \sum_j \inf_{x' \in \tilde{\mcal{X}}_a}
    \left\{m_j(x_j') + \lambda_a |X_j - x_j'|\right\} + \sum_{j < k}
    \inf_{x' \in \tilde{\mcal{X}}_a} \left\{m_{jk}(x_j',x_k') +
      \lambda_a \|X_{(j,k)} - x_{(j,k)}'\|\right\} + \cdots,
  \end{aligned}
\end{equation}
where $x_{(j,k)}$ is the subvector of components $j$ and $k$ of $x$.
Unlike in Appendix~\ref{sec:lip}, this extrapolates covariate by
covariate, finding the tightest bounds for each component.  For
instance, for a first-order additive model, the level of extrapolation
depends on the distance in each covariate $|x_j - x_j'|$ separately.
  
To construct the empirical model class for the class of additive
models, we use a $1-\alpha$ confidence interval that holds
simultaneously over all values of $x$ and for all components, i.e.,
\begin{equation*}
  \tilde{m}_j(x_j) \in \widehat{C}_{\alpha}^{(j)}(x_j), \;\; m_{jk}(x_j, x_k) \in \widehat{C}_{\alpha}^{(j,k)}(x_j, x_k), \ldots, \;\; \forall \;j=1,\ldots,d,\;\; k < j, \;\; \ldots,
\end{equation*}
with probability at least $1 - \alpha$. Analogous to the Lipschitz
case in Appendix~\ref{sec:lip} above, we can then construct the lower
and upper bounds using the lower and upper bounds of the confidence
intervals,
\begin{equation*}
  \begin{aligned}
    \widehat{B}_{\alpha \ell}(a,X) & = \sum_j \sup_{x' \in \tilde{\mcal{X}}_a}\left\{ \widehat{C}_{\alpha \ell}^{(j)}(x_j') - \lambda |X_j - x_j'| \right\}+ \sum_{j < k} \sup_{x' \in \tilde{\mcal{X}}_a}\left\{ \widehat{C}_{\alpha \ell}^{(j,k)}(x_j',x_k') - \lambda \|X_{(j,k)} - x_{(j,k)}'\| \right\} + \ldots\\
    B_{\alpha u}(a,X) & = \sum_j \inf_{x' \in \tilde{\mcal{X}}_a}
    \left\{\widehat{C}_{\alpha u}^{(j)}(x_j') + \lambda |X_j -
      x_j'|\right\} + \sum_{j < k} \inf_{x' \in \tilde{\mcal{X}}_a}
    \left\{\widehat{C}_{\alpha u}^{(j,k)}(x_j',x_k') + \lambda
      \|X_{(j,k)} - x_{(j,k)}'\|\right\} + \ldots.
  \end{aligned}
\end{equation*}

\section{Incorporating human decision-making}
\label{sec:human_decisions}

The PSA-DMF system we study is an example of a ``human-in-the-loop''
framework: rather than an algorithmic policy being the final arbiter
of decisions, the policy merely provides recommendations to a human
that makes an ultimate decision \citep{Imai2020, benmichael_ai_2024}.
In this section, we formalize and extend the potential outcomes framework
to incorporate human decisions, and then
briefly explore how our framework can be extended to explicitly model
human decisions and apply it to learn a new NVCA system.

\subsection{Potential human decisions and potential outcomes}
\label{sec:human_dec_setup}
We first show how to extend our framework to incorporate human decisions.
Let $D_i(a) \in \{0,1\}$ be the potential (binary) decision for
individual $i$ under action $a \in \mcal{A}$ (an algorithmic
recommendation in our application), and $Y_i(d, a) \in \{0, 1\}$ be
the potential (binary) outcome for individual $i$ under human decision
$d \in \{0, 1\}$ and algorithmic action $a \in \mcal{A}$.
This setup nests our main framework. To see this, note that we can re-define the
the potential outcome under algorithmic action $a$ as the potential outcome
when the algorithmic action is set to $a$ and the human decision is the natural
value under algorithmic action $a$:
\[
  Y_i(a) \equiv Y_i(D_i(a), a) = Y_i(0, a) (1 - D(a)) + Y_i(1, a) D(a).
\]
If the human decision under algorithmic action $a$ is $D(a) = 0$,
then the potential outcome under algorithmic action $a$ is $Y_i(a) = Y_i(0, a)$.
Conversely, if the human decision under algorithmic action $a$ is $D(a) = 1$,
the potential outcome under algorithmic action $a$ is $Y_i(a) = Y_i(1,a)$.
Then, the observed decision is
given by $D_i = D(\tilde{\pi}(X_i))$ whereas the observed outcome is
$Y_i = Y_i(\tilde{\pi}(X_i)) = Y_i(D_i(\tilde{\pi}(X_i)),
\tilde{\pi}(X_i))$.

Finally, we denote the expected
potential human decision under algorithmic action $a$,
conditional on covariates $x$, as
$d(a, x) = \E[D(a) \mid X = x]$ and represent the conditional expectation
of the potential outcome under algorithmic action $a$, conditional on covariates
$x$, as $m(a,x) = \E[Y(a) = 1 \mid X = x]$.

\subsection{Incorporating human decisions into the utility function}

To incorporate human decisions into the utility function, we write
the utility for outcome $y$
under human decision $d$ as $u(y,d)$.  With this setup, the value for a
policy $\pi$ is:
\begin{equation*}
  V(\pi) \ = \ \E\left[\sum_{a \in \mcal{A}} \pi(X, a) \sum_{d = 0}^1 \left[u(1, d) Y(d, a)  + u(0, d)(1 - Y(d,a))\right]\bbone\{D(a) = d\}\right].
\end{equation*}

If we make the simplifying assumption that the utility gain is
constant across decisions, i.e., $u(1,d) - u(0,d) = u$ for
$d \in \{0,1\}$, we can index the utility for $y = 0$ and $d= 0 $ as
$u(0,0) = 0$, and denote the added cost of taking decision 1 as
$c = u(0,1) - u(0,0)$.  This allows us to write the value by
marginalizing over the potential decisions, yielding,
\begin{equation}
  \label{eq:value_decisions}
  \begin{aligned}
    V(\pi) &\ = \ \E\left[\sum_{a \in \mcal{A}} \pi(X, a) \left(u Y(a) + cD(a) \right)\right].
  \end{aligned}
\end{equation}
  
Comparing Equation~\eqref{eq:value_decisions} to the value in
Equation~\eqref{eq:value_model} when actions are taken directly, we
see that the key difference is the inclusion of the potential decision
$D(a)$ in determining the cost of an action. Rather than directly
assigning a cost to an action $a$, there is an indirect cost
associated with the eventual decision $D(a)$ that action $a$ induces
in the decision maker. Therefore, the unidentifiability of the
expected potential decision under an action given the covariates,
$d(a,x)$, also must enter the robustness procedure.

We can treat the unidentifiability of the potential decisions in a
manner parallel to the outcomes. Denoting the conditional expected
observed decision as $d(\tilde{\pi}(x), x)=\E[D \mid X = x]$, we can
posit a model class for the decisions $\mcal{F}^\prime$ and create the
restricted model class
$\mcal{D} = \{f \in \mcal{F}^\prime \mid f(\tilde{\pi}(x), x) =
d(\tilde{\pi}(x), x)\}$.\footnote{These restrictions being on the
  \emph{decisions} gives more opportunities for structural
  restrictions on the model. For example, we could make a monotonicity
  assumption that $d(a,x) \leq d(a', x)$ for $a \leq a'$.}  We can now
construct a population safe policy by maximizing the worst case value
across the model classes for both the outcomes $\mcal{M}$ and the
decisions $\mcal{D}$,
\begin{equation}
  \label{eq:maximin_decisions}
  \begin{aligned}
    \max_{\pi \in \Pi}  & \left\{ \E\left[\sum_{a \in \mcal{A}} \pi(X, a)\tilde{\pi}(X, a) u Y \right] + \min_{f \in \mcal{M}} \E\left[\sum_{a \in \mcal{A}} \pi(X, a)\{1 - \tilde{\pi}(X, a)\} uf(a,X)\right] \right. \\
    & \left. + \E\left[\sum_{a \in \mcal{A}} \pi(X, a)\tilde{\pi}(X, a) c D \right] + \min_{g \in \mcal{D}} \E\left[\sum_{a \in \mcal{A}} \pi(X, a)\{1 - \tilde{\pi}(X, a)\} cg(a,X)\right] \right\}.
  \end{aligned}
\end{equation}

By allowing for actions to affect decisions through the decision maker
rather than directly, the costs of actions are not fully
identified. Therefore, we now find the worst-case expected outcome
\emph{and decision} when determining the worst case value in
Equation~\eqref{eq:maximin_decisions}. In essence, we solve the inner
optimization twice: once over outcomes for the restricted outcome
model class $\mcal{M}$ and once over decisions for the restricted
decision model class $\mcal{D}$.
  
From here, we can follow the development in the previous sections. We
create empirical restricted model classes for the outcome and decision
functions, $\widehat{\mcal{M}}_n(\alpha / 2)$ and
$\widehat{D}_n(\alpha / 2)$ using a Bonferonni correction so that
$P(\mcal{M} \in \widehat{\mcal{M}}_n(\alpha / 2), \mcal{D} \in
\widehat{D}_n(\alpha / 2)) \geq 1 - \alpha$. Then, we solve the
empirical analog to Equation~\eqref{eq:maximin_decisions}. Finally, we
can incorporate experimental evidence as above.  In this case, the
conditional expected potential decision $d(a,x)$ and outcome $m(a,x)$
--- and their model classes --- are replaced with the conditional
average treatment effect on the decision $\E[D(a) - D(-1) \mid X = x]$
and on the outcome $\tau(a,x)$.

\subsection{Learning a new NVCA point system}
  
In Section~\ref{sec:application}, we only considered the outcomes of
triggering the NVCA flag and have assigned costs directly to the
flag. However, the PSA serves as a recommendation to the presiding
judge who is the ultimate decision maker. Following the discussion
above, we can incorporate this into the construction of the robust
policy.  Rather than place a cost on triggering the NVCA flag, we use
the judge's binary decision of whether to assign a signature bond or
cash bail and place a cost of $-1$ to assigning cash bail. Unlike the
cost directly placed on the NVCA flag, this allows us to address the
cost of cash bail decision.  As discussed in
Section~\ref{sec:application}, the cost of the judge's decision to
assign cash bail includes the fiscal and socioeconomic costs, indexed
to be $-1$.
  
Following the same analysis as in Section~\ref{sec:application}, we
find maximin policies that take the decisions into account for
increasing costs of an NVCA relative to assigning cash bail, at
various confidence levels. For the additive and second order effect
models, however, we find policies that differ from the original rule
only when we do not take the statistical uncertainty into account ---
with confidence level $1 - \alpha = 0$ --- and have no finite sample
guarantee that the new policy is not worse than the existing rule. In
this case, the policy is extremely aggressive, responding to noise in
the treatment effects. Otherwise, we cannot find a new policy that
safely improves on the original rule. This is primarily because the
overall effects of the PSA on both the judge's decisions and
defendant's behavior are small \citep{Imai2020}. Therefore, there is
too much uncertainty to ensure that a new policy would reliably
improve upon the existing rule.

\section{Imputation, IPW, and double robust methods}
\label{sec:impute}

Here we briefly discuss how standard approaches to policy learning are not
applicable in our setting.
First, as discussed in Section~\ref{sec:opt_pol_learn}, the key identification
issue is that we can cannot point-identify the conditional expectation of the
potential outcome $m^\ast(a, x) = \E[Y(a) \mid X = x]$ for all pairs of actions
$a$ and covariates $x$.
In settings with overlap ($P(A = a \mid X = x) > 0$ for all $a \in \mathcal{A}$
and $x \in \mathcal{X}$), and unconfounded action assignment
($A \indep \{Y(0), Y(1),\ldots,Y(K-1)\} \mid X$), we can identify $m^\ast(a, x)$
via the conditional expectation of the observed outcome given the action and
the covariate
$\tilde{m}(a,x) \equiv \E[Y \mid A = a, X = x]$.
In such settings, we could then identify the value $V(\pi)$ using model-based
imputation, IPW, or augmented IPW:

\begin{align*}
  V(\pi, m^\ast) \ & = \ \E\left[\sum_{a \in
      \mathcal{A}} \pi(X, a) \{u(a) \tilde{m}(a, X) + c(a)\}\right] & \text{(Imputation)}\\
      & = \E\left[\sum_{a \in
      \mathcal{A}} \pi(X, a) \left\{u(a) \frac{\bbone\{A = a\}}{P(A = a \mid X)}Y + c(a)\right\}\right] & \text{(IPW)}\\
      & = \E\left[\sum_{a \in
      \mathcal{A}} \pi(X, a) \left\{u(a) \left(\tilde{m}(a, X) + \frac{\bbone\{A = a\}}{P(A = a \mid X)}(Y - m(a, X))\right) + c(a)\right\}\right] & \text{(AIPW)}
\end{align*}

In our setting, where the observed actions are the actions under the
deterministic baseline policy $A_i = \tilde{\pi}(X_i)$, the actions are
unconfounded given the covariates $X$ (indeed, we know exactly how the actions
are assigned), but there is no overlap because
$P(A = a \mid X) = P(\tilde{\pi}(X) = a \mid X)$ is either 0 or 1.
The implication is that the outcome model $m^\ast(a, x)$ is not point
identifiable. It is impossible to estimate the conditional
expectation of the observed outcome given $A = a$ and $X = x$,
$\tilde{m}(a,x)$, if $a \neq \tilde{\pi}(x)$ because it is an event of measure
zero (i.e. $P(A \neq \tilde{\pi}(X)) = 0$).

Nonetheless, we may try to use the imputation approach by estimating a model
$\hat{m}(a,x)$ and relying on it for extrapolation. We would then solve
\begin{equation}
  \label{eq:impute_estimator}
  \hat{\pi}^\text{impute} \in \max_{\pi \in \Pi} \frac{1}{n}\sum_{i=1}^n \sum_{a \in \mathcal{A}}\pi(X_i, a)\left\{u(a) \left(\tilde{\pi}(X_i, a) Y + (1 - \tilde{\pi}(X_i, a))\hat{m}(a, X_i)\right) +c(a)\right\}.
\end{equation}
This imputation-based policy will be highly sensitive to how the estimated
model $\hat{m}(a, X_i)$ extrapolates to combinations of $a$ and $x$ that are
not possible under the baseline policy, as we show via simulation
in Section~\ref{sec:sims}.

The identification problem is more transparent for the IPW and AIPW-based
approaches. Note that the inverse probability term with a deterministic baseline
policy is $\bbone\{A = a\}/\tilde{\pi}(a, X_i)$, which is equal to
$\tilde{\pi}(a, X_i)/\tilde{\pi}(a, X_i)$. If $\tilde{\pi}(a,X_i) = 1$,
then this term is equal to 1, but if $\tilde{\pi}(a,X_i) = 0$, it is 0/0,
which is undefined. Again, we may nonetheless try to use IPW by setting 0/0 = 0.
This would give:
\[
  \hat{\pi}^\text{ipw} \in \max_{\pi \in \Pi} \frac{1}{n}\sum_{i=1}^n \sum_{a \in \mathcal{A}}\pi(X_i, a)\left\{u(a) \tilde{\pi}(X_i, a) Y  + c(a)\right\}.
\]
As long as $u(a) > c(a)$, then defining the IPW-based policy in this way will
give that $\hat{\pi}^\text{ipw}  = \tilde{\pi}$, and so we will always keep
the baseline policy.

Finally, we might try to consider the AIPW estimator, again setting 0/0=0, but
note that 
\[
\hat{m}(a, X_i) + \tilde{\pi}(X_i, a) (Y_i - \hat{m}(a, X_i)) = \tilde{\pi}(X_i, a) Y_i - (1-\tilde{\pi}(X_i, a)) \hat{m}(a, X_i),
\]
and so the AIPW approach would recover the model-based imputation approach.

\section{Simulation study}
\label{sec:sims}

We have a single discrete covariate with 10 levels, $x \in \{0,\ldots,9\}$,
and a binary action so that the action set is $\mathcal{A} = \{0,1\}$.
We choose a baseline policy $\tilde{\pi} = \bbone\{x \geq 5\}$, and set the
utility gain to be $u(0) = u(1) = 10$ and the costs to be $c(0) = 0, c(1) = -1$,
so that action 0 is costless and action 1 costs one tenth of the potential
utility gain.
For each
simulation we draw $n$ i.i.d. samples $X_1,\ldots,X_n$ uniformly on
$\{0,\ldots,9\}$. Then we draw a smooth model for the expected control potential outcome
$m(0,x) \equiv \E[Y(0) \mid X = x]$ via random Fourier features.
We draw three random vectors: $\omega \in \R^{100}$ with i.i.d. standard normal
elements; $b \in \R^{100}$ with i.i.d. components drawn uniformly
on $[0,2\pi]$; and $\beta \in \R^{100}$ with i.i.d. standard normal elements.
Then we set
\[
  m(0,x) = \logit^{-1}\left(\sqrt{\frac{2}{100}} \beta \cdot  \cos\left(\omega \frac{x}{9} + b\right)\right),
\]
where the cosine operates element-wise. See \citet{Rahimi2008_random_features} for more discussion on random features.
For the potential outcome under treatment, $m(1, x) = \E[Y(1) \mid X = x]$, we add a linear treatment effect on the logit scale:
\[
  m(1,x) = \logit^{-1}\left(\logit\left(m(0,x)\right) + \frac{1}{2}\left(x - \frac{9}{2}\right) - \frac{8}{10}\right).
\]
We then generate the potential outcomes $Y_i(0),Y_i(1)$ as independent Bernoulli draws with probabilities $m(0, X_i)$ and $m(1, X_i)$, respectively.

With each simulation draw, we consider finding a safe empirical policy by solving
Equation~\eqref{eq:maximin_emp_bnd} under a Lipschitz restriction on the model as in
Appendix~\ref{sec:lip} and with the threshold policy class $\Pi_\text{thresh}$. Note that the true model is in fact much smoother than
Lipschitz; here we consider using the looser assumption. 
Following our empirical analysis in
Section~\ref{sec:nvca_threshold}, we take the average outcome at each
value of $x$, and compute the largest difference in consecutive averages as
pilot estimates for the Lipschitz constants $\lambda_0$ and $\lambda_1$.
We then solve Equation~\eqref{eq:maximin_emp_bnd} using  $\frac{1}{2}$, 1, and 2 times
these pilot estimates as the Lipschitz constants, and setting the significance
level to 0, 80\% and 95\%.

We also consider using a model-based imputation estimator without
accounting for partial identification. Because the baseline policy assigns
0 for $x \in \{0, 1, 2, 3, 4\}$ and 1 for $x \in \{5, 6, 7, 8, 9\}$, there are
5 unique values of the covariate when $\tilde{\pi}(x)$ is 0 or 1.
Therefore, we fit two separate non-parametric models for $\hat{m}(0, x)$ and
$\hat{m}(1, x)$ by fitting a logistic regression of $Y$ on $X$ with a degree
four polynomial of $X$. This creates 5 parameters for each model, one for
each unique observed data point. We then use each estimated 4-degree polynomial
logistic regression model to extrapolate $\hat{m}(0,x)$ for $x \geq 5$ and
$\hat{m}(1,x)$ for $x < 5$ and estimate an imputation-based policy
$\hat{\pi}^\text{impute}$ solving Equation~\eqref{eq:impute_estimator}.
We additionally compute the oracle threshold policy
that uses the true model values $m(0,x)$ and $m(1,x)$. We do this for sample
sizes $n \in (500, 1000, 1500, 2000)$.

\begin{figure}[t!]
  \centering\vspace{-.5in}
        \begin{subfigure}[t]{0.45\textwidth}  
    {\centering \includegraphics[width=\maxwidth]{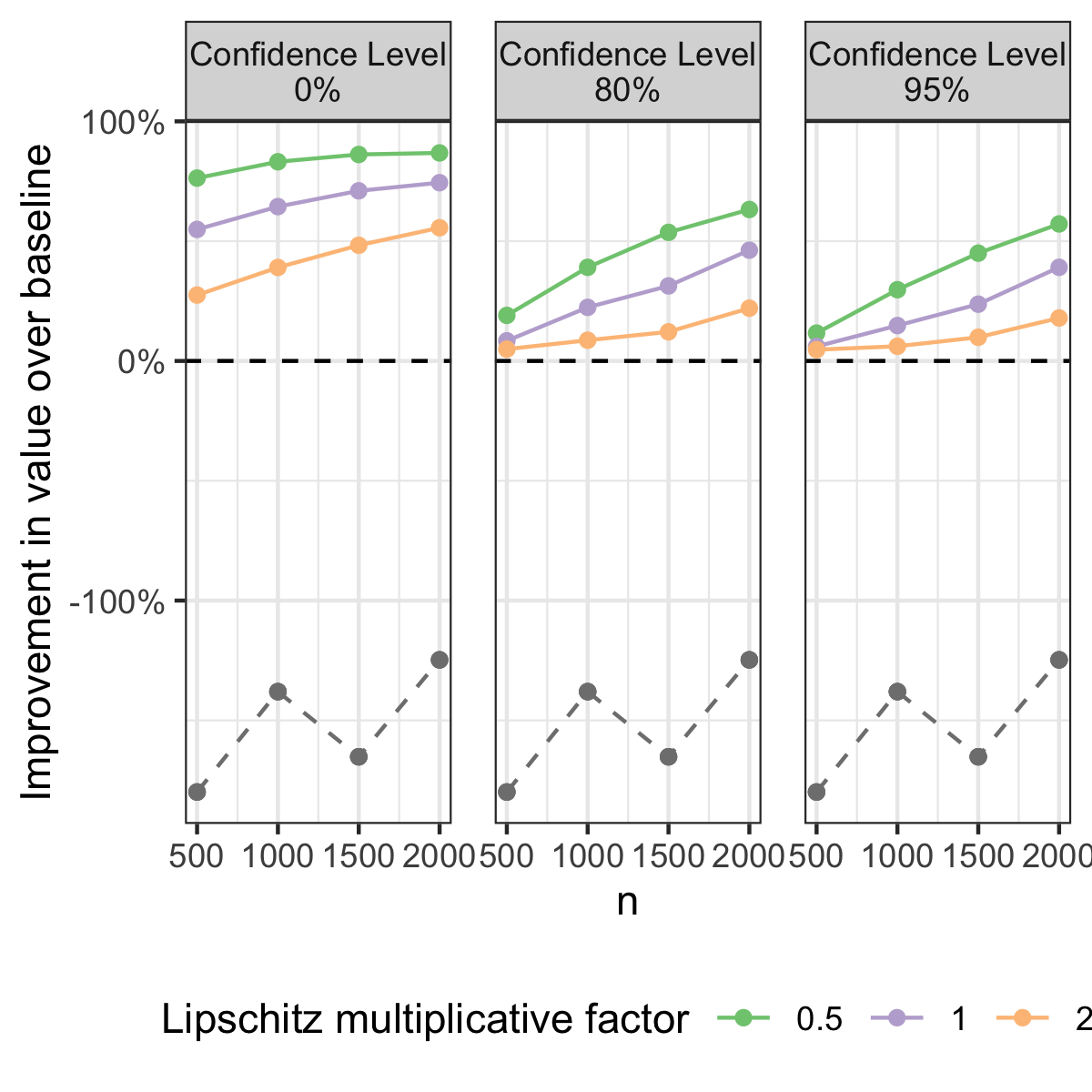} 
    }
  \end{subfigure}
    \begin{subfigure}[t]{0.45\textwidth}  
  {\centering \includegraphics[width=\textwidth]{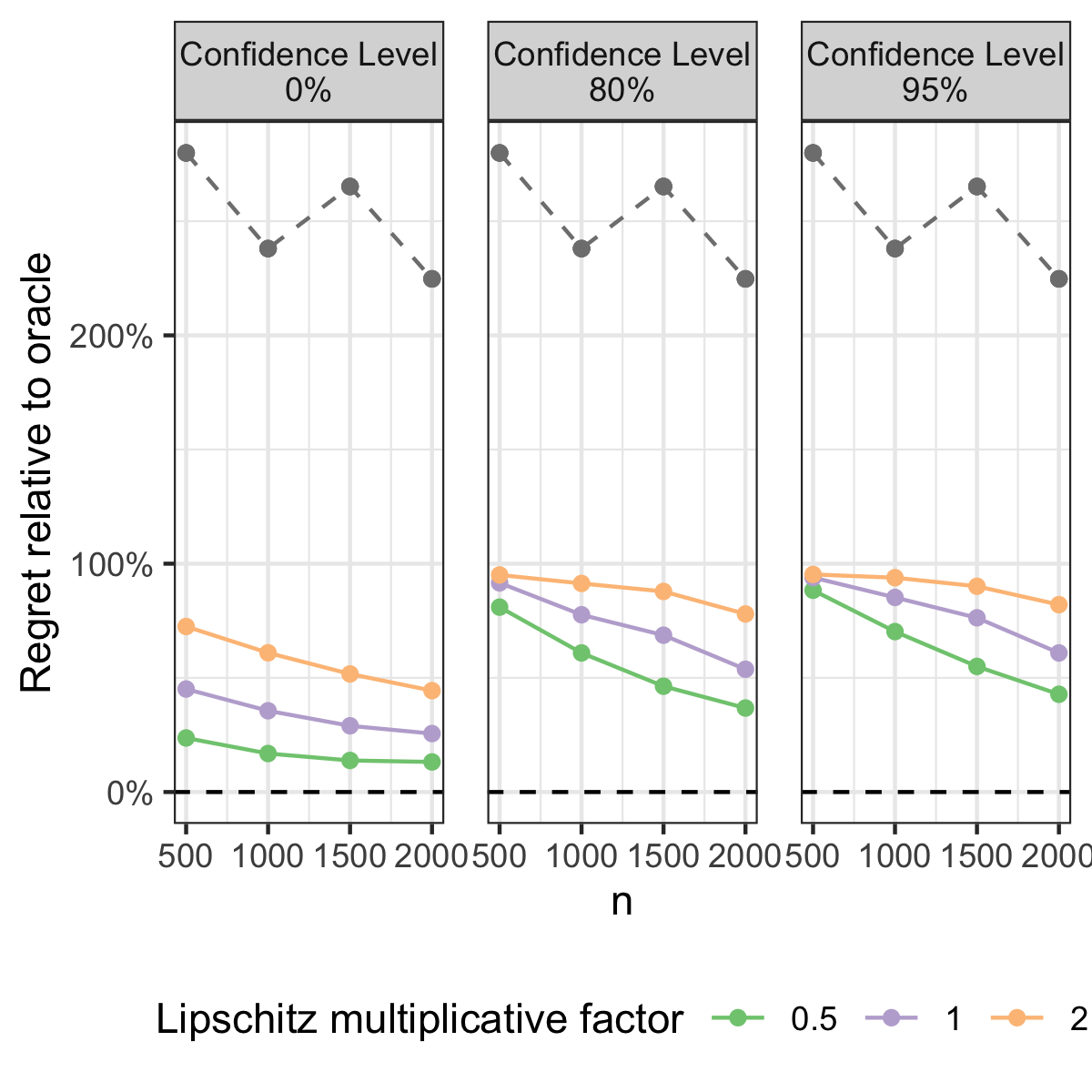} 
  }
    \end{subfigure}\quad
  \vspace{-.2in}
  \caption{Monte Carlo simulation results as the sample size $n$ increases,
           varying the multiplicative factor on the empirical Lipschitz
           constant and the significance level $1-\alpha$. The left panel
           shows the difference in the expected utility between
           the empirical safe policy $\hat{\pi}$, and the baseline policy
           $\tilde{\pi}$, normalized by the regret of the baseline relative to
           the oracle, i.e. $\frac{V(\hat{\pi}) - V(\tilde{\pi})}{V(\pi^\ast) - V(\tilde{\pi})}$.
           The right panel shows the regret of the safe policy relative to the
           oracle, scaled by the regret of the baseline relative to the oracle,
           i.e. $\frac{V(\pi^\ast) - V(\hat{\pi})}{V(\pi^\ast) - V(\tilde{\pi})}$.
           In both panels, the grey dashed line represents the imputation-based
           policy.
           }
  \label{fig:sim_results}
\end{figure}

Figure~\ref{fig:sim_results} shows how the empirical safe policy $\hat{\pi}$
and the model-based imputation policy $\hat{\pi}^\text{impute}$
compare to both the baseline policy $\tilde{\pi}$ and the oracle policy
$\pi^\ast$ in terms of expected utility.
First, we see that on average, the empirical safe policy improves over the
baseline, no matter the confidence level and the choice of Lipschitz constant.
This improvement is larger the less conservative we are, e.g. by choosing a lower
confidence level or a smaller Lipschitz constant. Furthermore, as the sample size
increases, the utility of the empirical safe policy also increases due to a lower
degree of statistical uncertainty.
We find similar behavior when comparing it to the oracle policy. Less conservative
choices lead to lower regret, and the regret decreases with the sample size.
Importantly, the regret does not decrease to zero; even when removing all
statistical uncertainty the safe policy can still be suboptimal due to
the lack of identification.

In contrast, model-based imputation without accounting for identification
issues performs poorly, yielding a policy that has much 
lower expected utility than the baseline, let alone the oracle.
This is because the extrapolation to unseen data does not perform well with the modeling
approach that we used.
It could have been  possible to choose an imputation estimator
that performs better in that the extrapolation proved to be correct.
However, for any imputation estimator we can come up with an adversarial
example where the extrapolation is incorrect and leads to a worse policy than
the status quo. Indeed, this is precisely what the maximin criterion is
designed to defend against.

\section{Additional empirical results}
\label{sec:addl_results}

In this section, we present additional empirical results for the FTA, NCA, and NVCA scores, as well as the results for the combined bail level and
monitoring conditions recommendation.
For reference, Table~\ref{tab:nvca_weights} displays the existing risk-factor weights for
the FTA, NCA, and NVCA risk scores.

\begin{table}[t!]
  \centering
  \begin{tabular}{l l r r r}
    \toprule
    Risk factor & & FTA & NCA  & NVCA\\
    \midrule
    \multirow{2}{*}{Current violent offense} & $>$ 20 years old & & & 2\\
    & $\leq$ 20 years old & & & 3\\
    Pending charge at time of arrest & & 1 & 3& 1\\
    \midrule
    \multirow{2}{*}{Prior conviction} & misdemeanor or felony & 1 & 1 & 1\\
    & misdemeanor and felony & 1 & 2 & 1\\
    \multirow{2}{*}{Prior violent conviction} & 1 or 2 & & 1& 1\\
    & 3 or more & & 2& 2\\
    Prior sentence to incarceration & & & 2 &\\
    \midrule
    \multirow{2}{*}{Prior FTA in past 2 years} & only 1 & 2 & 1& \\
    & 2 or more & 4 & 2 &\\
    Prior FTA older than 2 years & & 1 & &\\
    \midrule
    Age & 22 years or younger & & 2 &\\
    \bottomrule
  \end{tabular}
  \caption{Weights placed on risk factors to construct the failure to
    appear (FTA), new criminal activity (NCA), and new violent
    criminal activity (NVCA) scores.  The sum of the weights is then
    thresholded into six levels for the FTA and NCA scores and a
    binary ``Yes''/``No'' for the NVCA score. }
  \label{tab:nvca_weights}
\end{table}

\subsection{Additional results for the NVCA threshold and score}
\label{sec:addl_nvca}

\begin{figure}[t]
  \centering
    \centering \includegraphics[width=0.75\maxwidth]{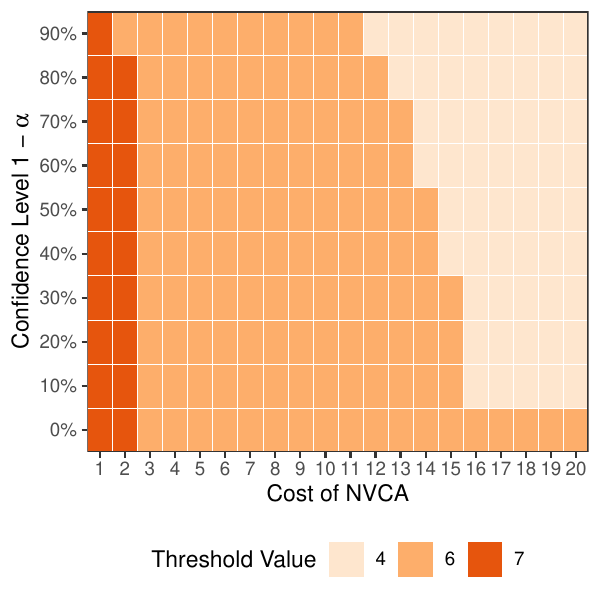} 
    
    \label{fig:threshold_lvl}
  \caption{Learned threshold values solving Equation \eqref{eq:maximin_emp_bnd}
    for the NVCA flag threshold rule as the cost of an NVCA increases from 1 to
    20 times the cost of triggering the NVCA flag, and the
    confidence level varies between 0\% and 90\%.}
  \label{fig:nvca_results_threshold}
\end{figure}

We begin by presenting the results regarding the NVCA threshold.
Figure~\ref{fig:nvca_results_threshold} shows how the maximin
threshold changes as we vary the confidence level $1-\alpha$ while setting $C=3$.
The
overall relationship between the threshold and the cost is robust to
the choice of confidence level.  The results show that when the cost
of an NVCA is low and/or the confidence level is low the learned safe
policy will raise the threshold, implying that fewer arrestees will
trigger the NVCA flag.

\begin{figure}[t]

  \centering
  \includegraphics[width = 0.75\maxwidth]{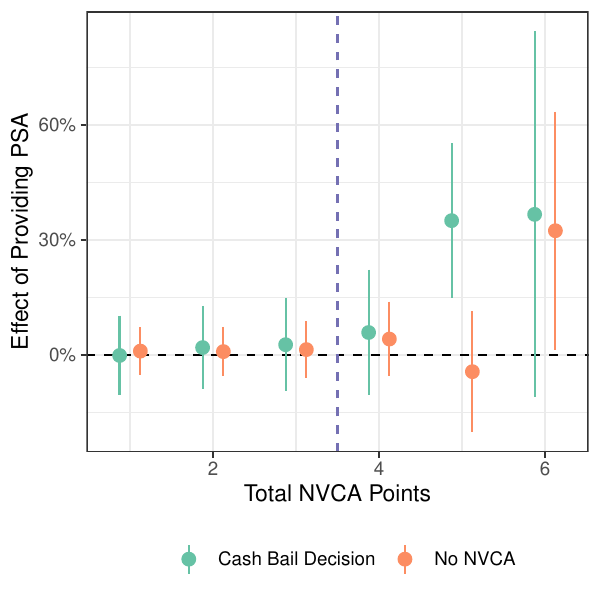}
  \caption{The effect of providing the PSA on (a) whether the judge makes a cash
  bail decision and (b) whether the arrestee does not engage in an NVCA,
  conditioned on the number of total NVCA points. Error bars
  indicate 95\% confidence intervals using heteroskedastic robust standard errors.
  The vertical dashed line represents the existing NVCA threshold.}
  \label{fig:nvca_bail_decision}
  
\end{figure}

Figure~\ref{fig:nvca_bail_decision} shows estimates of the effect of providing
the PSA on whether the judge makes a cash bail decision, and on whether the
arrestee engages in an NVCA,
conditioned on the number of total NVCA points. We find that when the NVCA
flag is not triggered (i.e. $x_\text{nvca} < 4$) there is little to no effect
of providing the PSA on either the judge's decision or the presence of an NVCA.
This appears to  remain true when  $x_\text{nvca} = 4$, even though the flag is triggered.
For $x_\text{nvca} \geq 5$, providing the PSA increases the proportion of
decisions that are cash bail by over 30 percentage points (though this is not
significant for $x_\text{nvca} = 6$.)
However, NVCAs do not meaningfully change for
$x_\text{nvca} = 5$, even though there are over 30 percentage points more
cash bail decisions, but they decrease for $x_\text{nvca} = 6$.

Next, we present several additional empirical results for the NVCA 
threshold and score. 
\paragraph{Second order effect model and model testing.}
First, Figure~\ref{fig:pct_diff_lvl_2nd} shows how the maximin NVCA
flag differs from the original rule as the cost of an NVCA and the
confidence level vary under the second order effect model.  We find
that under the second order effect model, there is too much
uncertainty to safely deviate from the original NVCA flag rule with
any reasonable degree of confidence if the cost of an NVCA is greater
than 1.  This is in contrast to the results under the additive effect
model shown in Figure~\ref{fig:pct_diff_lvl}; the addition of
unidentifiable second order interaction terms precludes safely
changing the policy.

\begin{figure}[t!]
  \centering
        \begin{subfigure}[t]{0.45\textwidth}  
    \caption{Second order effect model class}
    {\centering \includegraphics[width=\maxwidth]{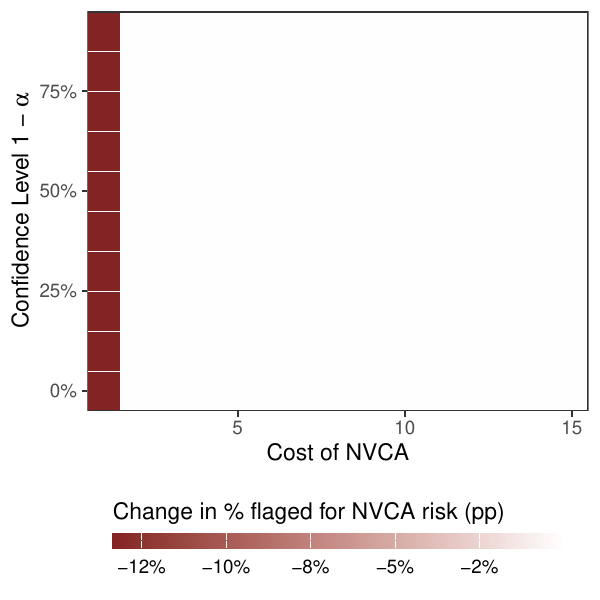} 
    }
    \label{fig:pct_diff_lvl_2nd}
  \end{subfigure}
    \begin{subfigure}[t]{0.45\textwidth}  
   \caption{Using all risk factors in Table~\ref{tab:nvca_weights}} 
  {\centering \includegraphics[width=\textwidth]{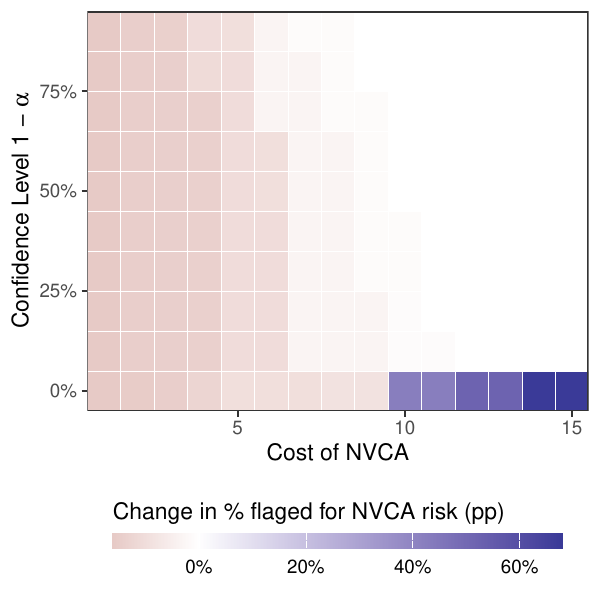} 
  }
  \label{fig:pct_diff_lvl_all}
  
    \end{subfigure}\quad
    \caption{The percentage point difference in the proportion of
      arrestees flagged for NVCA risk between the maximin policy and
      the original NVCA score as the cost of an NVCA increases from 1
      to 15 times of the cost of triggering the NVCA flag and the
      confidence level varies between 0\% and 100\% 
      (a) in the second order effect model
      class and (b) under the additive effect model class using all
      risk factors in Table~\ref{tab:nvca_weights}.}
  \label{fig:nvca_results_lvl_2nd}
\end{figure}

To understand whether the additive effects assumption is reasonable
for the NVCA rule, we estimate the CATE separately for arrestees with
and without the NVCA flag triggered via a similar spirit to the DR-learner
\citep{Kennedy2022_drlearner} by regressing the IP-weighted outcomes
$\Gamma(1,\bX,Y) - \Gamma(0, \bX, Y)$ on the 7 binary risk factors and all
observed pair-wise interactions. Note that this partial
second order model is point identified because it omits the
unidentified terms and so it is only a rough proxy for the full
second order model. We then test whether the interaction terms are
all zero using a Wald test with Huber-White heteroskedastic robust standard errors.
We do not find evidence against the null of the
additive model for cases where the flag is not triggered ($p=0.75$),
but there is some evidence for the existence of interactions when the
flag is triggered ($p = 0.067$).

\paragraph{Using a quadratic cost.}
We consider an alternative value function that assigns a larger
marginal utility loss to triggering the NVCA flag for an arrestee if a
larger proportion of arrestees have the flag triggered.  Formally,
defining $\bar{\pi} \equiv \E[\pi(X)]$, the policy value function is
given by:
\[
  V^\text{quad}(\pi) \equiv \E\left[\pi(X) \left\{ u \times (m^\ast(1, X) - m^\ast(0, X)) - (1 + \zeta \bar{\pi})\right\}\right] + \E[m^\ast(0, X)].
\]

This induces a \emph{quadratic} cost, with $\zeta$ determining the
additional marginal penalization per percent flagged as an NVCA risk.
Note that this value function is not an expectation of individual
utilities, because the cost of flagging one individual for NVCA risk
depends on how many other individuals are also flagged.  As with the
cost of an NVCA $u$, it is beyond the scope of this paper to argue for
a particular value of the quadratic penalty term $\zeta$, and so we
will document how the policy changes as it varies.
Note that other forms of such utilities are possible, for example, we
could consider a step function that adds an additional penalty if the
number of arrestees flagged as an NVCA risk exceeds some threshold.

Figure~\ref{fig:pct_diff_lvl_quad} shows how the maximin rule compares
to the original rule, again in terms of the the performance of the
maximin proportion of arrestees flagged for an NVCA risk as we vary
both $u$ and $\zeta$ while keeping the confidence level fixed to
$1 - \alpha = 80\%$.  For any given cost of an NVCA, the maximin
policy triggers the flag less often as the quadratic penalty
increases. 

Figure~\ref{fig:coef_by_cost_quad} shows the integer weights on the
risk factors for the maximin policy at the $1 - \alpha = 80\%$ level
as the quadratic penalty $\zeta$ increases with the cost of an NVCA
set to 9.  Increasing the quadratic penalty eventually changes the
maximin policy back to placing less weight on violent convictions and
offenses, similar to the results when we only vary the cost of an NVCA and keep
$\zeta = 0$ (e.g. in Figure~\ref{fig:coef_by_cost}).

\begin{figure}[t!]
  \centering
  \includegraphics[width=0.5\textwidth]{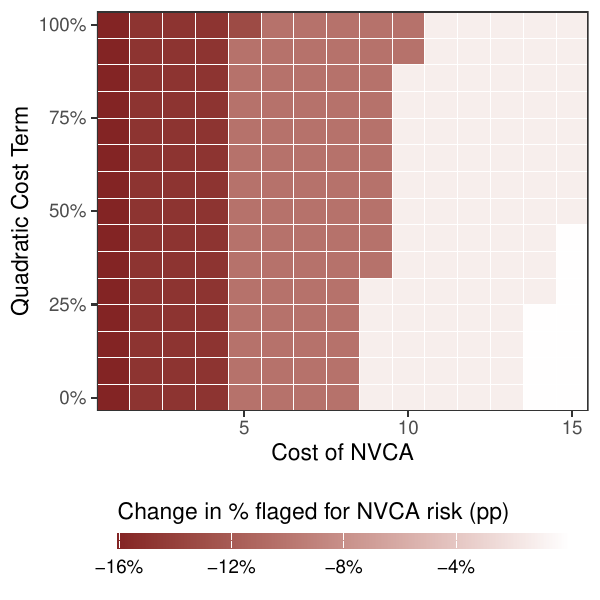} 
    
  \caption{The percentage point difference in the proportion of
    arrestees flagged for NVCA risk between the maximin policy and the
    original NVCA score as the cost of an NVCA increases from 1 to 15
    times of the cost of triggering the NVCA flag as $\zeta$ varies
    with a confidence level of 80\%.}
    \label{fig:pct_diff_lvl_quad}
\end{figure}

\begin{figure}[b!]
  \centering \includegraphics[width=\maxwidth]{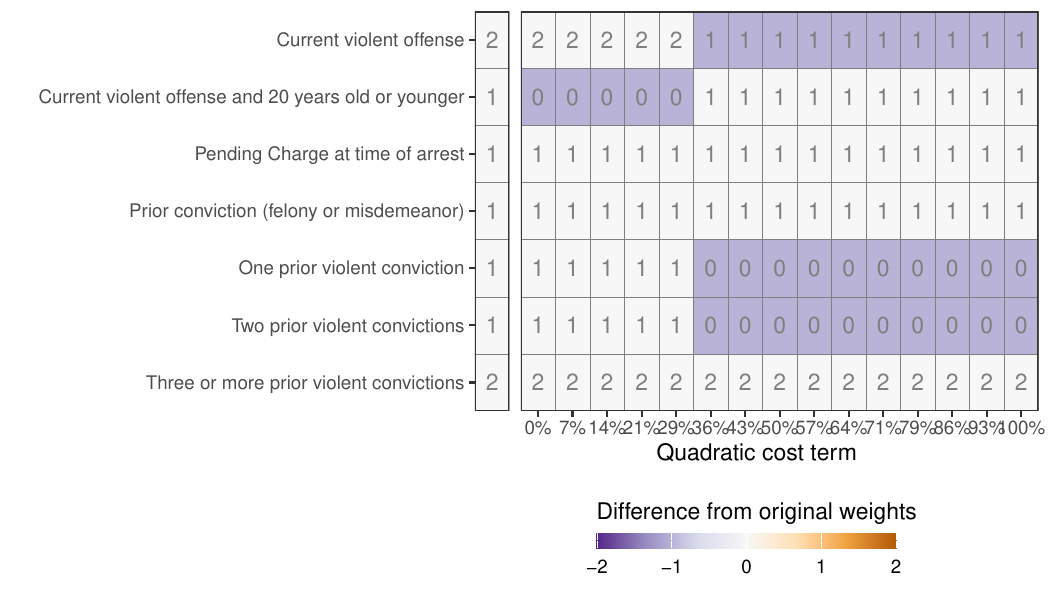}
  \caption{NVCA flag weights $\theta$ in
    Equation~\eqref{eq:integer_weight_policy}. Change in $\theta$ as
    the quadratic penalty $\zeta$ increases from 0 to with a cost an
    NVCA equal to 9 and a confidence level of 80\% (right panel). }
\label{fig:coef_by_cost_quad}
\end{figure} 

\paragraph{Using the full set of risk factors.}
We also consider learning a new NVCA flag rule that incorporates the
full set of risk factors listed in Table~\ref{tab:nvca_weights}.  The
scale of the weight placed on each factor is not necessarily
meaningful for comparisons across rules that use different risk
factors and thresholds.  For this reason, we place an upper bound on
the weights of 5.

Figure~\ref{fig:pct_diff_lvl_all} shows how the resulting maximin
rules differ from the original NVCA flag rule, again as the cost of an
NVCA and the confidence level vary under the additive effect model,
with a quadratic penalty term of zero, i.e., $\zeta = 0$.  We find
broadly similar results as when using the original reduced set of risk
factors. For all confidence levels at lower NVCA costs, the maximin
rule classifies fewer arrestees as NVCA risks, eventually collapsing
back to the status quo as the cost of an NVCA relative to the cost of
triggering the flag increases. Relative to the reduced covariate set,
including more risk factors increases the level of statistical
uncertainty, and so the maximin rule collapses back to the original
rule more quickly.

\begin{figure}[t!]
  \centering \includegraphics[width=0.8\maxwidth]{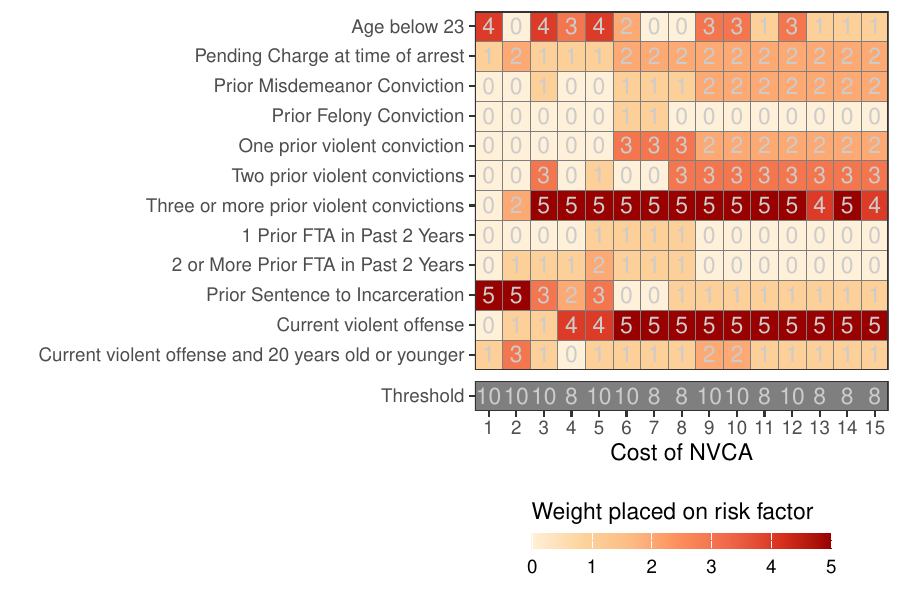}
  \caption{Change in the NVCA flag weights $\theta$ using all
    of the risk factors in Table~\ref{tab:nvca_weights} as the cost of
    an NVCA increases from 1 to 15 times the cost of triggering the
    NVCA flag, at a confidence level of $1 - \alpha = 80\%$ and no
    quadratic penalty $\zeta = 0$ .}
\label{fig:coef_by_cost_all}
\end{figure} 

Relative to the reduced covariate set, including more risk factors
increases the level of statistical uncertainty, and so the maximin
rule collapses back to the original rule more quickly.  In addition,
at a confidence level of 0\%, the learned NVCA flag rule eventually
begins to flag far more arrestees as NVCA risks than the original rule
as the cost of an NVCA increases. However, because more risk factors
are included, even when the maximin policy does not differ from the
baseline in terms of which arrestees it triggers the flag for, the
underlying risk factor weights can be different, as multiple
combinations of weights can produce the same recommendations.
Figure~\ref{fig:coef_by_cost_all} shows the set of weights found
during the optimization problem with a confidence level of 80\%, but
as the solutions are not unique and the scales arbitrary, these
weights are not directly comparable to the other sets of results.

\subsection{Additional results for the FTA and NCA scores}
\label{sec:addl_fta_nca}

Next, we present additional empirical results for the FTA and NCA
scoring systems.  We begin by formalizing the FTA and NCA policy
classes as follows:
\[
  \label{eq:integer_weight_policy}
  \Pi = \left\{\pi(x) = \sum_{a = 1}^{K-1}a \bbone\left\{\eta_{a-1} <
      \theta \cdot x \leq \eta_a \right\} \; \left \vert
      \vphantom{\sum_{j=1}^7 }\right. \; \theta \in \Z^d, \ \eta_a > \eta_{a-1}
    \geq 0 \ \forall a \in \{1, 2, \ldots, K-1\} \right\},
\]
where $x$ are the corresponding risk factors in either the FTA or NCA
rule, $\theta$ are the integer weights placed on the risk factors, and
$\eta_0,\ldots,\eta_{K-1}$ are thresholds that determine what the
final score is.  For example, the baseline FTA rule has thresholds (0,
1, 2, 4, 6, 7) and the baseline NCA rule has thresholds (0, 2, 4, 6,
8, 13).

There are $K=6$ possible actions for the FTA and NCA scores, each
giving scores between 1 and 6. Indexing the cost of the first action
to be zero, we must characterize the cost of the remaining 5
actions. There are many potential ways to do so. However, recall from
Figure~\ref{fig:nvca_widths} that there is little information to
extrapolate from the NCA score and none for the FTA score, so we do
not expect to be able to learn maximin policies that are different
from the status quo here.  Therefore, we extend our utility function
from the binary case to a simple linear parameterization of the costs,
writing the utility function as $u(y, a) = u \times y - a$ where $|u|$
is the cost of either an FTA or an NCA depending on the risk
score. This utility function and these costs are not directly
comparable to the binary utility function for the NVCA flag, because
the cost for choosing the highest score is indexed to 5 rather than 1
as in the binary case.\footnote{Recall that we index the first action
  to be $a = 0$.}  We note that it is straightforward to encode
different cost structures.

\begin{figure}[t!]
  \centering
        \begin{subfigure}[t]{0.45\textwidth}  
    \caption{FTA Score}
    {\centering \includegraphics[width=\maxwidth]{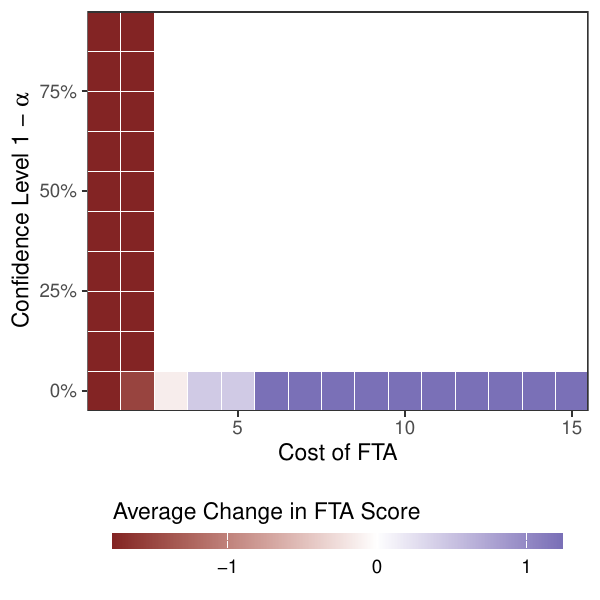} 
    }
    \label{fig:pct_diff_fta_lvl}
  \end{subfigure}
    \begin{subfigure}[t]{0.45\textwidth}  
   \caption{NCA Score} 
  {\centering \includegraphics[width=\textwidth]{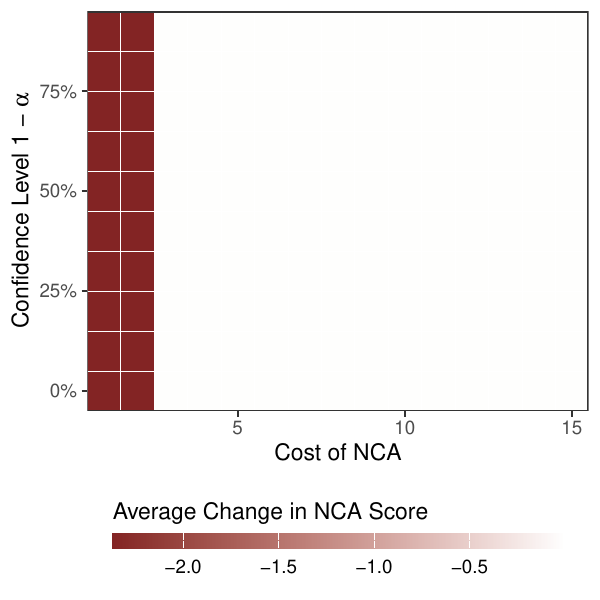} 
  }
    \label{fig:pct_diff_nca_lvl}
    \end{subfigure}\quad
  \caption{The average difference in (a) the FTA score and (b) the NCA score
    for arrestees under the maximin policy and the original FTA and NCA scores
    as the cost of an FTA (left panel) and NCA (right panel) increases from 1 to
    15 and the confidence level varies between 0\% and 100\%.}
  \label{fig:fta_nca_results_lvl}
\end{figure}

Figure~\ref{fig:fta_nca_results_lvl} shows how the maximin FTA and NCA
scores differ from the original rules as we vary the cost of an FTA or
NCA and the confidence level $1 - \alpha$. Overall, we find that with
any degree of statistical confidence, if the cost of an FTA or NCA is
above 2, the maximin rule collapses to the status quo rule.  This is
not surprising given the discussion in Section~\ref{sec:risk_scores}.

\begin{figure}[!t]
  \centering 
    \includegraphics[width = 0.9\maxwidth]{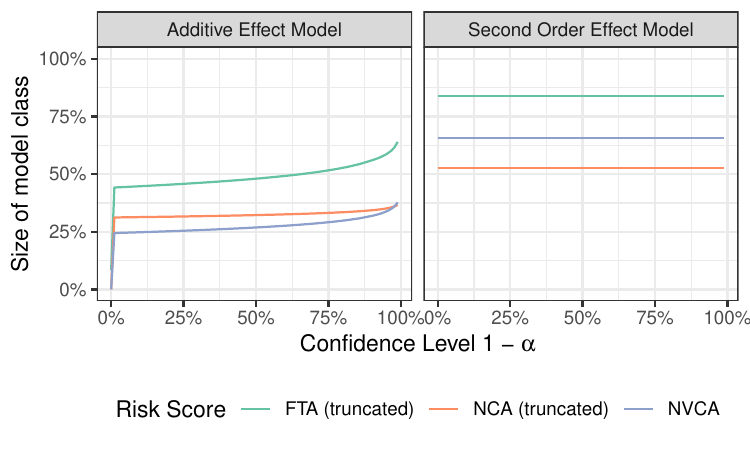}
    \caption{The size (as a percentage of its maximum
      value) of two different model classes with respect to the
      linear threshold policy class versus the confidence
      level $1-\alpha$ for the FTA (green) and NCA (orange), both truncated
      into an indicator for high risk (score greater than or equal to 4)
      and NVCA (purple) scoring rules.}
  \label{fig:nvca_widths_truncated}
\end{figure}

It may be possible, however, to learn simplified versions of the FTA
and NCA scores that are collapsed into low and high risk. To inspect
this, we create truncated versions of the scores that are indicators
for whether the scores are greater than or equal to 4.
Figure~\ref{fig:nvca_widths_truncated} shows the sizes of the
resulting model classes with respect to the truncated policy classes
for both the additive and second order effect models as the confidence
level varies, keeping the NVCA flag for comparison. We find that
truncating the scores leads to much smaller model classes. This suggests that it might be possible to
learn maximin policies that deviate from the status quo.

\begin{figure}[t!]
  \centering
        \begin{subfigure}[t]{0.45\textwidth}  
    \caption{FTA Score (truncated)}
    {\centering \includegraphics[width=\maxwidth]{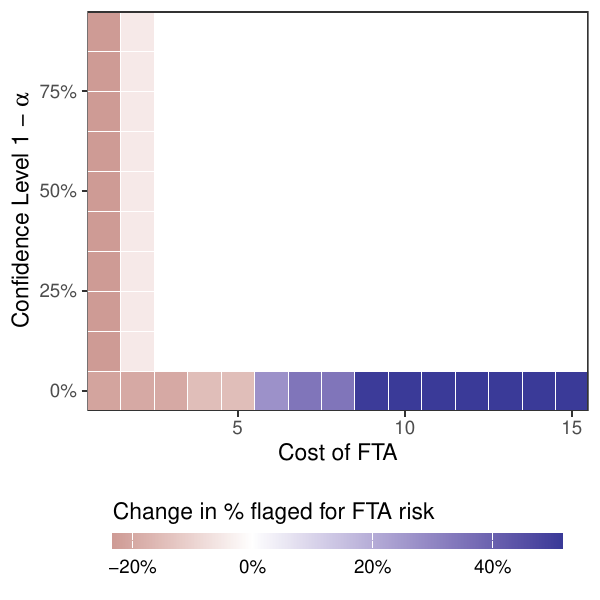} 
    }
    \label{fig:pct_diff_fta_lvl_truncated}
  \end{subfigure}
    \begin{subfigure}[t]{0.45\textwidth}  
   \caption{NCA Score (truncated)} 
  {\centering \includegraphics[width=\textwidth]{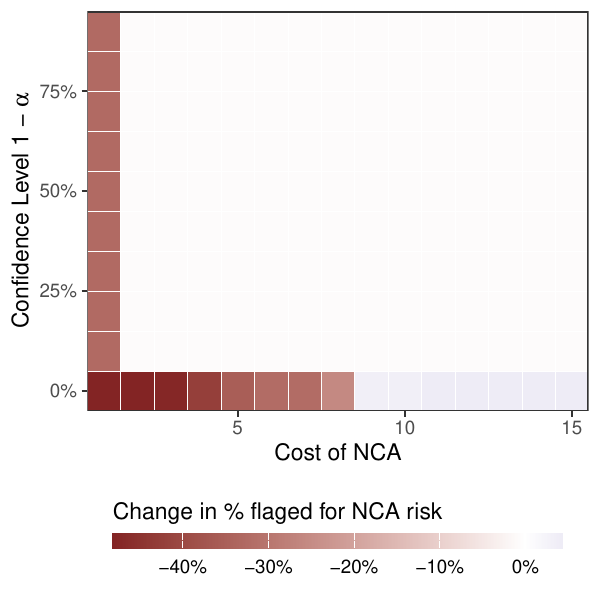} 
  }
    \label{fig:pct_diff_nca_lvl_truncated}
    \end{subfigure}\quad
    \caption{The percentage point difference in the proportion of
      arrestees flagged for (a) FTA risk and (b) NCA risk under the
      maximin policy and the original FTA and NCA scores truncated
      into low and high risk values as the cost of an FTA (left panel)
      and NCA (right panel) increases from 1 to 15 and the confidence
      level varies between 0\% and 100\%.}
  \label{fig:fta_nca_results_lvl_truncated}
\end{figure}

We learn such maximin policies using the binary utility function used
for the NVCA, and truncating the policy class to output either a low
or high risk.  Figure~\ref{fig:fta_nca_results_lvl_truncated} shows
how the resulting truncated scores differ from the original truncated
scores under the additive effect class as the cost of an FTA or NCA
and the confidence level vary.  We find the same pattern as in
Figure~\ref{fig:fta_nca_results_lvl}.  With any degree of statistical
confidence, it is not possible to safely change the underlying
scores. Since the sizes of the model classes are smaller with respect
to the truncated policy classes, the results suggest that there exist
substantial uncertainty as to the heterogeneous effects even for the
truncated FTA and NCA scores.

\subsection{Additional results for the overall DMF risk score and quaternary and ternary bail recommendations}
\label{sec:addl_dmf}

\begin{figure}[t!]
  \vspace{-.25in}
  \centering 
  
  \begin{subfigure}[t]{0.45\textwidth}  
    \caption{Bail recommendation}
    {\centering \includegraphics[width=\maxwidth]{analysis/figure/dmf_matrix_bail-1} 
    }
  \end{subfigure}
    \begin{subfigure}[t]{0.45\textwidth}  
   \caption{Release and monitoring conditions recommendation} 
  {\centering \includegraphics[width=\textwidth]{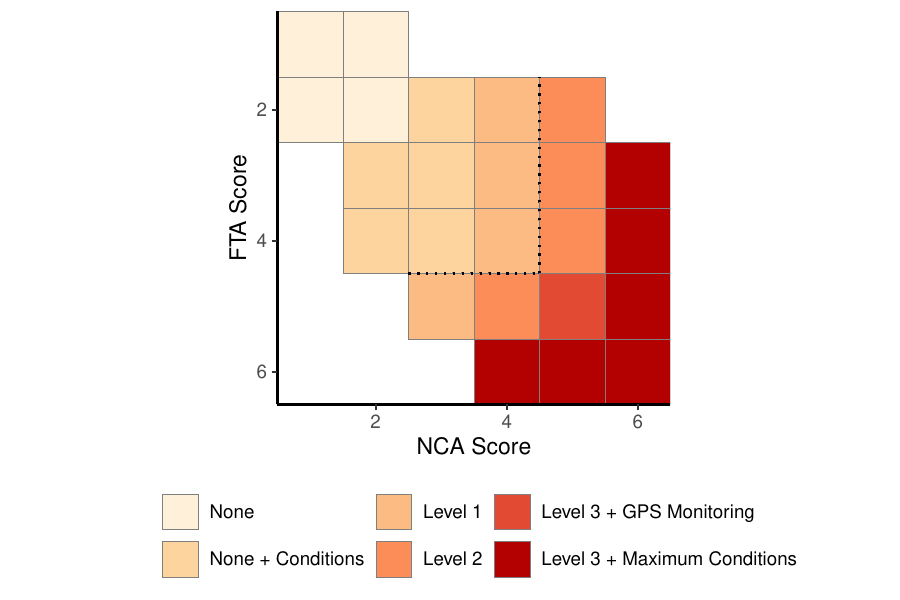} 
  }
    \end{subfigure}\quad
  
  \caption{Decision Making Framework (DMF) matrix recommendation for (a) the
    cash bail decision, and (b) additional release and monitoring conditions,
    for cases where the
    current charge is not a serious violent offense, the NVCA flag is
    not triggered, and the defendant was not extradited. If the FTA
    score and the NCA score are both less than 5, then the
    recommendation is to only require a signature bond. Otherwise, the
    recommendation is to require cash bail. The dashed line indicates
    this boundary.  Unshaded areas indicate impossible combinations of
    FTA and NCA scores. In (b) ``Levels'' 1,2 and 3 correspond to pre-defined
    levels of pretrial supervision, ``None + Conditions'' denotes minor
    conditions the signature bond if appropriate, ``Level 3 + Maximum Conditions''
    corresponds to the highest level of pretrial supervision along with
    additional measures such as biweekly face-to-face and phone contact with
    arrestee. }
\label{fig:dmf_matrix_full}
\end{figure}

\paragraph{Testing for interactions.}
In our main analysis for the binary cash bail recommendation,
 we use an additive model effective model where
$\tau_{\text{add}}(a,\bx) \ = \ \tau_{\text{fta}}(a, x_{\text{fta}}) +
\tau_{\text{nca}}(a, x_{\text{nca}})$.
We can assess the plausibility of this assumption following the same procedure as in
Section~\ref{sec:addl_nvca} above.
We regress the difference in IP-weighted outcomes
$\Gamma(1,\bX,Y) - \Gamma(0, \bX, Y)$ on all
observed interactions between the FTA and NCA scores separately for the
signature bond and cash bail groups.
We then again use a heteroskedastic robust Wald test to test whether there is
evidence for the coefficients for the interaction terms being non-zero, for each
of the signature bond and cash bail groups.
We find some weak evidence for interaction terms in the signature bond region
($p = 0.07$), but not in the cash bail region ($p = 0.13$).

\paragraph{Overall DMF risk score.}

Now we turn to the overall DMF 1--7 risk score that encodes
recommendations on both the level of cash bail and the level and type
of pre-trial supervision and monitoring conditions.  Recall from
Section~\ref{sec:dmf_matrix} that due to the structure of the DMF
matrix, it is not possible to identify the CATE for most risk levels
at most combinations of FTA and NCA scores.  Because we have $K=7$
possible actions, we again usethe linear utility specification
used for the FTA and NCA scores above, though other costs are also
possible.  For the DMF matrix, we again use the NVCA as the outcome.

Figure~\ref{fig:diff_recs_risk} shows the resulting maximin DMF risk
score recommendations for different costs of an NVCA and confidence
levels. We find that it is not possible to safely change the
DMF matrix for the full recommendation if the cost of an NVCA
is larger than 5, even without requiring any degree of statistical certainty.

\paragraph{Quaternary cash bail recommendation.}
We also consider the quaternary cash bail recommendation between a signature bond,
modest cash bail, moderate cash bail, and (full) cash bail. Here we have $K=4$ actions and use the
linear utility function.
Figure~\ref{fig:diff_recs_cash_full} shows the resulting maximin quaternary cash
bail recommendation. This is broadly similar to what we find for the
overall DMF risk score.

\begin{figure}[t!]
  \vspace{-.5in}
  \centering
  \includegraphics[width=0.7\maxwidth]{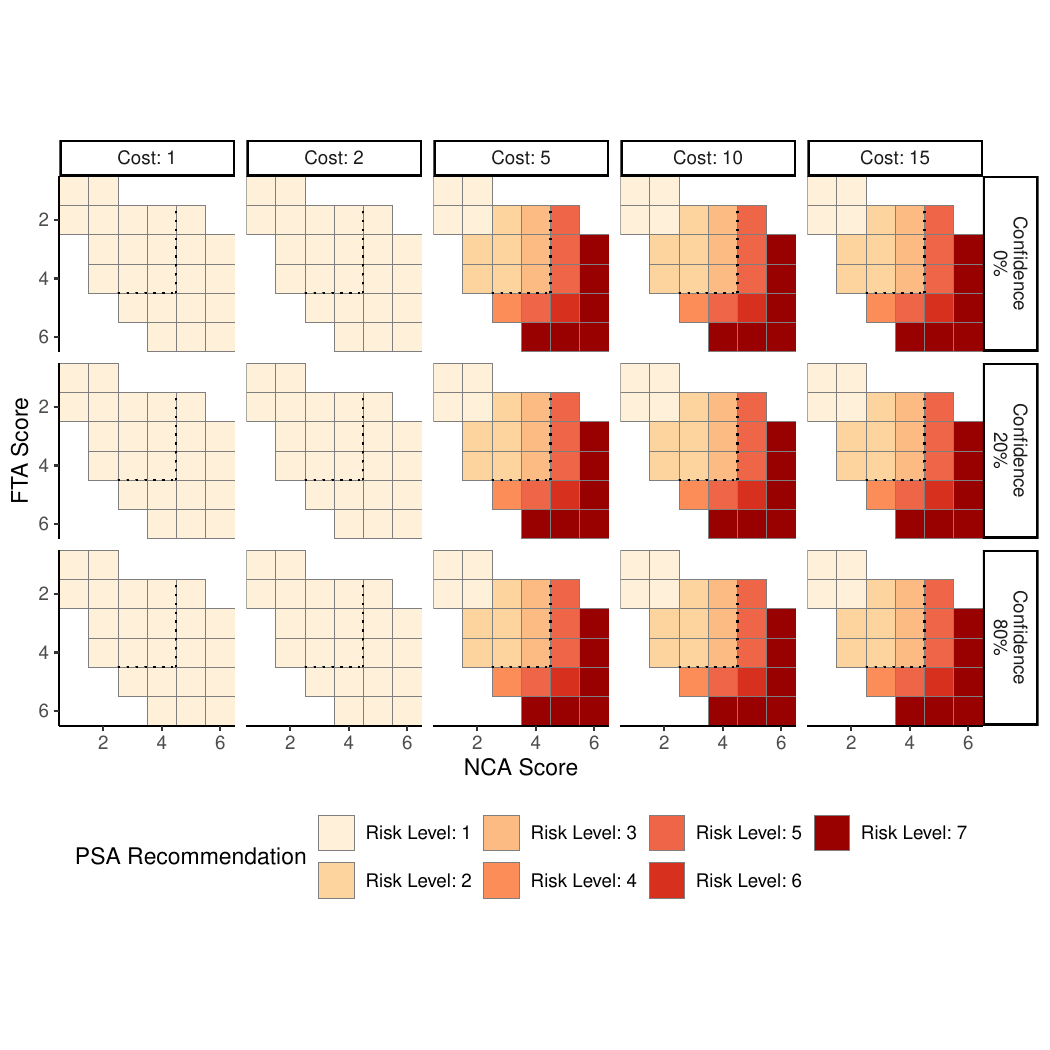}
  \vspace{-.5in}
\caption{Maximin monotone risk level cash bail and pre-trial
  supervision recommendations under an additive model for the
  treatment effects, as the cost of an NVCA and the confidence level
  vary. The dashed black line indicates the original decision boundary
  between a signature bond (above and to the left) and cash bail
  (below and to the right).} 
\label{fig:diff_recs_risk}
\end{figure} 

\begin{figure}[t!]
  \vspace{-.5in}
  \centering
  \includegraphics[width=0.7\maxwidth]{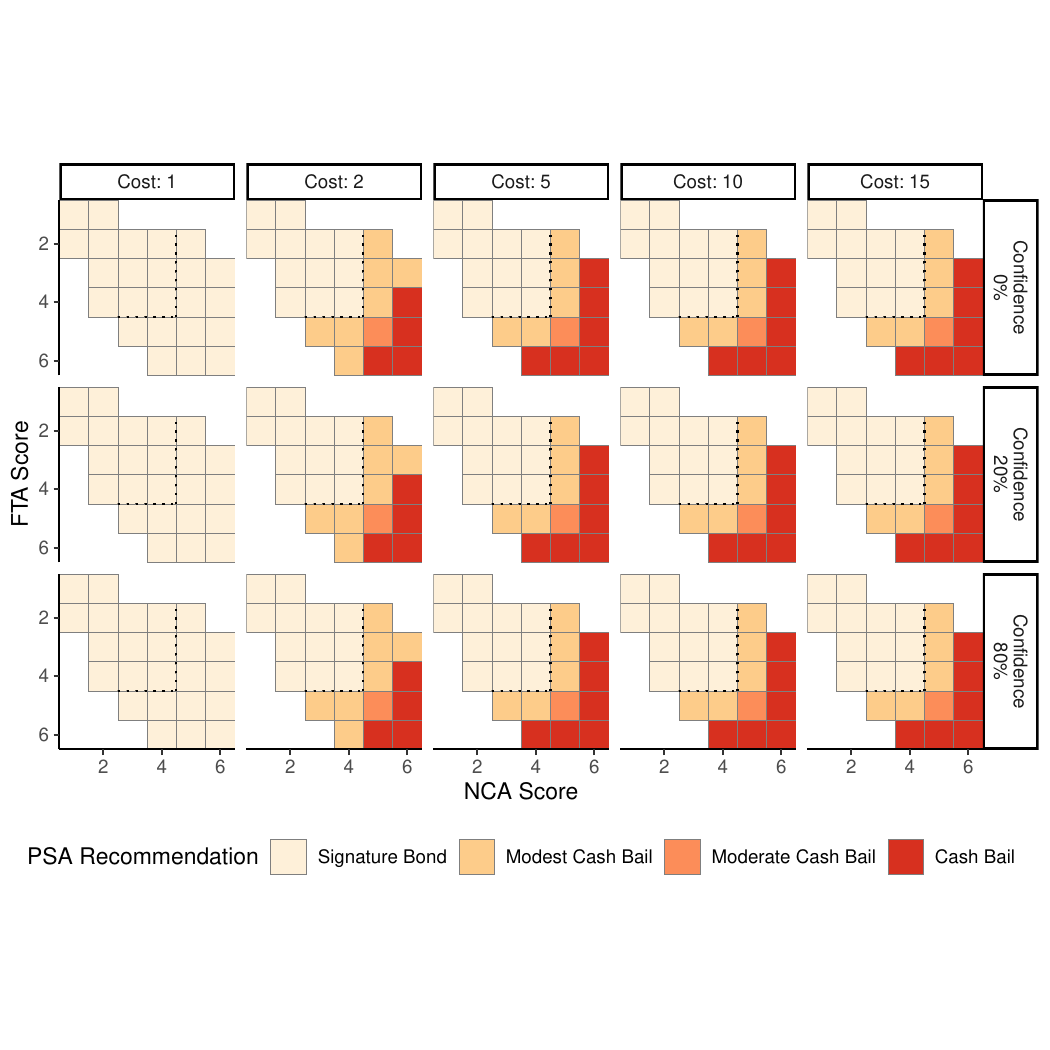}
  \vspace{-.5in}
\caption{Maximin monotone risk level ternary cash bail recommendations under
an additive model for the treatment effects, as the cost of an NVCA and the confidence level
  vary. The dashed black line indicates the original decision boundary
  between a signature bond (above and to the left) and cash bail
  (below and to the right).} 
\label{fig:diff_recs_cash_full}
\end{figure}

\paragraph{Ternary cash bail recommendation.}
We also consider the ternary cash bail recommendation between a signature bond,
moderate/modest cash bail, and full cash bail, collapsing the moderate and 
modest cash bail recommendations. Here we have $K=3$ actions and use the
linear utility function.
Figure~\ref{fig:diff_recs_cash_no_moderate} shows the resulting maximin ternary cash
bail recommendation. This is broadly similar to what we find for the
binary cash bail recommendation. When the confidence level is set to zero and
the cost of an NVCA is high enough, the maximin policy will extend the
region where moderate cash bail is assigned to include the intermediate region
between a signature bond and moderate cash bail. However, if any degree of
statistical confidence is required, the maximin policy reverts to the status
quo. Note that the maximin policy does not change the boundary between modest
cash bail and cash bail, only between a signature bond and modest cash bail.

\begin{figure}[t!]
  \vspace{-.5in}
  \centering
  \includegraphics[width=0.7\maxwidth]{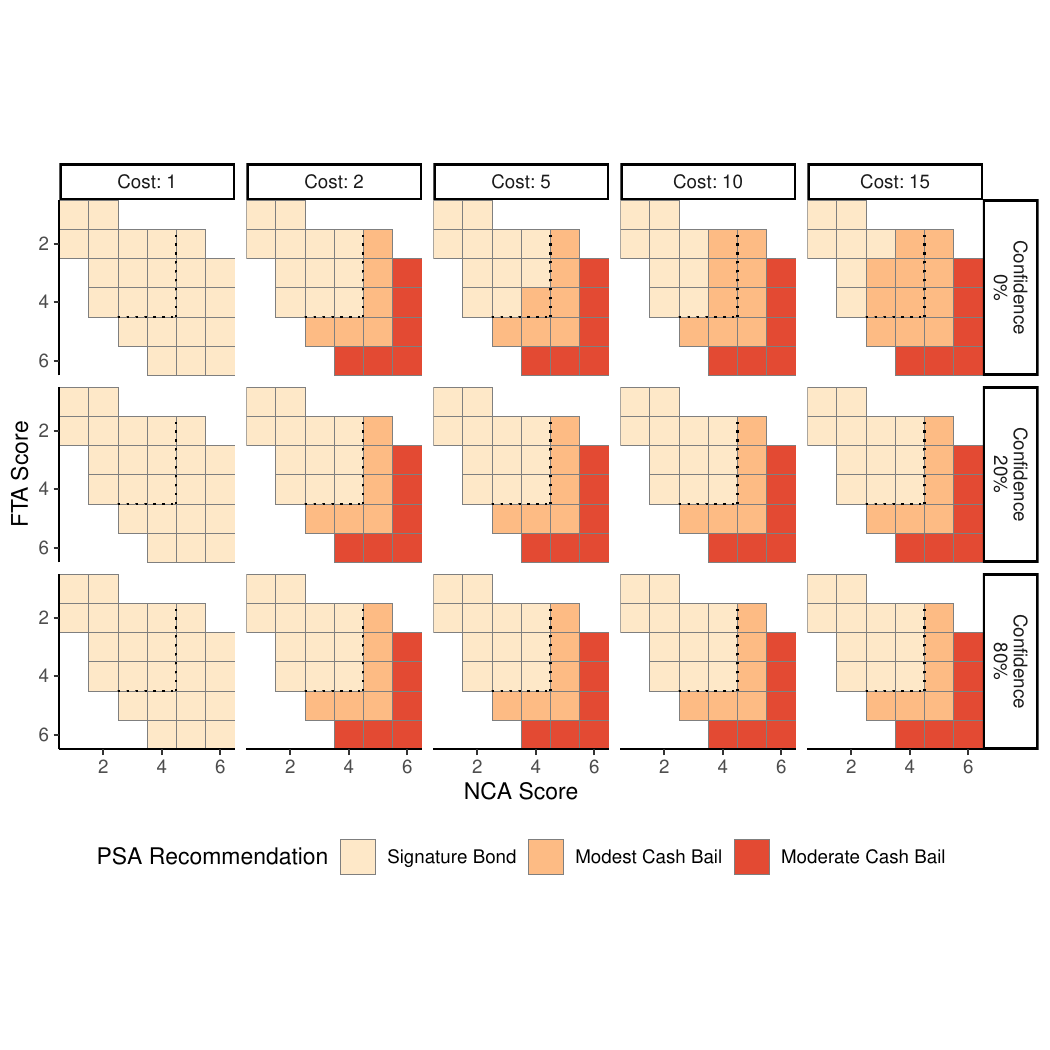}
  \vspace{-.5in}
\caption{Maximin monotone risk level ternary cash bail recommendations under
an additive model for the treatment effects, as the cost of an NVCA and the confidence level
  vary. The dashed black line indicates the original decision boundary
  between a signature bond (above and to the left) and cash bail
  (below and to the right).} 
\label{fig:diff_recs_cash_no_moderate}
\end{figure}

\begin{figure}[t!]
  \centering
  \includegraphics[width=.8\textwidth]{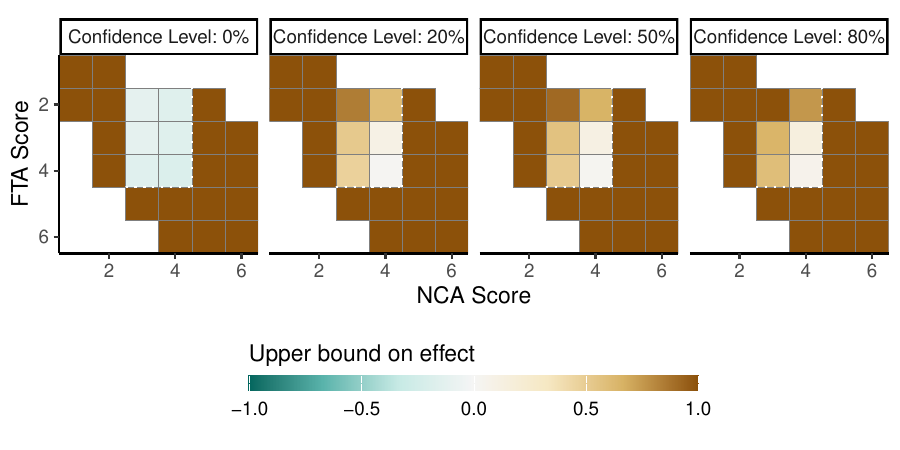}
  \vspace{-.2in}
\caption{Upper bound on the treatment effects under the additive model
  $\tau_{\text{add}}(a, x)$ for FTA and NCA scores. Values below and
  to the right of the dashed white line are areas where cash bail is
  recommended, and the bounds are on the effect of recommending a
  signature bond. Values above and to the left are areas where a
  signature bond is recommended, and the bounds are on the effect of
  recommending cash bail.}
\label{fig:dmf_bounds}
\end{figure}

\clearpage

\pdfbookmark[1]{References}{References}
\singlespacing
\bibliographystyle{chicago}
\bibliography{citations}

\end{document}